%% file: v2.tex
\documentclass{article}

\usepackage{PRIMEarxiv}

\usepackage[utf8]{inputenc} 
\usepackage[T1]{fontenc}    
\usepackage{hyperref}       
\usepackage{url}            
\usepackage{booktabs}       
\usepackage{amsfonts}       
\usepackage{amsmath, amssymb, amsthm}
\usepackage{nicefrac}       
\usepackage{microtype}      
\usepackage{lipsum}
\usepackage{fancyhdr}       
\usepackage{graphicx}       
\usepackage{subcaption}
\usepackage{xcolor}
\usepackage{algorithm}
\usepackage{algorithmic}

\usepackage{listings}
\usepackage{tocloft}  
\usepackage{enumitem}
\setlist[itemize]{leftmargin=1.5em}
\setlist[enumerate]{leftmargin=1.5em}

\usepackage[backend=biber,style=numeric]{biblatex}
\graphicspath{{media/}}     

\newcommand{\ELEMENTARY}{%
  \includegraphics[height=1.0em,keepaspectratio,width=!]{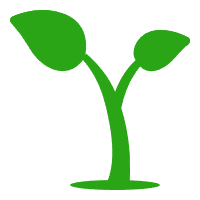}%
}
\newcommand{\TECHNICAL}{%
  \includegraphics[height=1.1em,keepaspectratio,width=!]{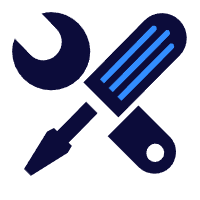}%
}
\newcommand{\CODE}{%
  \includegraphics[height=1.2em,keepaspectratio,width=!]{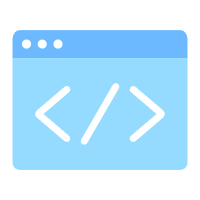}%
}
\newcommand{\THEORY}{%
  \includegraphics[height=1.0em,keepaspectratio,width=!]{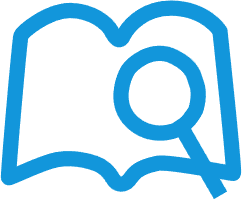}%
}
\newcommand{\EXPERIMENT}{%
  \includegraphics[height=1.0em,keepaspectratio,width=!]{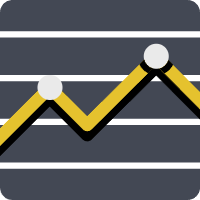}%
}
\newcommand{\SCI}{%
  \includegraphics[height=1.2em,keepaspectratio,width=!]{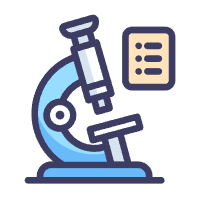}%
}

\theoremstyle{plain}
\newtheorem{theorem}{Theorem}[section]   

\theoremstyle{definition}
\newtheorem{definition}[theorem]{Definition}
\newtheorem{example}[theorem]{Example}

\theoremstyle{remark}

\title{Efficient and Principled Scientific Discovery through Bayesian Optimization: A Tutorial
}

\makeatletter
\def\blfootnote{\xdef\@thefnmark{}\@footnotetext}
\makeatother
\newcommand{\authornotes}{
\blfootnote{$\dagger$ Equal contribution. $\ddagger$ Corresponding authors.}
\blfootnote{Emails by order: 
\texttt{zhongwei-yu@outlook.com},
\texttt{\{rasul.tutunov, alexandre.maraval1\}@huawei.com},
\texttt{zikaix@ustc.edu.cn},
\texttt{tzz24@mails.tsinghua.edu.cn},
\texttt{wangjiankang@connect.hku.hk},
\texttt{bcao686@connect.hkust-gz.edu.cn},
\texttt{goldenli@connect.hku.hk},
\texttt{xuliang0826999@gmail.com},
\texttt{yangqi@hlsct.cn},
\texttt{jiangj1@ustc.edu.cn},
\texttt{luosz@tsinghua.edu.cn},
\texttt{zhenxiao.guo@hku.hk},
\texttt{mezhangt@hkust-gz.edu.cn},
\texttt{haitham.ammar@huawei.com},
\texttt{jun.wang@cs.ucl.ac.uk}
}
}

\author{
\textbf{Zhongwei Yu}\textsuperscript{1}, 
\textbf{Rasul Tutunov}\textsuperscript{2,$\dagger$}, 
\textbf{Alexandre Max Maraval}\textsuperscript{2,$\dagger$}, 
\textbf{Zikai Xie}\textsuperscript{3,$\dagger$}, 
\textbf{Zhenzhi Tan}\textsuperscript{4,$\dagger$},  \\
\textbf{Jiankang Wang}\textsuperscript{5,$\dagger$}, 
\textbf{Bin Cao}\textsuperscript{1,2, $\dagger$},
\textbf{Zijing Li}\textsuperscript{5}, 
\textbf{Liangliang Xu}\textsuperscript{5}, 
\textbf{Qi Yang}\textsuperscript{4,6,$\ddagger$}, 
\textbf{Jun Jiang}\textsuperscript{3,$\ddagger$},  \\
\textbf{Sanzhong Luo}\textsuperscript{4,$\ddagger$}, 
\textbf{Zhenxiao Guo}\textsuperscript{5,$\ddagger$}, 
\textbf{Tongyi Zhang}\textsuperscript{1,7,$\ddagger$}, 
\textbf{Haitham Bou-Ammar}\textsuperscript{2,$\ddagger$}, 
\textbf{Jun Wang}\textsuperscript{8,$\ddagger$}
\vspace{0.6em} \\
\small
\textsuperscript{1}The Hong Kong University of Science and Technology (Guangzhou) \quad
\textsuperscript{2}Huawei Noah's Ark Lab \\
\textsuperscript{3}University of Science and Technology of China \quad
\textsuperscript{4}Tsinghua University \\
\textsuperscript{5}The University of Hong Kong \quad
\textsuperscript{6}Haihe Laboratory of Sustainable Chemical Transformations \\
\textsuperscript{7}Shanghai University \quad
\textsuperscript{8}University College London
}

\begin{document}

\maketitle

\begin{figure}[!h]
    \centering
    \includegraphics[width=1\linewidth]{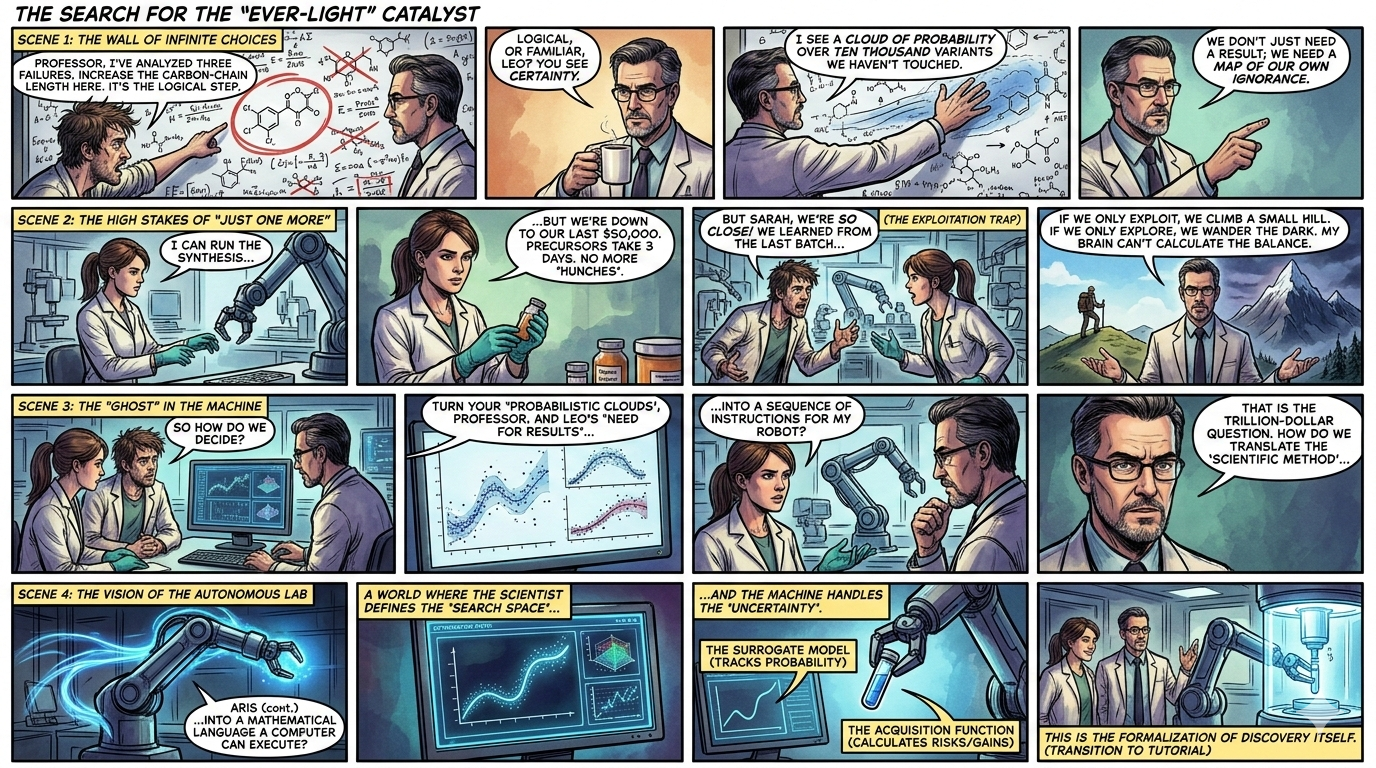}
    \caption{
        A four-panel comic illustrates the conceptual journey of translating human scientific intuition into a formalised, autonomous discovery framework for identifying a high-performance ``ever-light'' catalyst. \emph{All characters, dialogues, and events depicted in this narrative are entirely fictional and are used solely for illustrative purposes.}
        The main characters are \textbf{Professor Aris} (senior researcher), \textbf{Leo} (PhD student), and \textbf{Sarah} (research engineer).
        The four scenes:
        (1) \textit{The Wall of Infinite Choices} highlights the complexity of the chemical search space and the tension between logical certainty and probabilistic exploration;
        (2) \textit{The High Stakes of ``Just One More''} exposes the ``exploitation trap'' of settling for local optima under resource constraints versus unguided exploration;
        (3) \textit{The ``Ghost'' in the Machine} poses the question of translating uncertainty into actionable instructions for autonomous experimentation;
        (4) \textit{The Vision of the Autonomous Lab} resolves this tension via Bayesian optimization, where surrogate models quantify uncertainty, acquisition functions balance exploration/exploitation, and autonomous labs navigate the search space efficiently. 
    }
    \label{fig:comic_intro}
\end{figure}

\input{v2/abstract}

\authornotes

\newpage
\setcounter{tocdepth}{1}
\tableofcontents

\section*{Preliminary Note}
In order to help readers navigate the material and quickly identify which parts are most relevant to their background and interests, most sections are tagged by small visual markers (icons). Each marker corresponds to a particular style of content and an intended audience, as described below:

\begin{itemize}
\setlength{\itemsep}{0em}
\item \ELEMENTARY: This section covers fundamental principles, concepts, and summaries. It is designed to be accessible to all readers, featuring elementary mathematics and no complex technical details.
\item \TECHNICAL: This section provides technical methodological details, requiring some background in machine learning or statistics.
\item \CODE: This section includes coding instructions and hands-on samples for practical implementation. Readers are expected to be familiar with the basic use of Python.
\item \THEORY: This section provides theoretical derivations and analysis, prepared for readers with proper mathematical background and interests in the theoretical rationales.
\item \EXPERIMENT: This section details experimental results and performance benchmarks.
\item \SCI: This section presents case studies of BO applied to real-world scientific discovery.
\end{itemize}

\include{v2/intro}

\include{v2/sci_as_opt}

\include{v2/surrogate}

\include{v2/decision}

\include{v2/algorithm}

\include{v2/conclusion}

\printbibliography

\end{document}

%% file: v2/abstract.tex
\begin{abstract}

Traditional materials development follows a well-established paradigm: researchers formulate hypotheses, conduct experiments, and refine theories based on empirical evidence. This cycle has driven scientific progress for centuries; however, it often scales poorly with increasing problem complexity. A typical materials optimization problem involves 5--15 design variables spanning composition ratios, processing temperatures, annealing times, and processing atmospheres. Each experimental iteration can take days to weeks and cost hundreds to thousands of dollars. This combinatorial explosion restricts researchers to exploring less than 0.1\% of the feasible design space, leaving vast regions of potentially superior materials undiscovered.


This tutorial introduces \emph{Bayesian Optimization} (BO). By modeling unknown mechanisms within a Bayesian framework, BO formalizes key scientific processes: surrogate models (e.g., Gaussian processes) provide continuously updated representations of empirical observations, while acquisition functions guide the selection of the most promising candidate experiments for validation and refinement. In this way, BO balances exploitation and exploration, improving information gain and reducing the number of required evaluations relative to conventional trial-and-error approaches.

In this tutorial, we show that scientific discovery is usually associated with objectives and thus, can be framed as optimisation problems. Then, we unpack key building blocks of BO, including surrogate models, Gaussian processes, acquisition functions. Consequently, we present algorithmic workflows and validate BO’s efficacy through real-world case studies in catalysis, materials science, organic synthesis, and molecule discovery. In addition, we also address critical technical extensions for scientific applications (e.g., batched experimentation, heteroscedasticity, contextual optimisation, integrating human in the loop), ensuring robustness and adaptability to the constraints of real-world laboratory and computational science.

Tailored for a broad scientific audience, this tutorial bridges BO advancements in artificial intelligence with practical implementation across natural sciences. It offers tiered content: practical coding examples for experimentalists seeking to accelerate their work; mathematical foundations for researchers developing domain-specific BO methods; and principled insights into uncertainty-aware decision-making for general readers. Ultimately, this tutorial empowers cross-disciplinary researchers to design experiments more efficiently, leading to more principled, accelerated scientific discovery.

\end{abstract}

%% file: v2/intro.tex
\section{Introduction}
\label{sec:intro}

At its core, scientific discovery is the systematic acquisition of knowledge regarding the natural world. This process has evolved through a succession of rigorous methodological frameworks \cite{hepburnScientificMethod2026}, as illustrated in Figure~\ref{fig:scientific_discovery_paradigms}. In early practice, Aristotelian demonstrative science deduced knowledge from first principles, such as axioms and postulates. Conversely, Francis Bacon pioneered the inductive approach, advocating for the systematic accumulation of empirical observations to reveal nature's underlying patterns. Later, these deductive and empirical threads were integrated into the Hypothetico-Deductive (H-D) model, in which scientists propose tentative hypotheses and test their logical consequences against experimental evidence. Karl Popper further refined this model by introducing falsificationism, asserting that scientific progress occurs not by proving theories true, but by systematically attempting to prove them false---ensuring that only the most robust explanations survive.

These historical methodologies are not merely relics; they constitute the diverse toolkit of modern science. In contemporary research, different methods are applied at distinct stages of investigation. For instance, the inductive approach has been revitalised by the ``Big Data'' era, where machine learning identifies patterns in vast datasets. The H-D method remains the gold standard for experimental design, while deductive logic provides the backbone for theoretical derivations and mathematical modelling. Together, these frameworks create a multi-layered approach to understanding complex phenomena.

However, scientific discovery is more than a static set of rules; it is a \emph{meta-methodology} that evolves over time. \textcite{kuhnStructureScientificRevolutions1994} famously argued that science progresses through paradigm shifts rather than steady accumulation. Kuhn distinguished between two facets of discovery: \textit{discovery-that} (the initial observation of an anomaly) and \textit{discovery-what} (the subsequent conceptualisation of that anomaly within a coherent theoretical framework) \cite{kuhnHistoricalStructureScientific1962, schindlerScientificDiscoveryThatWhats2015}. When ``discovery-thats'' (observations that existing paradigms cannot reconcile) reach a critical mass, the prevailing framework enters a crisis, ultimately leading to a new paradigm that redefines the ``discovery-what.''

\begin{figure}[tbp]
    \centering
    \includegraphics[width=\textwidth]{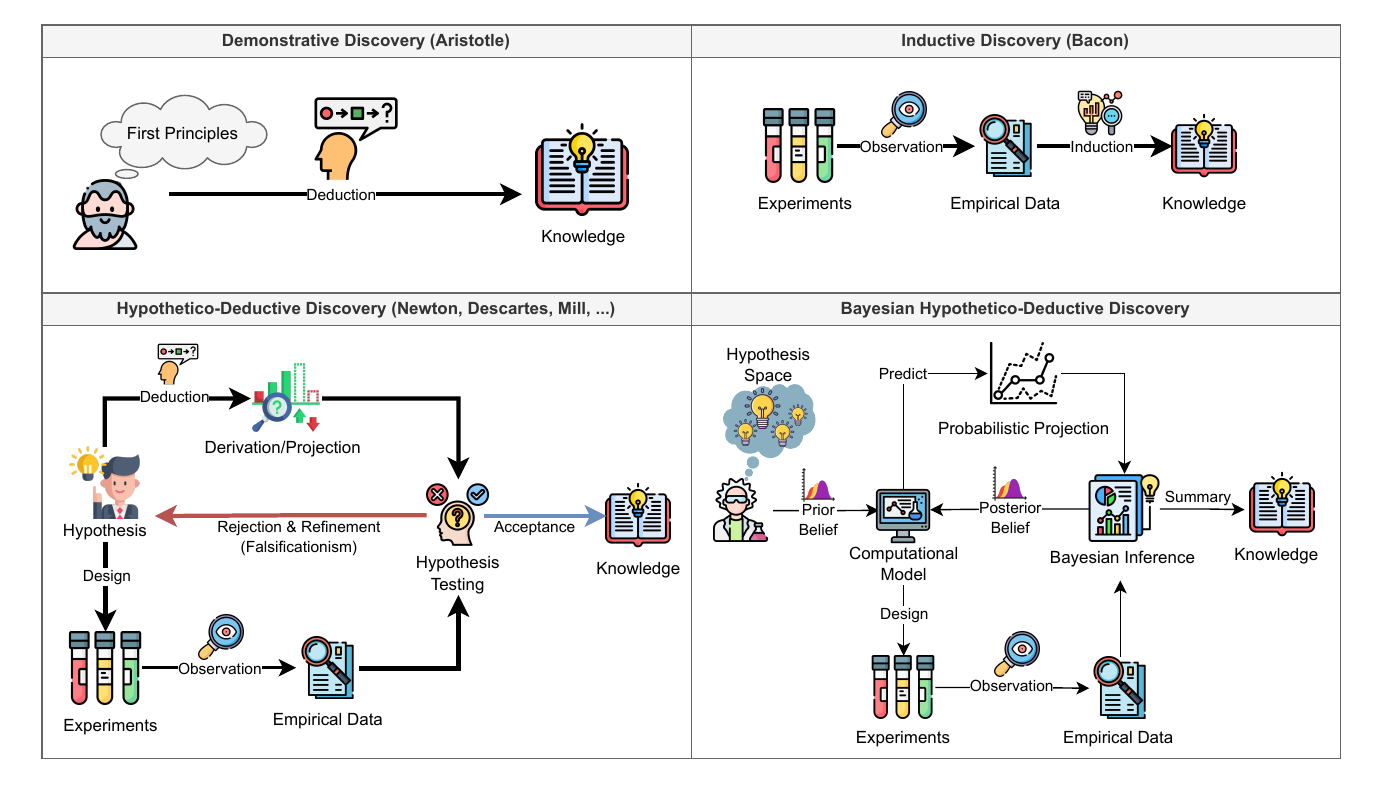}
    \caption{The illustration of core paradigms of scientific discovery with historical representatives. (i) \textbf{Demonstrative discovery} derives certain knowledge from self-evident first principles through formal logical deduction, , with the inferential process itself independent of empirical observation. (ii) \textbf{Inductive discovery} constructs general laws by systematically accumulating empirical data via controlled experiments, inducing patterns from an evidence base. (iii) \textbf{Hypothetico-deductive discovery} proposes tentative hypotheses, derives testable predictions, validates them against empirical data, and refines or rejects the hypotheses through consistency checks. (iv) \textbf{Bayesian hypothetico-deductive discovery} is an emergent modern methodology that quantifies degrees of belief across a hypothesis space and iteratively updates these beliefs via Bayesian inference; the resulting posterior provides both a summary of robust knowledge and guidance for decision-making.}
    \label{fig:scientific_discovery_paradigms}
\end{figure}

\subsection{The Bayesianism of Scientific Discovery}
\label{sec:intro:bayesian_discovery}

We are currently witnessing a new revolution in scientific methodology: the shift towards Bayesianism \cite{linBayesianEpistemology2024}, which extends the H-D paradigm by explicitly accounting for uncertainty over hypotheses (Figure~\ref{fig:scientific_discovery_paradigms}). In modern research, the hypothesis space (i.e., the set of all possible explanations) is often gargantuan or infinite. Because a hypothesis can only be absolutely certain if all competing alternatives are rejected (a rare feat given limited experimental resources), it is seldom feasible to declare a single hypothesis as definitively correct. To address this, the Bayesian H-D method acknowledges the inherent uncertainty that experimental evidence can rarely fully eliminate. It treats scientific discovery as the continuous updating of degrees of belief. This shifts the focus from the deterministic claims of the traditional H-D method to an adaptive belief distribution that is refined as new evidence emerges.

This distributional view revolutionises experimental design. Since the objective is no longer to validate a single deterministic claim, experimentation instead focuses on the belief dynamics across the entire hypothesis space. While integrating belief over such complex spaces was once computationally intractable, modern technology has made it feasible. Scientists can now embed hypothesis spaces into computational models, representing belief dynamics as computable functions of empirical data. Furthermore, rather than relying solely on human intuition, these models can serve as automated experiment designers. A notable application of this paradigm is \emph{Bayesian experimental design} (BED) \cite{rainforthModernBayesianExperimental2023}, where experiments are selected to maximise information gain.

Real-world science often extends beyond seeking pure explanations; we also seek knowledge to inform active decision-making. In many cases, scientific discovery is best represented as the optimisation of specific objectives, where the relationship between actions and outcomes must be mapped through active experimentation. In this tutorial, we introduce \emph{Bayesian optimisation} \cite{shahriariTakingHumanOut2016, garnettBayesianOptimization2023} as the primary methodology for this objective-oriented discovery. While BO shares the H-D and BED focus on exploring hypotheses, it goes further by coupling belief updates with the necessity of making decisions relative to specific goals.

\subsection{A General Workflow for Scientific Discovery with Bayesian Optimisation}

Bayesian Optimisation is an efficient optimisation algorithm originating from computer science. However, rather than being a mere black-box algorithm, it also serves as the modern computational crystallisation of the \textit{Bayesian hypothetico-deductive} philosophy discussed in Section~\ref{sec:intro:bayesian_discovery}. It provides a formal, iterative framework to navigate the gargantuan search spaces of modern science. To understand this workflow conceptually, it is helpful to view it as a cycle between two primary components that mirror the scientific method: a \textbf{surrogate model} and an \textbf{acquisition function}.

In the Bayesian framework, the \textbf{surrogate model}, commonly a \textbf{Gaussian Process} (GP), represents our current ``state of belief'' regarding the natural laws governing the system (e.g., the relationship between a protein sequence and its affinity). Much like the evolving paradigms in the history of science, the surrogate model acts as a probabilistic ``digital twin'' of the experiment. It does not simply provide a single best guess; it maintains a distribution over all possible outcomes, explicitly quantifying its own uncertainty. In regions where data is scarce, the model acknowledges its lack of knowledge through high variance, capturing the ``epistemic uncertainty'' that is central to Bayesian epistemology.

Guided by this internal model, we use an \textbf{acquisition function} to perform the ``deductive'' step of the discovery cycle. This component formalises the strategy for the next experiment, deducing which observation will be most informative based on our current beliefs. It balances two fundamental scientific drivers: \textit{exploitation} (testing hypotheses the model predicts are highly likely to succeed) and \textit{exploration} (testing hypotheses in unknown regions to challenge and refine the current paradigm). 

This iterative loop, proposing a hypothesis, testing it against the ``oracle'' (the real world or a high-fidelity simulation), and updating our beliefs, perfectly mirrors the evolution of scientific knowledge. By using this principled decision rule, BO ensures that every expensive experiment is chosen to provide the maximum possible value, accelerating the objective-directed scientific discovery in a data-efficient manner. Below is a practical and conceptual workflow to integrate BO into scientific discovery:

\begin{enumerate}
    \item \textbf{Frame the scientific challenge as an optimisation problem}: Start by defining three key components: the \textit{design space} (the experimental parameters you can adjust, e.g., reaction temperature, material composition, or solvent type), \textit{constraints} (practical limits like safety regulations, resource availability, or biophysical feasibility), and the \textit{objective} (what you aim to maximise or minimise, e.g., catalytic efficiency, molecular stability, or minimising energy use). This step translates vague scientific questions (``Can we identify a better solution?'') into concrete targets (``identify an optimal solution while minimizing cost and maximizing performance'').
    \item \textbf{Establish an objective evaluation process}: Determine how you will measure the level of success. If real-world experiments are costly or time-consuming, use existing datasets (e.g., from literature or lab records) to train an ``oracle'', which serves as a computational stand-in that mimics the outcome of actual experiments. For feasible hands-on work, design an evaluation pipeline compatible with lab equipment or simulations to generate reliable results.
    \item \textbf{Build an initial knowledge base and update mechanism}: Your initial ``belief'' about the problem comes from a \textit{surrogate model}---a flexible tool that learns from data to represent how design choices relate to outcomes. Kickstart this model with a small set of initial experiments (e.g., random samples across the design space or expert-guided tests) to capture baseline knowledge. The surrogate will then update its beliefs as new data emerges, mirroring how scientific understanding evolves.
    \item \textbf{Define a strategy for selecting next experiments}: To avoid haphazard testing, use an \textit{acquisition function}---a principle to balance two goals: exploring understudied regions (to avoid missing breakthroughs) and exploiting regions the surrogate deems promising (to refine top performers). For simplicity, this step can be streamlined by choosing user-friendly BO software \cite{cao2026bgolearn} that preconfigures these strategies for scientific use cases.
    \item \textbf{Execute the iterative BO loop}: This loop follows active learning stargate \cite{lookman2019active}: (i) use the surrogate to predict outcomes of potential experiments; (ii) select the most valuable experiment (via the acquisition function); (iii) run the experiment or query the oracle; (iv) feed the new result back to update the surrogate. Repeat until your objective is met or resources are exhausted.
    \item \textbf{Synthesise knowledge and actionable results}: Conclude by extracting two key outputs: \textit{model-based insights} (from the surrogate, e.g., how temperature influences yield) and \textit{optimal experimental conditions} (the best design from your history of tests). This dual outcome ensures you gain both practical solutions and a deeper understanding of the system.
\end{enumerate}


Crucially, BO is not a one-size-fits-all tool; its components (surrogate, acquisition function) must be tailored to the problem domain. For instance, catalyst or molecular design often leverages graph-based surrogates to encode structural information, while materials discovery prioritizes discrete, conditionally constrained composition spaces and trade-offs among competing objectives. By embedding BO into your existing scientific workflow, you replace guesswork with principled decision-making, accelerating discovery without sacrificing rigour.

\subsection{Illustrative Example I: Antibody Discovery with BO}

To illustrate how BO functions in practice, we consider its application to automated antibody design. The goal is to identify high-affinity protein sequences within an astronomically large combinatorial search space using as few costly evaluations as possible. Through this example, we establish how scientific discovery is framed as an optimisation problem and computerised through the BO framework, as illustrated in Figure~\ref{fig:ant-bo-workflow}.

We focus on the design of antibodies that bind strongly to a specific antigen. The ability of an antibody to recognise a target is largely determined by the Complementarity-Determining Region 3 of the heavy chain (CDRH3). From a design perspective, we treat an antibody candidate as a sequence of amino acids defining this CDRH3 region. The scientific objective is to \textit{discover sequences that maximise binding affinity to a target antigen while remaining biologically viable}.

The challenge is twofold. First, the space of possible sequences is combinatorially explosive; for a typical CDRH3 length, the number of combinations exceeds the number of atoms in the observable universe. Second, evaluating a candidate, whether through wet-lab experiments or high-fidelity simulations, is expensive and time-consuming. We treat this evaluation as a ``black-box'' oracle: we provide a sequence, and it returns a noisy estimate of binding affinity, without providing any clues about why that sequence worked or how to improve it.

\begin{figure}[tbh]
    \centering
    \includegraphics[width=\linewidth]{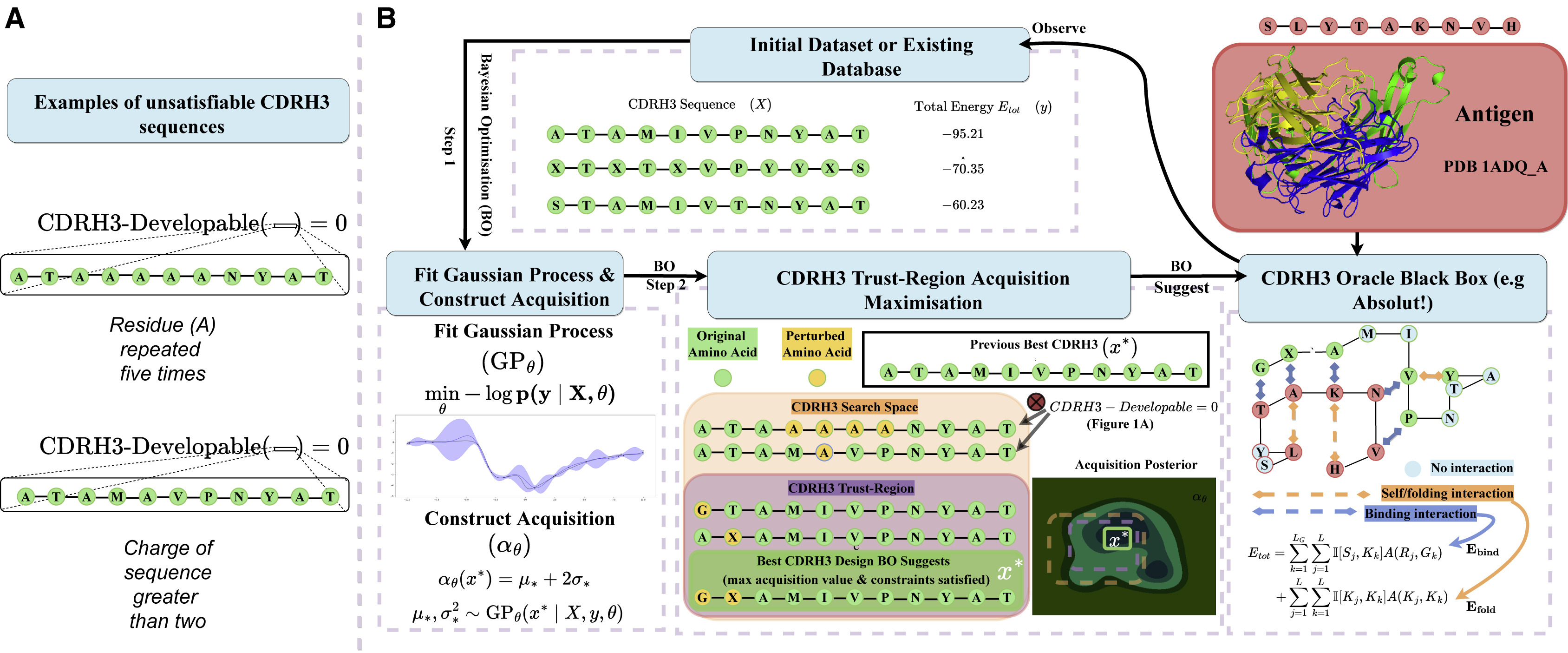}
    \caption{The AntBO workflow for automated antibody design (Figure 1 of \textcite{khan2022antborealworldautomatedantibody}). \textbf{(A) Constraints on the search space:} Examples of biologically unsatisfiable CDRH3 sequences (e.g., excessive repetition or invalid charge) that the optimiser must learn to avoid. \textbf{(B) The Bayesian Optimisation loop:} (Step 1) Starting from an initial dataset of sequences and their energy/affinity scores; (Step 2) Fitting a Gaussian Process ($\text{GP}_\theta$) to model the objective and its uncertainty; (Step 3) Maximising an acquisition function ($\alpha_\theta$) within a local \textit{trust region} to propose a candidate $x^*$ that balances performance and biological plausibility; (Step 4) Evaluating the candidate via a black-box oracle (e.g., the \textit{Absolut!} simulator) to generate new data and update the model.}
    \label{fig:ant-bo-workflow}
\end{figure}

This paradigm is instantiated in AntBO \cite{khan2022antborealworldautomatedantibody}, a framework designed for real-world antibody discovery. As shown in Figure~\ref{fig:ant-bo-workflow}B, AntBO formalises the discovery process into a four-step loop. It begins with an initial dataset (Step 1) and fits a Gaussian Process to represent the current understanding of the sequence-affinity relationship (Step 2). 

A critical refinement in AntBO is the use of a \textbf{trust region} during the acquisition phase (Step 3). As illustrated in Figure~\ref{fig:ant-bo-workflow}A, many mathematically ``optimal'' sequences might be biologically ``unsatisfiable''---for instance, they may have an impossible chemical charge or repetitive patterns that cannot be manufactured. To avoid wasting resources on these ``dead ends,'' the trust region constrains the search to sequences that are similar to known, developable antibodies. Finally, the suggested sequence is sent to an oracle (Step 4), such as the \textit{Absolut!} simulator, and the result is fed back into the database to refine the model for the next round.

In experimental trials, AntBO identified high-affinity candidates after evaluating fewer than fifty sequences---a tiny fraction of the billions possible. This demonstrates that by coupling a probabilistic ``digital twin'' with a strategic decision rule, researchers can bypass the trial-and-error of traditional discovery and move directly toward optimal scientific solutions.

\subsection{Illustrative Example II: Accelerating OER Catalyst Discovery with BO}

We consider the application of BO to the design of acidic oxygen evolution reaction (OER) catalysts \cite{cao2025spatial}. The goal is to identify synthesis conditions that jointly optimise catalytic activity and durability, while reducing the reliance on expensive experiments. This setting exemplifies how materials discovery can be cast as a data-efficient optimisation problem under realistic laboratory constraints.

We study RuO$_2$-based catalysts for acidic water electrolysis, where performance is governed by two competing objectives: \textit{low overpotential} (indicating high activity) and \textit{high stability}. Each candidate catalyst is parameterised by synthesis variables, including the Ru precursor volume ($V_{\mathrm{Ru}}$), solvent volume ($V_{\mathrm{H_2O}}$), calcination temperature ($T$), and calcination time ($t$). The objective is therefore to identify synthesis configurations that minimise overpotential while maintaining robust durability in acidic environments.

This optimisation problem is inherently challenging. Overpotential can be measured relatively quickly, whereas stability evaluation requires prolonged experiments that may span hundreds of hours. In addition, the design space is combinatorially large; even coarse discretisation of synthesis parameters results in tens of thousands of candidates. Exhaustive evaluation is thus impractical, necessitating careful experimental selection.

We model the experimental workflow as a black-box function that maps synthesis conditions to observed properties such as overpotential and stability, without providing gradients or mechanistic insight. This formulation is well suited to BO, where a probabilistic surrogate guides sequential decision-making.

To account for the asymmetric evaluation cost across objectives, we employ a \textbf{two-stage spatially adaptive optimisation strategy} (Figure~\ref{fig:BGO-OER}). In the first stage, BO targets the inexpensive objective, namely overpotential. Starting from a limited initial dataset, a surrogate model (e.g., a Gaussian process or ensemble regressor) is trained to approximate the relationship between synthesis parameters and activity. An acquisition function, such as expected improvement, knowledge gradient or predictive entropy search, is then optimised to propose new candidates. Iterative updates quickly identify a subset of high-activity materials.

A key step connects the two stages. Rather than directly exploring stability over the entire design space, a conditional variational autoencoder (CVAE) is used to construct an \textit{adaptive subspace} composed of candidates with low predicted overpotential. This effectively restricts the search to regions that satisfy a predefined activity criterion, filtering out low-potential candidates prior to costly evaluations.

In the second stage, BO is applied within this reduced domain, now focusing on maximising stability. Given the high cost of stability testing, each iteration evaluates only a single candidate. By concentrating on a high-quality subspace, the method substantially improves data efficiency and reduces the number of required long-term experiments.

Experimental results show that this strategy successfully identifies a Cu-doped RuO$_2$ catalyst with both low overpotential and excellent durability, using only a small number of synthesis trials and stability tests. Compared with conventional trial-and-error approaches, the BO-based framework reduces experimental cost by orders of magnitude.

\begin{figure}[tbh]
    \centering
    \includegraphics[width=\linewidth]{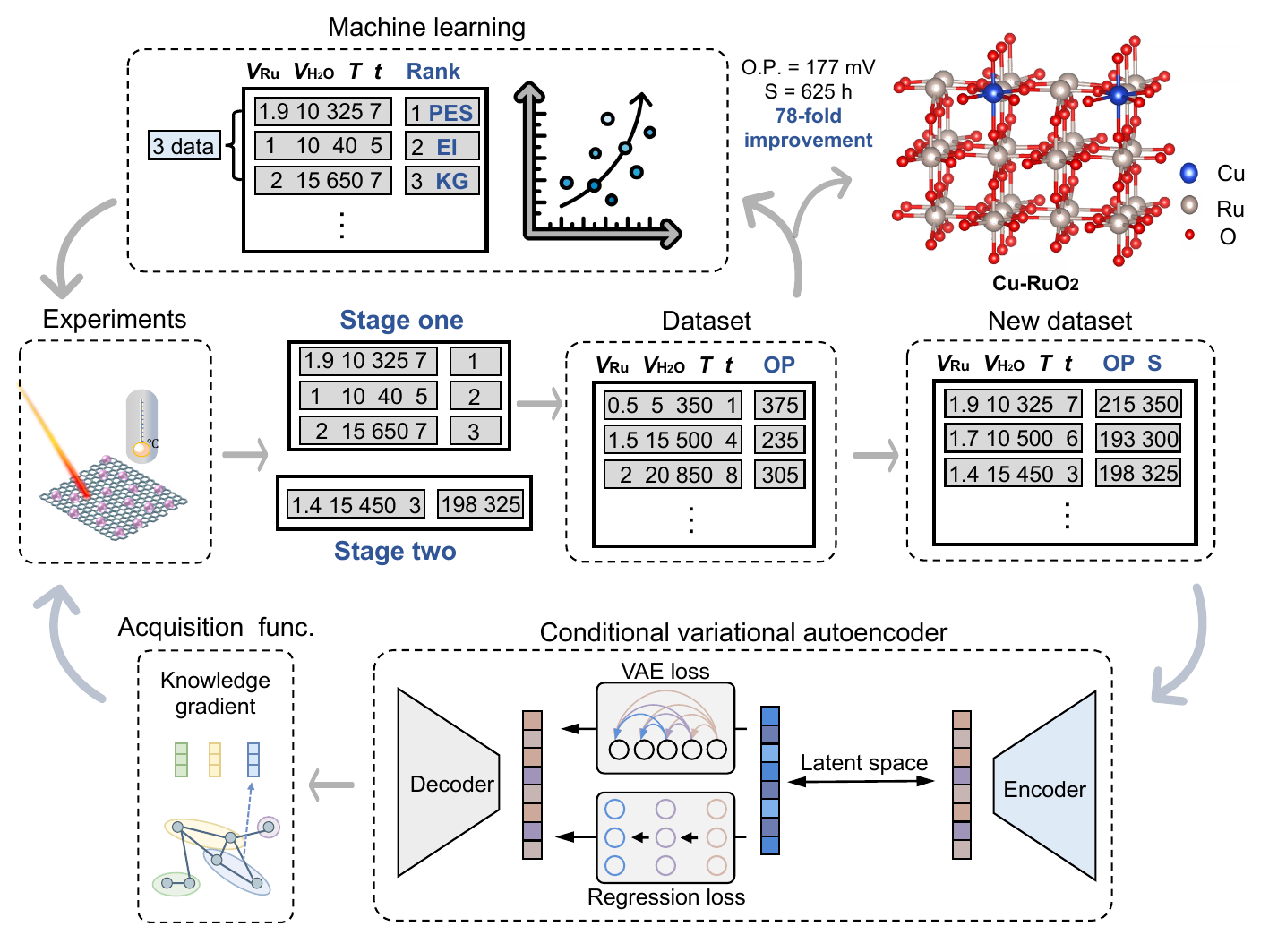}
    \caption{Workflow of the spatially adaptive active learning strategy. Stage I minimises overpotential via active learning, while a CVAE defines a low-overpotential subspace. Stage II then maximises stability within this subspace to identify highly active and durable catalysts. PES, EI, and KG denote predictive entropy search, expected improvement, and knowledge gradient, respectively. Func.: function.}
    \label{fig:BGO-OER}
\end{figure}

\subsection{Literature Review of Bayesian Optimisation}

Bayesian Optimisation \cite{garnettBayesianOptimization2023} is a principled approach for searching the global maximum of an expensive-to-evaluate black-box objective function. Its origins lie in early prototypes developed by \textcite{kushnerNewMethodLocating1964, mockusBayesMethodsSeeking1975, zhilinskasSinglestepBayesianSearch1976}, with \textcite{mockusBayesMethodsSeeking1975} systematically formalising the Bayesian framework for global optimisation. By the end of the 20th century, \textcite{jonesEfficientGlobalOptimization1998} popularised this methodology through an algorithm termed Efficient Global Optimisation (EGO). A further surge in BO research emerged around 2010, when Srinivas \textit{et al.} \cite{srinivasGaussianProcessOptimization2010, srinivasInformationTheoreticRegretBounds2012} integrated BO with bandit algorithm techniques \cite{auerUsingConfidenceBounds2002}, firmly establishing BO as a distinct subfield within machine learning and black-box optimisation.

Bayesian optimisation has been known by several historical names, reflecting the evolution of its theoretical foundations and practical applications. Central to its early implementations is Gaussian process regression \cite{snelsonSparseGaussianProcesses2005}, also referred to as Kriging \cite{krigeStatisticalApproachBasic1951}---a term coined in honour of the South African mining engineer D.~G.~Krige, who first applied the method in geostatistics. Consequently, BO was also commonly referred to as ``Kriging-Based Optimisation'' or ``Gaussian Process Bandit'' in its early stages. Modern BO now encompasses a diverse range of probabilistic surrogate models and stands as a core paradigm within the broader framework of ``\emph{Sequential Model-Based Optimisation}'' (SMBO) \cite{hutterSequentialModelBasedOptimization2011}.

In computer science and artificial intelligence (AI), BO has become a cornerstone of automated machine learning (AutoML) and hyperparameter optimisation (HPO) \cite{snoek2012practical, bergstraHyperoptPythonLibrary2013, cowen-riversHEBOPushingLimits2022}. It also plays a central role in sequential decision-making formulations (e.g., Gaussian process bandits), with a strong focus on principled algorithm design and theoretical regret analysis \cite{srinivasGaussianProcessOptimization2010, srinivasInformationTheoreticRegretBounds2012, vakiliInformationGainRegret2021, caiLowerBoundsStandard2021}. Recent advances in the field have emphasised algorithmic improvements \cite{wangRecentAdvancesBayesian2022}, including mitigating the curse of dimensionality \cite{djolongaHighDimensionalGaussianProcess2013, gonzalez-duqueSurveyBenchmarkHighdimensional2024}, designing more effective acquisition strategies \cite{scottCorrelatedKnowledgeGradient2011, wuParallelKnowledgeGradient2016}, enhancing robustness to noise \cite{daultonRobustMultiObjectiveBayesian2022, cowen-riversHEBOPushingLimits2022}, and integrating more expressive surrogate models \cite{springenbergBayesianOptimizationRobust2016, wistubaFewShotBayesianOptimization2021, mullerPFNs4BOInContextLearning2023, liuLargeLanguageModels2024}.

It was not until the mid-2010s that BO gained traction within the natural sciences. The need to optimise costly and time-consuming experiments made BO a natural choice for automating the search over experimental conditions. A key advantage of BO over heuristic search algorithms lies in its Bayesian modelling framework, which enables it to achieve meaningful results with significantly fewer experimental evaluations. In recent years, BO has been adopted across a diverse range of scientific domains, including chemistry \cite{hasePhoenicsBayesianOptimizer2018, burgerMobileRoboticChemist2020, shieldsBayesianReactionOptimization2021, wuRaceBottomBayesian2024, cao2025spatial} and materials science \cite{macleodSelfdrivingLaboratoryAccelerated2020, liangBenchmarkingPerformanceBayesian2021, leiBayesianOptimizationAdaptive2021, shoyebraihanAcceleratingMaterialDiscovery2024, cao2024active, li2025optimize}. Research in these fields typically focuses on tailoring BO to domain-specific challenges, such as structured or constrained design spaces, and on incorporating scientific prior knowledge (e.g., material descriptors) into the optimisation loop \cite{haseGryffinAlgorithmBayesian2021, xieDomainKnowledgeInjection2023}.

Notably, a critical gap has emerged between the AI and natural science communities in the application and understanding of BO. State-of-the-art BO algorithms, such as TuRBO \cite{erikssonScalableGlobalOptimization2019} and HEBO \cite{cowen-riversHEBOPushingLimits2022}, are often developed and evaluated primarily on AI-centric tasks like HPO, with limited demonstration of their utility across a wide range of real-world scientific problems. 
Conversely, practitioners in the natural sciences often depend on established, user-friendly software packages \cite{cao2026bgolearn, nogueiraBayesianOptimizationPython2014, hasePhoenicsBayesianOptimizer2018}, which can make it difficult to extend existing methods or implement new methodological advances in practice. This disconnect impedes the translation of cutting-edge Bayesian optimization research into real-world scientific progress, ultimately constraining the efficiency and impact of automated experimental discovery.

This tutorial aims to bridge the gap between AI advances and natural sciences, by presenting BO as a fundamental tool for practitioners with diverse backgrounds. Surprisingly, existing textbook and tutorial resources for BO are insufficiently tailored to this purpose. While several high-quality resources exist---including a systematic and detailed textbook \cite{garnettBayesianOptimization2023}, a comprehensive survey \cite{shahriariTakingHumanOut2016}, and a well-regarded technical tutorial \cite{frazierTutorialBayesianOptimization2018}---none are specifically targeted at scientific practitioners, who often have varied technical backgrounds and distinct needs. 

\paragraph{Motivation and Focus of this Tutorial}
Many scientific researchers do not require deep knowledge of BO’s technical intricacies and simply wish to use it as a software tool, while others seek a more thorough understanding of its theoretical foundations. A particular oversight in existing resources is the failure to explain why and how many scientific problems can, in essence, be formulated as optimisation problems; indeed, many scientific practitioners remain unaware that their day-to-day research implicitly involves optimisation. Furthermore, existing tutorials provide little to no code-level guidance on implementing BO, leaving practitioners, such as chemists or materials scientists without extensive programming experience, unable to apply BO to their own problems without developing software from scratch. To address these critical gaps, this tutorial provides structured case studies, hands-on coding instructions, and content stratified by technical complexity, ensuring that BO is accessible and applicable to all scientific practitioners engaged in automated discovery.

\subsection{Organisation of this Tutorial}

This tutorial is structured into five modular parts:

\begin{enumerate}
    \item Part~\ref{part:sci_as_opt} establishes foundational concepts, framing scientific discovery as a black-box optimisation problem, formalising key definitions, and illustrating how real-world scientific challenges map to BO frameworks. It requires only basic scientific literacy and elementary probability knowledge, making it accessible to all readers aiming to grasp why BO is suited to scientific discovery.
    \item Part~\ref{part:surrogate} delves into surrogate models, focusing on Gaussian Processes as the canonical choice for BO, covering kernel design, uncertainty quantification, and hyperparameter adaptation. It includes practical implementations with Scikit-Learn and GPyTorch, targeting readers with foundational machine learning or statistical knowledge to master the core predictive component of BO.
    \item Part~\ref{part:decision} explores acquisition functions that balance exploration and exploitation, alongside their theoretical grounding in Bayesian decision theory. It explains both myopic and non-myopic strategies, equipping readers to select and customise decision-making criteria for specific scientific constraints.
    \item Part~\ref{part:alg} delivers complete algorithmic BO implementations (standard and human-in-the-loop) and discusses advanced technical extensions for scientific settings. It also provides coding instruction of BO, using a high-level BO package such as HEBO \cite{cowen-riversHEBOPushingLimits2022} and Bgolearn \cite{cao2026bgolearn}, or basic modelling packages including Scikit-Learn \cite{scikit-learn} and Gpytorch \cite{gardnerGPyTorchBlackboxMatrixMatrix2018}. It also includes experimental validation against baseline methods to demonstrate BO’s sample efficiency. Additionally, we also present an example of how BO synergises with existing scientific tools such as molecule descriptors.
\end{enumerate}

Each part contains multiple sections targeted at different readers with diverse backgrounds, from experimental scientists seeking actionable tools to method developers pursuing theoretical rigour, with clear learning objectives and flexible navigation paths.

%% file: v2/sci_as_opt.tex
\part{Scientific Discovery as Optimisation Problems}
\label{part:sci_as_opt}

Bayesian optimisation, originally developed in computer science, has become a widely adopted methodology for optimising black-box functions. A natural question arises: how is optimisation fundamentally related to the process of scientific discovery? Somewhat surprisingly, this connection is both profound and pervasive, yet it is often underappreciated by practising scientists. In reality, scientific researchers routinely engage in implicit optimisation---manually adjusting experimental conditions, exploring parameter spaces, and seeking optimal settings---without necessarily recognising that these activities are instances of solving optimisation problems over unknown functions. The cognitive gap lies in the fact that, during the iterative design and execution of experiments, scientists may not explicitly frame their process within the formalism of optimisation theory.

In this part, we systematically elucidate how scientific discovery can be cast as an optimisation problem, specifically in the context of experimental design. By precisely identifying the core optimisation procedures underpinning scientific research, we open the door to leveraging computational tools for automating and accelerating scientific progress. We begin by introducing the foundational elements of optimisation in scientific contexts, catering to readers without prior exposure to optimisation theory. For those interested in a deeper theoretical understanding, we provide a conceptual roadmap for further study.

\section{Optimisation Problems and Algorithms \ELEMENTARY}

Optimisation lies at the heart of both scientific inquiry and technological advancement, providing a rigorous mathematical framework for systematically identifying the best possible solutions under a given set of circumstances. In the context of scientific discovery, optimisation formalises the process of selecting experimental designs, adjusting parameters, and making decisions that maximise the attainment of desired objectives. This section introduces the formal definition of optimisation problems and discusses how scientific experimentation can be reframed as a process of sequential optimisation, ultimately enabling the automation and acceleration of discovery through computational algorithms.

\subsection{The Definition of Optimisation Problems}
\label{sec:sci_as_opt:definition}

Let $x$ be a variable defined on a domain $\mathcal{X}$. An optimisation problem with respect to $x$ seeks to identify a value $x^* \in \mathcal{X}$ that maximises (or minimises) a given objective function. The canonical form of an unconstrained optimisation problem is:

\begin{equation} \label{eq:unconstrained_opt}
\operatorname{maximise}_{x \in \mathcal{X}} f(x)
\end{equation}
where $f: \mathcal{X} \to \mathbb{R}$ is a real-valued objective function defined on the decision (or parameter) space $\mathcal{X}$.

In more general scenarios, additional requirements may be imposed on $x$ in the form of constraints, leading to a constrained optimisation problem:
\begin{equation} \label{eq:constrained_opt}
\begin{aligned}
    & \operatorname{maximise}_{x \in \mathcal{X}} f(x) \\ 
    &\text{subject to} \begin{cases}
        g_i(x) = 0, & i=1,\ldots,m_{e} \\
        g_i(x) \ge 0, & i=m_e+1,\ldots,m \\
    \end{cases}
\end{aligned}
\end{equation}
where $f: \mathcal{X} \to \mathbb{R}$ and each $g_i: \mathcal{X} \to \mathbb{R}$ are real-valued functions defined on $\mathcal{X}$. The set $\{g_i\}_{i=1}^{m_e}$ represents the \emph{equality constraints}, which must be exactly satisfied; $\{g_i\}_{i=m_e + 1}^{m}$ denotes the \emph{inequality constraints}, indicating that some criteria should meet or exceed certain thresholds. The \emph{feasible set} under these constraints is defined as:
\begin{equation}
\mathcal{X}_{g} = \left\{x \in \mathcal{X}\ \middle|\ g_i(x) = 0 \text{ for } i=1,\ldots,m_e;\ g_i(x) \ge 0 \text{ for } i=m_e +1,\ldots, m\right\}.
\end{equation}
The optimisation problem thus seeks $x^* \in \mathcal{X}_g$ maximising $f(x)$.

\subsection{The Latent Optimisation Problem of Scientific Discovery}

From the era of Baconian induction to the hypothetico-deductive (H-D) paradigm, empirical observation has served as the cornerstone of scientific knowledge acquisition. At its core, scientific experimentation involves selecting an action or design $x \in \mathcal{X}$ and observing the resultant outcome $o \in \mathcal{O}$. The underlying mechanism by which nature produces this observation can be formalised as the (unknown) true model $\mathcal{M}^*$:
\begin{equation}
o = \mathcal{M}^*(x).
\end{equation}

Scientific inquiry is motivated not only by the desire to understand the world (\textit{explaining}), but also by the need to utilise this understanding to guide action (\textit{doing}). Broadly, there are three hierarchical types of knowledge that scientific discovery aims to uncover:
\begin{enumerate}
    \item \textbf{Hypothesis-level knowledge} (\textit{what}): Identifying the class of plausible hypotheses of $\mathcal{M}^*$ governing the system under study.
    \item \textbf{Explanatory knowledge} (\textit{why}): Pinpointing which hypothesis of $\mathcal{M}^*$ best explains our experience, i.e., why actions $x$ have produced the observed outcomes $o$.
    \item \textbf{Actionable knowledge} (\textit{how}): Determining how to choose actions $x$ under $\mathcal{M}^*$ to optimise a reward (or utility) function $r(x, o): \mathcal{X} \times \mathcal{O} \to \mathbb{R}$, subject to constraints $\{h_i(x, o)\}_{i=1}^m$ (if any).
\end{enumerate}

Focusing on the ultimate goal of actionable discovery, the scientific process can be formalised as the following optimisation problem:
\begin{equation} \label{eq:sci_discovery_opt}
\begin{aligned}
    & \operatorname{maximise}_{x \in \mathcal{X}} r(x,\mathcal{M}^*(x)) \\ 
    &\text{subject to} \begin{cases}
        h_i(x, \mathcal{M}^*(x)) = 0, & i=1,\ldots,m_{e} \\
        h_i(x, \mathcal{M}^*(x)) \ge 0, & i=m_e+1,\ldots,m \\
    \end{cases}
\end{aligned}
\end{equation}
By defining $f(x) = r(x,\mathcal{M}^*(x))$ and $g_i(x) = h_i(x, \mathcal{M}^*(x))$, we recover the standard constrained optimisation framework~\eqref{eq:constrained_opt}.

Here, the reward function $r$ encodes preferences over design--outcome pairs, such that $r(x_1, o_1) < r(x_2, o_2)$ if $(x_2, o_2)$ is preferred to $(x_1, o_1)$. In the special case where $r$ is constant, the problem reduces to pure explanation-driven discovery---the focus of Bayesian experimental design (BED). This distinction highlights how Bayesian optimisation (BO) extends beyond BED, not only addressing \textit{what} and \textit{why}, but also guiding \textit{how} to act.

It is worth noting that, in some cases, one may define the decision space $\mathcal{X}$ to coincide with the feasible set $\mathcal{X}_g$ induced by the constraints. While this unifies all problems into an unconstrained form, it obscures important distinctions at the hypothesis and explanation levels. In practice, exploring actions beyond the current feasible set can be instrumental in revealing the true nature of $\mathcal{M}^*$. Moreover, constraints $g_i$ may themselves depend on the unknown mechanism, so that the feasibility of an action $x$ is not always known \textit{a priori}.

\subsection{Optimisation Challenges for Scientific Discovery}

Scientific discovery, when formulated as an optimisation problem, presents several unique challenges that distinguish it from classical numerical optimisation~\cite{nocedalNumericalOptimization2006}, such as gradient descent or Newton's method. These challenges include:

\paragraph{Black-box Objectives and Constraints}
In scientific discovery, the objective function $f(x) = r(x, \mathcal{M}^*(x))$ is a \emph{black-box}: the mechanism $\mathcal{M}^*$ is unknown and must be inferred from data. Consequently, optimisation must proceed without analytic expressions for the objective or its derivatives. Furthermore, constraints $g_i(x)$ may also depend on unknown aspects of $\mathcal{M}^*$, necessitating real-world experimentation to determine feasibility~\cite{gramacyOptimizationUnknownConstraints2011}. Scientific discovery thus inherently demands optimisation methods capable of reasoning under epistemic uncertainty and integrating model learning with decision-making.

\paragraph{Partial Observability and Noisy Evaluation}
Experimental outcomes are often subject to measurement noise and influenced by latent, unobserved variables. Thus, the evaluation of $f(x)$ is typically noisy. For instance, additive noise may be included when evaluating the objective:
\begin{equation}
y = f(x) + \varepsilon,
\end{equation}
where $\varepsilon$ denotes stochastic noise that summarizes the influence of variability over experimental outcomes. Effective optimisation algorithms must therefore account for uncertainty in observations and appropriately model the noise structure.

\paragraph{Expensive Experiments}
Scientific experiments are frequently costly in terms of time, resources, or materials. For instance, synthesising a new chemical compound, running a clinical trial, or conducting a high-energy physics experiment may require substantial investment. As a result, sample efficiency---the ability to identify optimal or near-optimal solutions using as few experiments as possible---is paramount.

\subsection{Optimisation Algorithms for Scientific Discovery}

In this tutorial, we focus on Bayesian optimisation, a representative approach that exemplifies sequential decision-making and serves as a direct analogue to Bayesian H-D scientific discovery (Section~\ref{sec:intro:bayesian_discovery}). Here, we conceptualise the core ideas underlying Bayesian optimisation; further details are provided in subsequent sections of this tutorial.

\subsubsection{Optimisation through Sequential Decision Making}

Given the black-box and noisy nature of scientific experimentation, an analytic optimum is typically intractable. Therefore, a practical optimisation strategy operates in a sequential manner. A \emph{sequential optimisation algorithm} is defined by a policy $\pi(\hat{x}|\mathcal{D})$ that, at each iteration, selects the next design $\hat{x}$ based on the accumulated data $\mathcal{D}$ from previous experiments. Typically, $\mathcal{D} = \{(x_{i}, y_i)\}_{i=1}^n$ records $n$ pairs of experimental designs and observed (possibly noisy) rewards. Additional intermediate observations may also be incorporated into the data if available.

Conventional sequential optimisation algorithms, such as those relying on gradient information, are ineffective in the optimisation of scientific discovery due to the absence of analytic derivatives. This motivates the use of \emph{gradient-free} or \emph{derivative-free} methods. Notably, multi-armed bandit algorithms fall within this class, although they often bypass explicit modelling of the underlying mechanism and instead focus on empirical performance. 

\subsubsection{Modelling the Objective}

Conventional bandit algorithms suffer from a key drawback: their exploration is often low-efficiency and somewhat haphazard, which makes them poorly suited for expensive experiments where each evaluation of the objective is costly. To address this, we need methods that explicitly store and update our knowledge about the objective, and then choose actions based on this accumulated knowledge rather than blind or naive exploration. 
With this idea, \emph{sequential model-based optimisation} (SMBO) embodies the scientific method by explicitly constructing and updating a model of the unknown mechanism, using this model to guide experimental design and action selection.

The SMBO procedure typically comprises three steps:
\begin{enumerate}
    \item \textbf{Model selection}: Choosing a class of models (e.g., linear models, neural networks) to represent plausible hypotheses for $\mathcal{M}^*$.
    \item \textbf{Model fitting}: Updating the model parameters based on the data $\mathcal{D}$ collected thus far.
    \item \textbf{Decision-making}: Selecting the next experimental design based on the current model.
\end{enumerate}
This iterative process mirrors the scientific cycle of hypothesis generation (\textit{what}), explanation (\textit{why}), and action (\textit{how}). Bayesian optimisation is the principled Bayesian realisation of SMBO, systematically integrating uncertainty quantification and decision-theoretic principles.

\begin{figure}[tbh]
    \centering
    \includegraphics[width=0.6\linewidth]{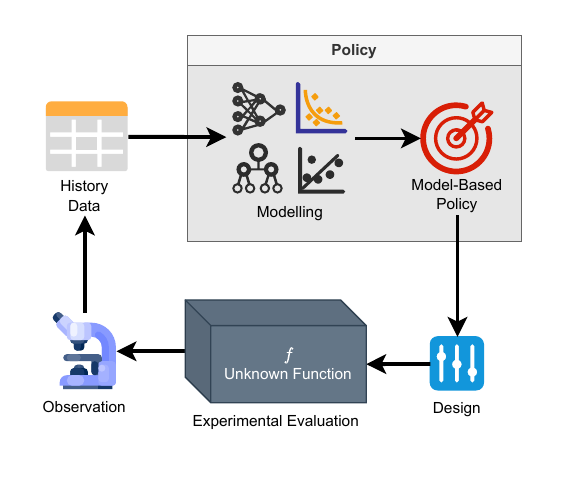}
    \caption{Scientific discovery as sequential model-based optimisation: iteratively updating a model of the unknown mechanism based on experimental data, and using this model to guide the selection of future experiments.}
    \label{fig:smbo_diagram}
\end{figure}

\subsubsection{Incorporating Model Uncertainty: Bayesianism for Optimisation}

Here, we provide the conceptual idea of Bayesian optimisation (BO), as a pivotal variant of SMBO that aligns with the Bayesianism of scientific discovery. What sets BO apart is its combination of Bayesian inference with the modelling of the unknown objective. Instead of relying on a fixed or deterministic view of the problem, BO builds a flexible, probabilistic model of the relationship between choices and outcomes, which aligns with the core principles of Bayesian scientific discovery (Section~\ref{sec:intro:bayesian_discovery}). Every time we try a new option and observe the result, we use Bayes’ rule to update our beliefs about the entire decision space. This approach offers two key advantages:

\begin{itemize}
    \item First, it allows us to use prior knowledge about how different choices might be related: if we learn that one option works well, we can infer that similar options are also likely to be good, even before testing them. This is because BO’s model captures the correlation among different choices of experiments, so each new observation not only teaches us about the chosen setting but also influences our beliefs about related alternatives.
    \item Second, BO actively manages uncertainty. When deciding where to search next, it balances between focusing on areas that already look promising and exploring regions where our knowledge is limited but the potential reward could be high. This active learning and exploration strategy helps BO to efficiently navigate even very large or complex decision spaces, often finding the global best solution with far fewer experiments than random guessing or trial-and-error approaches.
\end{itemize}

In summary, Bayesian optimisation enables us to make the most of every experiment by combining probabilistic reasoning with thoughtful, step-by-step decision-making. By leveraging both what we know and what we have yet to learn, BO offers a practical and effective way to tackle challenging optimisation problems where every trial counts.

\section{Theoretical Concepts of Optimisation \THEORY}
\label{sec:sci_as_opt:theory}

Optimisation theory provides the mathematical underpinnings for systematically identifying the best solution to a problem within a prescribed set of possibilities. In this section, we introduce the fundamental concepts and theoretical results that form the basis of optimisation, with a particular focus on sequential optimisation. We begin by revisiting the definition of a general optimisation problem in Section~\ref{sec:sci_as_opt:definition}.

\begin{definition}[Optimisation Problem]
Let $\mathcal{X} \subseteq \mathbb{R}^d$ denote the \emph{decision space} and $f: \mathcal{X} \to \mathbb{R}$ a real-valued \emph{objective function}. The (unconstrained) optimisation problem seeks
\begin{equation}
    x^* \in \arg\max_{x \in \mathcal{X}} f(x).
\end{equation}
If constraints are present, the problem becomes
\begin{equation}
    \begin{aligned}
        & \text{find } x^* \in \arg\max_{x \in \mathcal{X}_g} f(x) \\
        & \text{where } \mathcal{X}_g = \left\{ x \in \mathcal{X} \mid g_i(x) = 0,~ i=1,\ldots,m_e;~ g_i(x) \ge 0,~ i=m_e+1,\ldots,m \right\},
    \end{aligned}
\end{equation}
with $g_i: \mathcal{X} \to \mathbb{R}$ representing constraint functions.
\end{definition}

A central distinction in optimisation theory is that between local and global optimality. In scientific discovery, the objective is produced from an unknown and sophisticated mechanism. Therefore, it may possess multiple local maxima, and identifying the global optimum is substantially more challenging. To formalise these concepts, we introduce the following definitions.

\begin{definition}[Global Optimum]
A point $x^* \in \mathcal{X}_g$ is a \emph{global maximiser} of $f$ if
\begin{equation}
    f(x^*) \geq f(x), \quad \forall x \in \mathcal{X}_g.
\end{equation}
\end{definition}

\begin{definition}[Local Optimum]
A point $x^\dagger \in \mathcal{X}_g$ is a \emph{local maximiser} of $f$ if there exists $\epsilon > 0$ such that
\begin{equation}
    f(x^\dagger) \geq f(x), \quad \forall x \in \mathcal{X}_g \cap B(x^\dagger, \epsilon),
\end{equation}
where $B(x^\dagger, \epsilon)$ denotes the open ball of radius $\epsilon$ centred at $x^\dagger$.
\end{definition}

For scientific discovery, the objective function $f$ is not known in closed form and can only be evaluated sequentially, often at significant cost. Sequential optimisation algorithms address this by selecting a sequence of query points $(x_1, x_2, \ldots, x_T)$, observing (possibly noisy) outcomes $y_t = f(x_t) + \varepsilon_t$, and updating their strategy based on accumulated data. To rigorously assess the performance of sequential optimisation algorithms, two key metrics are defined: convergence and regret.

\begin{definition}[Convergence]
A sequential optimisation algorithm is said to \emph{converge} to the global maximiser $x^*$ if, as the number of iterations $T \to \infty$, the sequence of iterates $(x_t)$ satisfies
\begin{equation}
    \lim_{T \to \infty} f(x_T) = f(x^*).
\end{equation}
\end{definition}

\begin{definition}[Regret]
Let $x^*$ denote the global maximiser of $f$. The \emph{instantaneous regret} at iteration $t$ is defined as
\begin{equation}
    r_t := f(x^*) - f(x_t).
\end{equation}
The \emph{cumulative regret} after $T$ iterations is
\begin{equation}
    R_T := \sum_{t=1}^{T} r_t = \sum_{t=1}^{T} \left[ f(x^*) - f(x_t) \right].
\end{equation}
\end{definition}

Low cumulative regret indicates that the algorithm rapidly identifies high-performing solutions, making regret a central criterion in the analysis of sequential optimisation methods. These two notions---convergence and regret---are deeply connected, and understanding their relationship is crucial for evaluating the effectiveness of optimisation algorithms.

The following theorem establishes a fundamental link between the convergence of an optimisation algorithm and its cumulative regret.

\begin{theorem}[Convergence Implies Vanishing Average Regret]
Suppose a sequential optimisation algorithm produces a sequence $(x_t)_{t=1}^T$ such that $\lim_{T \to \infty} f(x_T) = f(x^*)$. Then, the average regret vanishes asymptotically:
\begin{equation}
    \lim_{T \to \infty} \frac{1}{T} R_T = 0.
\end{equation}
\end{theorem}

\begin{proof}
By definition, $r_t = f(x^*) - f(x_t) \geq 0$ and, by hypothesis, $f(x_t) \to f(x^*)$ as $t \to \infty$. Thus, for any $\epsilon > 0$, there exists $N$ such that for all $t > N$, $r_t < \epsilon$. Therefore,
\[
R_T = \sum_{t=1}^{T} r_t \leq N \cdot M + (T-N)\epsilon,
\]
where $M = \max_{t \leq N} r_t$. Dividing both sides by $T$ and taking the limit $T \to \infty$, the contribution from the first $N$ terms vanishes, and we obtain $\lim_{T \to \infty} \frac{1}{T} R_T \leq \epsilon$. Since $\epsilon$ is arbitrary, the result follows.
\end{proof}

Conversely, bounded cumulative regret ensures that the algorithm spends most of its iterations near-optimal solutions. This is formalised in the following result.

\begin{theorem}[Sublinear Regret Implies Convergence in Average]
Suppose an algorithm achieves sublinear cumulative regret, i.e., $R_T = o(T)$ as $T \to \infty$. Then, the average performance converges to the optimum:
\begin{equation}
    \lim_{T \to \infty} \frac{1}{T} \sum_{t=1}^{T} f(x_t) = f(x^*).
\end{equation}
\end{theorem}

\begin{proof}
By definition,
\[
\frac{1}{T} \sum_{t=1}^{T} f(x_t) = f(x^*) - \frac{1}{T} R_T.
\]
If $R_T = o(T)$, then $\frac{1}{T} R_T \to 0$ as $T \to \infty$, so the average converges to $f(x^*)$.
\end{proof}

These results provide a rigorous foundation for evaluating and designing sequential optimisation algorithms, particularly in the context of scientific discovery, where efficient identification of optimal experimental designs is paramount. Ideally, we expect a sequential optimisation algorithm to achieve sublinear regret. The interplay between regret minimisation and convergence properties guides both the theoretical analysis and practical implementation of modern optimisation strategies.

\section{Coding an Optimisation Problem \CODE}
\label{sec:sci_as_opt:coding}

In this section, we demonstrate how to implement the problem definition of \eqref{eq:unconstrained_opt} using Python code. We focus on unconstrained optimisation, meaning that the objective function is defined over the entire design space without explicit constraints. Constrained Bayesian optimisation (BO) is a more advanced topic \cite{gardnerBayesianOptimizationInequality2014, pichenyBayesianOptimizationMixed2016} and is beyond the scope of this tutorial.

\subsection{Specifying the Design Space}

Framing scientific discovery as an optimisation problem begins with accurately identifying and representing the design space $\mathcal{X}$---not only mathematically but also in a manner compatible with computational tools. Many BO libraries are designed primarily for hyper-rectangular spaces (i.e., continuous variables bounded within intervals). However, scientific applications often involve heterogeneous design spaces containing variables of various types, such as categorical choices, integers, and booleans, which require more flexible representations.

For this reason, we use the HEBO library \cite{cowen-riversHEBOPushingLimits2022} to illustrate how to define a design space that supports multiple variable types. This section outlines the common variable types encountered in scientific experiments, providing both their mathematical interpretation and realistic examples from domains such as chemistry and materials science. The following Python snippet demonstrates how to define a design space with eight different variable types using HEBO:

\begin{lstlisting}[language=Python]
import pandas as pd
from hebo.design_space.design_space import DesignSpace

# Define the design space with diverse variable types
space = DesignSpace().parse([
    # Continuous numerical interval (e.g., reaction temperature in Celsius)
    {'name': 'x0', 'type': 'num', 'lb': 0, 'ub': 7},
    # Integer interval (e.g., number of catalyst layers)
    {'name': 'x1', 'type': 'int', 'lb': 0, 'ub': 7},       
    # Logarithmically scaled continuous variable (e.g., concentration of a reagent)
    {'name': 'x2', 'type': 'pow', 'lb': 1e-4, 'ub': 1e-2, 'base': 10},
    # Categorical variable (e.g., choice of solvent)
    {'name': 'x3', 'type': 'cat', 'categories': ['a', 'b', 'c']},
    # Boolean variable (e.g., whether to enable a catalyst)
    {'name': 'x4', 'type': 'bool'},
    # Logarithmically scaled integer variable (e.g., number of molecules)
    {'name': 'x5', 'type': 'pow_int', 'lb': 1, 'ub': 10000, 'base': 10},
    # Stepped integer variable (e.g., time intervals in fixed increments)
    {'name': 'x6', 'type': 'step_int', 'lb': 1, 'ub': 9, 'step': 2},
    # Integer exponent variable (e.g., memory size in powers of two)
    {'name': 'x7', 'type': 'int_exponent', 'lb': 1, 'ub': 1024, 'base': 2}, 
])

# Sample 10 points from the space
samples = space.sample(10)
# Sort columns alphabetically for readability
samples = samples.reindex(sorted(samples.columns), axis=1)

# Display information about the space
print("Numeric parameters:", space.numeric_names)
print("Categorical parameters:", space.enum_names)
print("Samples:\n", samples)
\end{lstlisting}

\subsubsection{Numeric Variables}

Numeric variables take values from a continuous real interval $[a, b]$. They are the most common type in BO frameworks and are suitable for parameters that can be controlled continuously.

\begin{example}
In a chemical experiment, the reaction temperature (measured in degrees Celsius) can be modelled as a numeric variable:
\begin{lstlisting}[language=Python]
{'name': 'temperature', 'type': 'num', 'lb': 20, 'ub': 100}
\end{lstlisting}
\end{example}

\paragraph{Logarithmically Scaled Variables}

When parameters span multiple orders of magnitude, such as concentrations or rate constants, it is often more effective to optimise on a logarithmic scale to ensure balanced exploration across scales. HEBO supports this via the \texttt{pow} type, which internally transforms the variable logarithmically.

\begin{example}
Consider the reaction rate constant $k$, which can vary between $10^{-5}$ and $10^{-1}$:
\begin{lstlisting}[language=Python]
{'name': 'reaction_rate', 'type': 'pow', 'lb': 1e-5, 'ub': 1e-1, 'base': 10}
\end{lstlisting}
\end{example}

\subsubsection{Integer Variables}

Integer variables represent discrete numerical values within specified bounds, suitable for parameters that must be integral.

\begin{example}
The number of catalyst layers in a reactor can be represented as:
\begin{lstlisting}[language=Python]
{'name': 'num_layers', 'type': 'int', 'lb': 1, 'ub': 10}
\end{lstlisting}
\end{example}

\paragraph{Logarithmic Integer Variables}

For integer variables spanning orders of magnitude, HEBO provides the \texttt{pow\_int} type, which samples integers on a logarithmic scale.

\begin{example}
The number of molecules added into a material might range from $10^{20}$ to $10^{24}$:
\begin{lstlisting}[language=Python]
{'name': 'num_molecules', 'type': 'pow_int', 'lb': 10**20, 'ub': 10**24, 'base': 10}
\end{lstlisting}
\end{example}

\paragraph{Stepped Integer Variables}

When integer parameters increment in fixed steps (e.g., time intervals of 2 seconds), the \texttt{step\_int} type allows modelling such discretisation.

\begin{example}
Time intervals measured in steps of 2 seconds between 1 and 9 seconds:
\begin{lstlisting}[language=Python]
{'name': 'time_intervals', 'type': 'step_int', 'lb': 1, 'ub': 9, 'step': 2}
\end{lstlisting}
\end{example}

\paragraph{Integer Exponents}

HEBO supports variables that are powers of an integer base, useful for parameters such as memory sizes or batch sizes that grow exponentially.

\begin{example}
Memory size in powers of two, ranging from $2^0=1$ to $2^{10}=1024$:
\begin{lstlisting}[language=Python]
{'name': 'memory_size', 'type': 'int_exponent', 'lb': 1, 'ub': 1024, 'base': 2}
\end{lstlisting}
\end{example}

\subsubsection{Binary and Categorical Variables}

\paragraph{Binary Variables}

Binary variables take values in $\{0,1\}$, representing boolean decisions such as enabling or disabling a catalyst.

\begin{example}
Whether to use a catalyst in a reaction:
\begin{lstlisting}[language=Python]
{'name': 'use_catalyst', 'type': 'bool'}
\end{lstlisting}
Here, \texttt{use\_catalyst = 0} means the catalyst is disabled, and \texttt{use\_catalyst = 1} means it is enabled.
\end{example}

\paragraph{Categorical Variables}

Categorical variables represent unordered discrete options, such as solvent choices or catalyst types. Unlike integers, categories have no inherent numeric ordering.

\begin{example}
Choice of solvent in a reaction:
\begin{lstlisting}[language=Python]
{'name': 'solvent', 'type': 'cat', 'categories': ['Water', 'Ethanol', 'Methanol']}
\end{lstlisting}
\end{example}

\paragraph{Distinction from Integers}

While integers imply an ordered magnitude (e.g., $2 > 1 > 0$), categorical variables treat each category as unique and unordered. For example, selecting the number of cooling stages is naturally an integer variable, whereas choosing between solvents is categorical.

\subsection{Defining and Evaluating the Objective Function}

The objective function $f: \mathcal{X} \to \mathbb{R}$ maps a design or experimental condition $x \in \mathcal{X}$ to a scalar performance metric. This objective function is typically unknown or expensive to evaluate analytically; instead, it is implicitly defined by interacting with an external system such as laboratory experiments, physical processes, or complex simulations via a software interface.

Because the objective function involves real-world evaluations, the implementation of the evaluation function depends on the specific experimental or simulation setup. In HEBO, the objective function is implemented as a Python function that accepts a batch of $m$ design points $\{x_i\}_{i=1}^m \subset \mathcal{X}$, typically represented as a pandas \texttt{DataFrame}, and returns a NumPy array of corresponding scalar values $\mathbf{y} = (y_1, \ldots, y_m) \in \mathbb{R}^m$. This batch interface enables efficient parallel evaluation, which is common in modern scientific workflows.

A general template for such an evaluation function is shown below:

\begin{lstlisting}[language=Python]
import numpy as np
import pandas as pd

def evaluate(xs: pd.DataFrame) -> np.ndarray:
    """
    Evaluate the objective function f on a batch of design points xs.

    Parameters:
        xs (pd.DataFrame): Batch of design points; each row corresponds to one sample.

    Returns:
        np.ndarray: Array of objective values, one per input design.
    """
    ys: np.ndarray 
    # Implementation details depend on the application context:
    # 1. Translate designs xs into commands for experimental equipment or simulators.
    # 2. Execute experiments or simulations, possibly in parallel.
    # 3. Collect and preprocess results to compute ys.
    ...
    return ys
\end{lstlisting}

\paragraph{Manual Evaluation with User Input}

In some cases, automatic data acquisition from experiments or simulations may be unavailable. A practical workaround is manual evaluation, where the user performs the experiment for each design point and inputs the observed results via the console. This approach maintains the batch interface required by HEBO while incorporating human-in-the-loop evaluation.

\begin{lstlisting}[language=Python]
import numpy as np
import pandas as pd

def evaluate(xs: pd.DataFrame) -> np.ndarray:
    ys = np.zeros(len(xs))

    for i in range(len(xs)):
        x = xs.iloc[i]
        print(f"Please perform the experiment with design x =\n{x}\n")

        y = None
        while y is None:
            try:
                y = float(input("Enter the observed result y: "))
            except ValueError:
                print("Invalid input, please enter a numeric value.")

        ys[i] = y
    return ys
\end{lstlisting}

\paragraph{Evaluation Using a Mock Oracle}

Because real-world evaluations can be costly or time-consuming, it is often beneficial to use a \emph{mock oracle} during development and testing. A mock oracle is a surrogate function that mimics the behaviour of a true black-box function, typically implemented as a simple analytic expression or a fitted model. This enables rapid iteration, debugging, and validation of the optimisation pipeline without incurring real evaluation costs.

For example, consider a design space with one continuous variable and one integer variable:

\begin{lstlisting}[language=Python]
import numpy as np
import pandas as pd
from hebo.design_space.design_space import DesignSpace

# Define the design space with a continuous variable x0 and an integer variable x1.
space = DesignSpace().parse([
    {'name': 'x0', 'type': 'num', 'lb': -3, 'ub': 3},
    {'name': 'x1', 'type': 'int', 'lb': 0, 'ub': 5},
])

def evaluate(xs: pd.DataFrame) -> np.ndarray:
    x0 = xs["x0"]
    x1 = xs["x1"]
    # Example oracle combining quadratic, sinusoidal, and linear terms
    ys = x0**2 + 0.5 * np.sin(2 * np.pi * (x1 - 0.2)) + 0.25 * x0
    return ys.values  # Return as numpy array
\end{lstlisting}

More sophisticated mock oracles can be built by fitting regression models to an empirical database, providing realistic proxies for expensive black-box functions (see Section~\ref{sec:sci_as_opt:examples}).

\section{Case Studies \SCI\,\CODE}
\label{sec:sci_as_opt:examples}

These case studies illustrate how diverse scientific discovery problems can be formulated as optimisation tasks, each presenting unique challenges such as high-dimensional continuous parameter spaces, simplex design spaces, mixed variable types, and high-dimensional feature representations. Since this tutorial is mainly for demonstrative purposes, we define the optimisation problem mainly with mock oracle evaluation fitted from existing databases, thereby avoiding the real-world evaluation cost. However, the principles demonstrated here can also be applied to problems that are evaluated on real-world experiments. Most importantly, through these examples, we see how objective-directed scientific discovery is formulated as an optimisation problem, not only mathematically but also through engineering of practical challenges. This formulation stage is critical for enabling effective application of optimisation techniques by translating complex scientific questions into well-posed mathematical problems amenable to systematic exploration and solution.

To provide a preliminary overview of the performance advantages of Bayesian optimisation (BO) for scientific discovery tasks, we summarise the final optimisation outcomes across all case studies in Table \ref{tab:case_study_preview}. These results serve as a preview of the efficacy of BO frameworks. We observe that BO consistently outperforms random trials across all tasks, achieving superior objective values with the same fixed experimental budget. A detailed description of the experimental setup, full optimisation trajectories, and in-depth performance analysis is provided in Section~\ref{sec:exp:results}.

\begin{table}[htbp]
\centering
\caption{Preview of final optimisation performance across scientific discovery case studies}
\label{tab:case_study_preview}
\resizebox{\textwidth}{!}{
\begin{tabular}{lccccc}
\toprule
Case Name & Section & Total Trials & Performance Measure & \multicolumn{2}{c}{Best Solution Performance} \\
\cmidrule(lr){5-6}
& & & (Maximise / Minimise) & Random Search & BO \\
\midrule
Photocatalytic HER Catalyst & \ref{sec:sci_as_opt:examples:HER} & 200 & Regret (minimise) & 21.8 & 9.7 \\
HEA Nanozyme Formulation & \ref{sec:sci_as_opt:examples:HEA} & 200 & Regret (minimise) & 244 & 155 \\
OER Electrocatalyst Design & \ref{sec:sci_as_opt:examples:OER} & 200 & Overpotential (mV, minimise) & 245 & 219 \\
Buchwald-Hartwig Reaction & \ref{sec:sci_as_opt:examples:BH} & 200 & Regret (minimise) & 47 & 9.9 \\
Molecular QED Optimisation & \ref{sec:sci_as_opt:examples:QED} & 100 & QED Score (maximise) & 0.881 & 0.918 \\
\bottomrule
\end{tabular}
}
\end{table}

\subsection{Catalyst Design for Photocatalytic Water Splitting}
\label{sec:sci_as_opt:examples:HER}


\begin{center}
\CODE\,\url{https://github.com/zwyu-ai/BO-Tutorial-for-Sci/blob/main/examples/HER}
\end{center}

We begin with a straightforward optimisation problem in chemistry that aligns well with the BO workflow described above. The challenge is to design efficient catalysts for photocatalytic water splitting \cite{nishiokaPhotocatalyticWaterSplitting2023a}, a process that produces hydrogen fuel from water using sunlight. The objective is to maximise the \emph{hydrogen evolution rate} (HER)---the amount of hydrogen generated per unit time. The efficiency of this process is highly sensitive to the catalyst’s composition, which typically consists of a mixture of several materials in varying proportions.

In this example, we aim to discover \textit{the proportions of catalyst materials for maximising HER}. To formalise this challenge, we need to identify the objective (HER), the design variables (the amounts of each catalyst material), and the space of the design variables. As our first case, we focus on illustrating the fundamental principle of BO, and the problem setting is much simplified compared to practical challenges. Specifically, the catalyst design problem can be expressed mathematically as:
\begin{equation}
\begin{aligned}
    \operatorname{maximise}_{x_1,\ldots,x_k}\quad & \mathrm{HER}(x_1, x_2, \ldots, x_k) \\
    \text{w.r.t. design space:}\quad & x_i \in [0, h_i], \quad i = 1, \ldots, k.
\end{aligned}
\end{equation}
Here, $x_i$ denotes the amount of the $i$-th material, and $h_i$ its maximum allowable value.

In this example, the design space comprises $k = 10$ parameters, each corresponding to the concentration or amount of a specific material in the catalyst: \textit{L-Cysteine, Acid Red 87, Methylene Blue, NaCl, NaOH, P10-MIX1, Polyvinylpyrrolidone (PVP), Rhodamine B, Sodium dodecyl sulfate (SDS), and Sodium silicate}. These materials are commonly used in photocatalytic systems and contribute to catalyst activity or stability. For instance, organic dyes such as Acid Red 87, Rhodamine B, and Methylene Blue serve as photosensitisers to extend light absorption. L-Cysteine acts as a sacrificial electron donor, promoting charge separation and hydrogen evolution. NaCl and NaOH regulate ionic strength and pH, which strongly influence performance. Additives like PVP and SDS function as surfactants affecting catalyst particle dispersion and morphology, while sodium silicate may stabilise or structure the catalyst. Proprietary mixtures (e.g., P10-MIX1) can provide additional or synergistic effects tailored to specific systems.

For this illustration, each parameter varies continuously within $[0, 5]$ (in appropriate units), resulting in a ten-dimensional optimisation space. Even if each parameter is discretised coarsely with step size $1$, the number of candidate compositions reaches $6^{10}$---an astronomically large search space for exhaustive search. To implement the problem as an optimisation problem, we use the HEBO interface to construct the design space. For this tutorial, we fit a \emph{random forest model} from an empirical dataset to construct a mock oracle evaluation.

\subsection{High-Entropy Alloy Formulation Optimisation as Nanozymes} 
\label{sec:sci_as_opt:examples:HEA}

\begin{center}
\CODE\,\url{https://github.com/zwyu-ai/BO-Tutorial-for-Sci/blob/main/examples/HEA}
\end{center}

As potent alternatives to natural horseradish peroxidase (HRP), peroxidase (POD)-mimicking nanozymes hold great promise in biomedicine. This example focuses on optimising high-entropy alloy (HEA) nanozymes. The catalytic performance metrics (e.g., $V_{\max}$ and $K_M$) of HEAs are nonlinearly influenced by various physicochemical parameters, including elemental composition, particle size, crystal structure (FCC/BCC/HCP), and surface morphology. Since these micro-features determine the electronic structure and active site exposure, and the parameter space is vast, traditional trial-and-error methods are inefficient. Here, we hope to efficiently explore this space and identify \emph{the composition--structure combinations that maximise catalytic efficiency}.

In this case study, we consider HEAs composed of five distinct metal elements: Co, V, Mn, Cu, and Fe. The goal is to optimise the elemental ratios to maximise $V_{\max}$ and minimise $K_M$, thereby maximising the overall catalytic efficiency ($E$) for the conversion of hydrogen peroxide into hydroxyl radicals. Because the decision variables represent compositional proportions, they are subject to the summation constraint that all ratios sum to 1. Additionally, to ensure the formation of a true five-component HEA, each element is constrained within the range $x_i \in [0.05, 0.35]$. Consequently, the optimisation problem is formulated as follows:

\begin{equation}
\begin{aligned}
    \operatorname{maximise}_{x_{Cu}, \ldots, x_{Co}} \quad & E(x_{Cu}, x_{Fe}, x_{Mn}, x_{V}, x_{Co}) \\
    \text{w.r.t. design space:} \quad & x_{Co} + x_{Cu} + x_{Mn} + x_{Fe} + x_{V} = 1, \\
    & 0.05 \leq x_i \leq 0.35, \quad i \in \{Co, Cu, Mn, Fe, V\}.
\end{aligned}
\label{eq:sci_as_opt:hea}
\end{equation}

\paragraph{Handling Bounded-Simplex Design Space} A challenge that impedes direct BO application to this problem is handling the simplex space, i.e., the sum of design variables must be $1$. As introduced in Section~\ref{sec:sci_as_opt:coding}, this simplex space is not directly supported by most software. To proceed, we seek to define an equivalent problem to the original problem \eqref{eq:sci_as_opt:hea}, such that the design spaces of the two problems can be invertibly transformed. More precisely, we can map the bounded simplex into a hyper-rectangle via an invertible reparameterisation transform $\phi: x \in \mathcal{X} \mapsto z \in [0,1]^4$, such that $\phi(\mathcal{X}) = [0,1]^{4}$. Then, the original problem is converted to
\begin{equation}
\operatorname{maximise}_{z \in [0,1]^4} \quad E(\phi^{-1}(z)),
\label{eq:sci_as_opt:hea_transformed}
\end{equation}
which has a rather simple design space as a 4-dimensional hypercube. We refer to Section~\ref{sec:sci_as_opt:simplex_transform} as one possible implementation of the mapping $\phi$. Given an empirical database, a random forest model or an AutoGluon model is fitted as the mock oracle.

\subsubsection{Reparameterisation Transform of the Bounded Simplex Space}
\label{sec:sci_as_opt:simplex_transform}

Consider the original feasible set
\[
\mathcal{X} = \left\{ x \in \mathbb{R}^n \mid \mathbf{1}^\top x = 1, \quad l \preceq x \preceq h \right\},
\]
where $n > 1$ and the element-wise bounds satisfy $\mathbf{0} \preceq l \prec h \preceq \mathbf{1}$. We note that the set $\mathcal{X}$ is nonempty if and only if
\[
\sum_{i=1}^n l_i \leq 1 \leq \sum_{i=1}^n h_i.
\]
Below, we derive the desired bijection $\phi$, such that $\phi(\mathcal{X}) = [0, 1]^{n - 1}$.

Define the tail sums for indices $k = 1, \ldots, n$ as
\[
L_k = \sum_{i=k}^n l_i, \quad H_k = \sum_{i=k}^n h_i,
\]
noting that $L_1 = \sum_{i=1}^n l_i$ and $H_1 = \sum_{i=1}^n h_i$.

To sequentially characterise feasible intervals for each component $x_k$, define the remaining sum before choosing $x_k$ as
\[
T_k = 1 - \sum_{i=1}^{k-1} x_i, \quad (T_1 = 1).
\]
For $k = 1, \ldots, n-1$, the feasible interval for $x_k$ that ensures feasibility of the remaining components is
\[
x_k \in [a_k, b_k], \quad \text{where} \quad
a_k = \max\left(l_k,\, T_k - H_{k+1}\right), \quad
b_k = \min\left(h_k,\, T_k - L_{k+1}\right).
\]
By construction, these intervals are nonempty under the global feasibility condition. The last component is uniquely determined as
\[
x_n = T_n \in [l_n, h_n].
\]

\paragraph{Forward map $\phi$} For any $x \in \mathcal{X}$, define $z = \phi(x) \in \mathbb{R}^{n-1}$ by
\[
z_k = \frac{x_k - a_k}{b_k - a_k}, \quad k = 1, \ldots, n-1.
\]
Since $x_k \in [a_k, b_k]$, each $z_k \in [0,1]$, so $\phi(\mathcal{X}) \subseteq [0,1]^{n-1}$.

\paragraph{Inverse map $\phi^{-1}$} Given $z \in [0,1]^{n-1}$, recover $x = \phi^{-1}(z)$ recursively. Starting with $T_1 = 1$, for $k = 1, \ldots, n-1$ compute:
\begin{gather*}
a_k = \max\left(l_k,\, T_k - H_{k+1}\right), \quad
b_k = \min\left(h_k,\, T_k - L_{k+1}\right), \\
x_k = a_k + z_k (b_k - a_k), \quad
T_{k+1} = T_k - x_k.
\end{gather*}
Finally, set $x_n = T_n$.

By construction, $x \in \mathcal{X}$, and the maps $\phi$ and $\phi^{-1}$ are inverses, establishing a bijection between $\mathcal{X}$ and the hypercube $[0,1]^{n-1}$. This reparameterisation enables efficient Bayesian Optimisation over the constrained simplex domain by transforming it into a standard hyper-rectangular space.

\subsection{Catalyst Design for Electrolytic Water Splitting}
\label{sec:sci_as_opt:examples:OER}

\begin{center}
\CODE\,\url{https://github.com/zwyu-ai/BO-Tutorial-for-Sci/blob/main/examples/OER}
\end{center}

The urgent need for carbon emission reduction has catalysed global interest in sustainable energy technologies. Among these, hydrogen generation via water electrolysis stands out for its potential to utilise surplus renewable electricity, offering a zero-emission and high-efficiency pathway. Despite its promise, the penetration of green hydrogen remains minimal. The principal bottleneck is economic: green hydrogen production is 2--3 times more costly than conventional methods, primarily due to the high expense and suboptimal efficiency of electrocatalysts employed in the \emph{oxygen evolution reaction} (OER). Traditional catalyst discovery relies on empirical intuition and extensive trial-and-error, with a search space exceeding 30 dimensions and combinatorial possibilities in the hundreds of millions. Thus, there is a critical need for efficient, data-driven approaches to accelerate electrocatalyst discovery.

The OER process, which governs water electrolysis efficiency, is typically assessed by the \textit{overpotential} $\eta$ at a fixed current density (10 mA cm$^{-2}$). The optimisation task is to minimise $\eta$ as a function of high-dimensional experimental parameters, formalised as:
\[
\text{minimise}\ \eta = f(x)
\]
where $x$ encompasses both categorical variables (e.g., metal types, support materials) and continuous variables (e.g., annealing temperature, catalyst loading, proton concentration), as well as process parameters such as stirring and washing conditions.

\paragraph{Multiple-objective Extension}
Although this is beyond the scope of this tutorial, we would also like to point out some practical concerns in scientific experiments. In practical catalyst development, multiple objectives are involved, such as activity, stability, and cost. The multi-objective optimisation (MOO) problem is formulated as:
\[
\max_{x \in X} \left[ -\eta(x),\, \mathrm{Stability}(x),\, -\mathrm{Cost}(x),\, \mathrm{TOF}(x) \right].
\]
A simplistic approach is to scalarise the objectives through a weighted sum, so that the problem reverts back to the single-objective case. However, this potentially causes issues because some objectives compete, and assigning proper weights over sub-objectives becomes rather difficult. Therefore, MOO usually seeks to approximate the Pareto frontier, representing solutions where no objective can be improved without compromising others. To approach this, multi-objective BO \cite{zhangExpensiveMultiobjectiveOptimization2010, belakariaMaxvalueEntropySearch2019, yangMultiObjectiveBayesianGlobal2019} is a popular research topic, providing more advanced BO tools.

\paragraph{Evaluation with a Literature-Based Mock Oracle}
For this case, we curated a comprehensive database of over 1600 OER catalyst records from published literature, capturing essential variables such as precursor types, modification sites, synthesis protocols, and process conditions. Initially, the dataset undergoes rigorous preprocessing, including target standardisation (overpotential at 10~mA~cm$^{-2}$), duplicate and outlier removal (IQR method), categorical value normalisation, and train-test partitioning ($80\% : 20\%$). Then, a random forest model is trained for the mock oracle evaluation. Feature importance analysis highlights catalyst loading (14.9\%), primary metal amount (11.3\%), and annealing temperature (9.0\%) as key determinants of OER performance.

\subsection{High-Dimensional Organic Synthesis}
\label{sec:sci_as_opt:examples:BH}

\begin{center}
\CODE\,\url{https://github.com/zwyu-ai/BO-Tutorial-for-Sci/blob/main/examples/BH}
\end{center}

Synthetic chemistry is foundational to advances in drug discovery, materials science, and energy research. The discovery of novel reactions has historically driven the field forward. Synthetic organic chemistry relies on the systematic discovery of \emph{reaction conditions to maximise key outcomes} such as yield, enantiomeric excess (\textit{ee}), catalytic efficiency (TON/TOF), or cost-effectiveness. However, the combinatorial explosion of possible reagents, solvents, and parameters makes manual trial-and-error approaches inefficient and often biased, especially as modern reaction spaces become increasingly high-dimensional.

Traditional paradigms---relying on mechanistic hypotheses, manual trial-and-error, and expert intuition---face critical limitations in efficiency, scalability, and innovation, particularly as reaction spaces become high-dimensional and combinatorially vast \cite{tanAIMolecularCatalysis2025, gomezDecisionMakingMedicinal2018, copelandSerendipityScienceDiscovery2019}. Manual approaches cover only a tiny fraction of possible conditions, leading to low sample efficiency and path dependence. This has contributed to a declining rate of novel reaction discovery despite an increasing number of reported reactions \cite{provostDataScienceIts2013, liCurrentComplexityTool2015, szymkucOrganicChemistryReally2021}. Overcoming these barriers requires systematic, data-driven methods for exploring complex chemical spaces.

In this section, we use the Buchwald--Hartwig (BH) high-throughput experimentation (HTE) dataset to demonstrate the optimisation problem underlying organic synthesis reactions. \textcite{torresMultiObjectiveActiveLearning2022} provides a dataset with 1,728 unique reaction conditions described by 530 DFT-derived features.
At the current stage, we consider the problem of \textit{directly finding the best-performing DFT features that lead to the highest yield}.\footnote{We note that this setting may not be valid in practice, since we cannot directly control the DFT features but only the reaction conditions. However, in this example we merely aim to illustrate a mathematical latent optimisation problem, without considering these practical constraints. We will leave the direct optimisation in a chemistry space to Section~\ref{sec:sci_as_opt:examples:QED}.}
Therefore, based on this dataset, we frame the optimisation problem as
\[
\text{maximise}_{\mathbf{x} \in \mathbb{R}^{530}} \text{Yield}(\mathbf{x})
\quad\text{subject to}\ \mathbf{a} \preceq \mathbf{x} \preceq \mathbf{b},
\]
where $\mathbf{a}, \mathbf{b} \in \mathbb{R}^{530}$ are the lower and upper bounds of the features derived from the dataset, respectively.

Above, we present an optimisation problem in a high-dimensional setting. However, jointly optimising over 530 features is impractical for a black-box objective without any prior knowledge of latent structure or gradient information. High dimensionality is a universal challenge for any black-box optimisation algorithm, as it prohibits effective exploration, especially when scientific experiments are expensive.

\paragraph{Dimensionality Reduction}
To overcome the curse of high dimensionality, in this case, we employ simple dimensionality reduction techniques to make this optimisation tractable for practical algorithms. In other words, we optimise only the features that have high importance and ignore other unimportant features (e.g., assuming that they are randomly set or fixed at average values). To implement this, we first fit a random forest model that naturally extracts feature importance, and then optimise over the top 20 important features. Then, we fit another random forest model using these selected features as the mock objective $\overline{\text{Yield}}$ to be optimised, where the rest of the features contribute to evaluation noise. That is, the optimisation problem becomes
\[
\text{maximise}_{\mathbf{x}_{selected}} \overline{\text{Yield}}(\mathbf{x}_{selected})
\quad\text{subject to}\ \mathbf{a}_{selected} \preceq \mathbf{x}_{selected} \preceq \mathbf{b}_{selected}.
\]
For this reduced problem, we also learn a random forest to behave as the mock oracle, where the ignored variables are incorporated as prediction noise of this oracle.

For more advanced solutions to high dimensionality, we refer readers to high-dimensional BO (HDBO) \cite{gonzalez-duqueSurveyBenchmarkHighdimensional2024}, which addresses the challenge by identifying a low-dimensional or sparse latent structure underlying the high-dimensional design space.

\subsection{Discovering the Best-Performing Molecule}
\label{sec:sci_as_opt:examples:QED}

In many scientific domains, optimisation must be performed over non-conventional spaces such as molecules, which are typically represented by discrete, structured objects like SMILES strings rather than standard numerical vectors. This setting poses unique challenges for Bayesian optimisation (BO), since classical surrogate models and acquisition functions are primarily designed for continuous vector spaces and struggle to handle large combinatorial discrete spaces with complex structural constraints.

Formally, let $\mathcal{X}$ denote the space of valid SMILES strings representing chemically feasible molecules. The optimisation problem can be expressed as
\begin{equation}
    \max_{x \in \mathcal{X}} f(x),
\end{equation}
where $f: \mathcal{X} \to \mathbb{R}$ is a black-box objective function that assigns a performance score to each molecule $x$. For example, $f$ could be the quantitative estimate of drug-likeness (QED) score \cite{bickertonQuantifyingChemicalBeauty2012}, which evaluates the drug-like properties of a molecule. Considering all strings satisfying the SMILES syntax is intractable. In practice, it is common to restrict $\mathcal{X}$ to a large but finite subset $\mathcal{X}_{\mathrm{sub}} \subset \mathcal{X}$ extracted from curated databases such as ZINC \cite{irwinZINC20AFreeUltralargeScale2020}.

The core challenges in this domain are:
\begin{itemize}
    \item \textbf{Discrete and Structured Decision Space}: Unlike continuous vector spaces, molecules are discrete objects with rich structural constraints (e.g., valence, ring closures) that must be respected to ensure chemical validity.
    \item \textbf{Huge Search Space}: The cardinality of $\mathcal{X}$ is enormous, making exhaustive search infeasible and requiring highly sample-efficient optimisation strategies.
    \item \textbf{Leveraging Domain Knowledge}: Effective optimisation must incorporate chemical domain knowledge, making use of the key properties inferred from the molecule SMILES strings.
\end{itemize}

%% file: v2/surrogate.tex
\part{Understanding the Black Box}
\label{part:surrogate}

As discussed in Part~\ref{part:sci_as_opt}, we view scientific discovery as a sequential process of optimising an objective function $f: \mathcal{X} \to \mathbb{R}$. This function is an unknown black box, requiring costly real-world evaluations in terms of both resources and time. Consequently, at each step $t$, the design $x_{t} \in \mathcal{X}$ must be judiciously selected based on current knowledge of $f$. To formalise this knowledge, the Sequential Model-Based Optimisation (SMBO) framework employs a surrogate model that approximates $f$ using data from previous trials, $\mathcal{D}_{t-1} := \{(x_i, y_i)\}_{i=1}^{t-1}$, where $y_{i} \approx f(x_i)$ is a noisy observation from the $i$-th experiment. In this part, we address the crucial challenge of how to properly establish such a surrogate model. 

\section{Surrogate Model \ELEMENTARY}

In the context of black-box optimisation, a surrogate model serves as a computationally efficient proxy for the true objective function $f$. Instead of directly querying $f$ (which is costly and noisy), we use the surrogate to make predictions, guide experimental design, and quantify uncertainty. Figure~\ref{fig:surrogate_illustration} illustrates how a surrogate model approximates the true objective function (one-dimensional) in black-box optimisation. Conceptually, the surrogate model predicts the objective utilising the noisy observation data, while also quantifying the uncertainty of its prediction. As more observations are incorporated, the surrogate’s predictions become more accurate and its uncertainty decreases, especially near observed data.

This section introduces the fundamental concepts behind surrogate modelling, starting from the definition of the hypothesis space, motivating the use of probabilistic models, and culminating in Bayesian inference techniques. These foundations are essential for understanding how sequential model-based optimisation leverages surrogate models to accelerate scientific discovery and decision making.

\begin{figure}[tbh]
    \centering
    \includegraphics[width=1.0\linewidth]{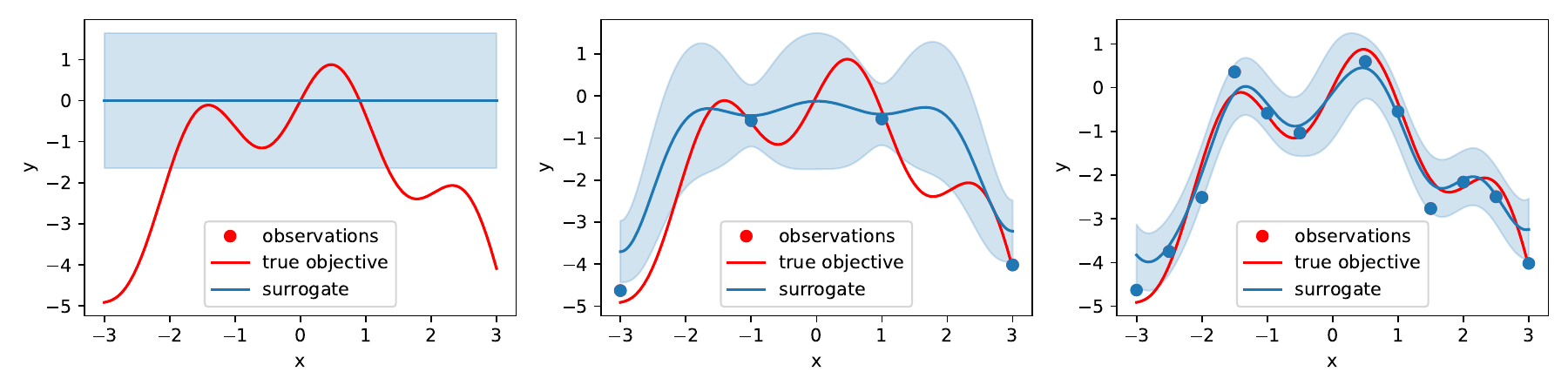}
    \caption{
        Illustration of surrogate modelling in black-box optimisation. The red curve shows the true objective function, and red dots indicate observed data points that are slightly perturbed from the true objectives. The blue curve represents the surrogate model's predicted mean, while the shaded region denotes the 95\% confidence interval, quantifying the model's uncertainty about the objective.
    }
    \label{fig:surrogate_illustration}
\end{figure}

\subsection{Hypothesis Space}

The first step is to \emph{define a space of hypothesized models $\mathcal{H}$} (i.e., the hypothesis space) that can approximate the true objective $f$. As discussed in Part~\ref{part:sci_as_opt}, experimental evaluations of $f$ are typically noisy, so our hypotheses must also incorporate assumptions about this noise. Mathematically, for a hypothesis $h \in \mathcal{H}$, we consider the joint distribution $p(f_{(x)}, y|x, h)$, where $f_{(x)}$ denotes the denoised objective value at $x$. We use the subscript notation $(x)$ to distinguish $f_{(x)}$ as a value, in contrast to $f(x)$ as a function.

\begin{example}[Linear-Function Hypothesis Space]
Suppose we use a linear model, $y = \omega^\top \phi(x) + \varepsilon$, where $\phi(x) \in \mathbb{R}^d$ is a feature representation of $x$ and $\varepsilon$ is i.i.d. Gaussian noise, $\mathcal{N}(0, \sigma^2)$. Assuming that $\sigma^2$ is known in advance, the hypothesis space is then the weight space $\mathcal{H} = \mathbb{R}^d$. The joint distribution is
\[
p(f_{(x)}, y|x, \omega) = \begin{cases}
    0, & f_{(x)} \ne \omega^\top \phi(x); \\
    \mathcal{N}(y - f_{(x)} \mid 0, \sigma^2), & f_{(x)} = \omega^\top \phi(x).
\end{cases}
\]
\end{example}

\begin{example}[Parametric Hypothesis Space]
Consider a general parametric function $F(x, \theta)$, where $\theta$ are model parameters belonging to a parameter space $\Theta_f$. The noise $\varepsilon$ follows a known distribution $p_\varepsilon(\varepsilon|x)$. Then, the hypothesis space is then given by $\mathcal{H} := \Theta$, and the joint distribution is
\[
p(f_{(x)}, y|x, \theta_f, \theta_{\varepsilon}) = \begin{cases}
    0, & f_{(x)} \ne F(x, \theta_f); \\
    p_\varepsilon(y - f_{(x)} | x), & f_{(x)} = F(x, \theta_f).
\end{cases}
\]
\end{example}

\begin{example}[Function Hypothesis Space]
Alternatively, we may hypothesize directly over the space of functions $\mathcal{X} \to \mathbb{R}$. Here, each hypothesis $h$ is a function $\tilde{f}$, with additive noise $\varepsilon$ whose distribution $p_\varepsilon(\varepsilon|x)$ is known. Consequently, the surrogate space is the function space $\mathcal{H} = \mathbb{R}^{\mathcal{X}}$ (the set of functions from the design space $\mathcal{X}$ to real numbers $\mathbb{R}$). The Bayesian surrogate model infers $\tilde{f}$ as a random function in the function space $\mathbb{R}^{\mathcal{X}}$. The joint distribution is:
\[
p(f_{(x)}, y|x, \tilde{f}, \theta_{\varepsilon}) = \begin{cases}
    0, & f_{(x)} \ne \tilde{f}(x); \\
    p_{\varepsilon}(y - f_{(x)} | x), & f_{(x)} = \tilde{f}(x).
\end{cases}
\]
\end{example}

The choice of hypothesis space determines both the representational power of the surrogate model and its computational efficiency. Two key considerations should guide the selection of $\mathcal{H}$:

\begin{itemize}
    \item \textbf{Practicality:} The hypothesized models should be computationally tractable with available resources. Fitting and evaluating the surrogate must be much cheaper than evaluating the true objective $f$. For instance, if each real experiment takes only 30 seconds, it would be impractical to spend several minutes training a deep neural network as the surrogate.
    \item \textbf{Prior Knowledge:} The hypothesis space should reflect any prior knowledge about the true underlying mechanism $\mathcal{M}^*$ of $f$. While we may not know the exact mechanism, domain knowledge often provides useful priors (e.g., smoothness). For example, if $x, x' \in \mathcal{X}$ are close, then $f(x)$ and $f(x')$ are likely to be similar. The hypothesis space should encode such properties.
\end{itemize}

\subsection{Motivation for Probabilistic Surrogate Models}

A straightforward approach is to select a model $h \in \mathcal{H}$ that best explains the observed data $\mathcal{D}_{t-1}$. Formally,
\[
\hat{h} = \arg \min_{h} \mathcal{L}(h;\mathcal{D}_{t-1}),
\]
where $\mathcal{L}$ is a loss function. A common choice is the least-squares loss:
\[
\mathcal{L}(h; \mathcal{D}_{t-1}) = \sum_{i=1}^{t-1} (y_i - \mathbb{E}_h[f_{(x_i)}])^2.
\]
This approach reflects the Hypothetico-Deductive paradigm (see Section~\ref{sec:intro}), where knowledge of the unknown objective is represented by a single hypothesis that best fits the empirical data.

However, this deterministic approach has significant limitations: it assumes the true objective $f$ lies within the representational capacity of $\mathcal{H}$, which is rarely the case. Nature's mechanisms may be arbitrarily complex, and our hypothesis space can never fully capture all possibilities. Thus, no single hypothesis can be perfectly accepted or rejected based on finite data. This motivates a probabilistic view, as discussed in Section~\ref{sec:intro:bayesian_discovery}: our knowledge about $f$ is better represented as a belief distribution over $\mathcal{H}$. We formalize this as follows:

\begin{definition}[Probabilistic Surrogate Model]
A probabilistic surrogate model, denoted $\mathcal{M}(\hat{\mathbf{f}}|\hat{\mathbf{x}}, \mathcal{D})$, maps a finite sequence of $m$ designs $\hat{\mathbf{x}} \in \mathcal{X}^m$---given the dataset $\mathcal{D} = \{(x_i, y_i)\}_{i=1}^n$ of $n$ design-observation pairs---to the joint distribution of their (denoised) objective values $\hat{\mathbf{f}} := (\hat{f}(\hat{x}_i))_{i=1}^m$.
\end{definition}

\subsection{Bayesian Inference of the Surrogate Model \THEORY}
\label{sec:surrogate:bayes_inference}

In this section, we present the statistical rationale underlying surrogate models, establishing the crucial link between Bayesian optimisation and the philosophy of scientific discovery. Some background in probability theory is assumed; see the preliminaries for details. While this section is foundational for understanding the principles of Bayesian Optimisation (BO), it is not strictly necessary for those who only wish to use BO as a software tool.

Let $\mathbf{x} = (x_i)_{i=1}^n \in \mathcal{X}^n$ and $\mathbf{y}=(y_i)_{i=1}^n$ denote the sequences of designs and corresponding observations. Applying Bayes' rule, we have
\begin{equation}
    p(h|\mathcal{D}) = \frac{p(h, \mathcal{D})}{p(\mathcal{D})} =  \frac{p(\mathbf{y}|h, \mathbf{x})p(h)}{p(\mathbf{y}|\mathbf{x})} \propto p(\mathbf{y}|h, \mathbf{x})p(h),
    \label{eq:surrogate:bayes_belief_hypothesis}
\end{equation}
where $p(\mathbf{y}|h, \mathbf{x})$ marginalises over the latent objective values:
\[
p(\mathbf{y}|h, \mathbf{x}) = \int_{\mathbb{R}^n} p(\mathbf{y}, \mathbf{f}|h, \mathbf{x})\, \mathrm{d}\mathbf{f}.
\]
Thus, the posterior distribution $p(h|\mathcal{D})$ can be computed from a prior $p(h)$. For a new set of designs $\hat{\mathbf{x}}$, predictions are made by marginalising over the hypothesis:
\begin{equation}
    p(\hat{\mathbf{f}}, \hat{\mathbf{y}}|\mathcal{D}, \hat{\mathbf{x}})
    = \sum_{h \in \mathcal{H}} p(\hat{\mathbf{f}}, \hat{\mathbf{y}}|h, \hat{\mathbf{x}}) p(h|\mathcal{D}),
\end{equation}
or, if $\mathcal{H}$ is continuous, by integrating over $h$. Marginalising out $\hat{\mathbf{y}}$ yields the surrogate model:
\[
\mathcal{M}(\hat{\mathbf{f}}|\mathcal{D}, \hat{\mathbf{x}}) = \int_{\mathbb{R}^m} p(\hat{\mathbf{f}}, \hat{\mathbf{y}}|\mathcal{D}, \hat{\mathbf{x}})\, \mathrm{d} \hat{\mathbf{y}}.
\]
In other words, starting from a prior $p(h)$, the probabilistic surrogate model uses the data $\mathcal{D}$ to update our belief to $p(h|\mathcal{D})$, and then uses this posterior to make predictions.

For most hypothesis spaces, it is natural to assume the following chain decomposition:
\[
p(\hat{\mathbf{f}}, \hat{\mathbf{y}}|h, \hat{\mathbf{x}}) = p(\hat{\mathbf{y}}|h, \hat{\mathbf{f}}, \hat{\mathbf{x}}) \, p(\hat{\mathbf{f}}|h, \hat{\mathbf{x}}),
\]
which reflects the intuition that observations $\hat{\mathbf{y}}$ are noisy versions of the true objective values $\hat{\mathbf{f}}$. Therefore, we do not need to explicitly marginalise over $\hat{\mathbf{y}}$, and the surrogate model can be written as
\begin{equation}
\mathcal{M}(\hat{\mathbf{f}}|\mathcal{D}, \hat{\mathbf{x}}) = \sum_{h \in \mathcal{H}} p(\hat{\mathbf{f}}|h, \hat{\mathbf{x}}) p(h|\mathcal{D}),
\label{eq:surrogate:prediction}
\end{equation}
where $p(h|\mathcal{D})$ is given by Eq.~\eqref{eq:surrogate:bayes_belief_hypothesis}.

\paragraph{Treatment of inference.} For certain hypothesis spaces, the surrogate model and all necessary distributions can be computed analytically. For example, Gaussian process regression (see Section~\ref{sec:surrogate:gp}) provides closed-form solutions for both the posterior and predictive distributions. However, for more general or complex hypothesis spaces, the required integrals and sums are often intractable, necessitating approximate inference methods such as Markov Chain Monte Carlo or variational inference. The details of such approximate inference are beyond the scope of this tutorial.

\subsection{Choices of Surrogate Models}

The choice of surrogate model is central to the success of sequential model-based optimisation. While the Gaussian process (GP) is the most widely adopted surrogate due to its flexibility and analytic tractability, a variety of alternative models have been developed to address specific challenges or leverage domain knowledge. In this section, we provide an overview of the principal surrogate model families, highlighting their key properties, advantages, and limitations.

\paragraph{Gaussian Processes (GPs).}  
GPs \cite{williams2006gaussian} are the canonical surrogate model in Bayesian optimisation, offering a non-parametric, probabilistic approach to function approximation. By placing a prior directly over the space of functions, GPs can flexibly model a wide range of objective functions, and crucially, provide closed-form expressions for the posterior and predictive distributions under Gaussian noise. This analytic tractability enables efficient uncertainty quantification and principled acquisition strategies. Section~\ref{sec:surrogate:gp} provides a detailed treatment of GPs, including their mathematical foundations, kernel design, and practical considerations.

\paragraph{Other Non-Parametric Surrogates.}  
Beyond GPs, other non-parametric models have been proposed. The Student-\emph{t} process extends the GP framework by replacing the Gaussian assumption with a heavier-tailed Student-\emph{t} distribution, offering increased robustness to outliers and model misspecification \cite{shahStudenttProcessesAlternatives2014}. The Tree-structured Parzen Estimator (TPE) is another popular alternative \cite{bergstraAlgorithmsHyperParameterOptimization2011}, which models the objective as a density estimation problem, using non-parametric kernel density estimators to model the likelihood of good and bad observations separately. These models can be advantageous when the objective exhibits non-smooth or multi-modal behaviour, or when the noise distribution is non-Gaussian.

\paragraph{Parametric Surrogates.}  
Parametric models, such as Bayesian linear regression, random forests, and Bayesian neural networks, are also widely used as surrogates \cite{nealBayesianLearningNeural1996, murphyMachineLearningProbabilistic2013, springenbergBayesianOptimizationRobust2016}. Linear models, possibly with engineered feature embeddings, are computationally efficient and interpretable, but may lack the expressiveness required for complex objectives. Bayesian neural networks offer greater representational power, capturing highly non-linear relationships, but typically require approximate inference (e.g., variational methods or Monte Carlo sampling) and are more computationally demanding. Ensembles of parametric models, such as random forests, while not strictly Bayesian, can provide empirical uncertainty estimates via ensemble variance. In addition, many researchers have also explored how to integrate deep learning with Bayesian inference \cite{liPreconditionedStochasticGradient2016, teyeBayesianUncertaintyEstimation2018}.

\paragraph{Pretrained and Data-Driven Surrogates.}  
Using pretrained foundation models as surrogates is an emergent direction in the machine learning community. Large language models \cite{chenLLMEnhancedBayesianOptimization2024, liuLargeLanguageModels2024, mahammadliSequentialLargeLanguage2025, ramosBayesianOptimizationCatalysis2025}, prior-data fitted networks \cite{mullerPFNs4BOInContextLearning2023}, and other foundation models can be adapted to serve as surrogates, especially in domains where large, relevant datasets are available. While these models may not strictly adhere to Bayesian principles, they can incorporate vast prior knowledge and provide strong inductive biases, often yielding impressive empirical performance.

\paragraph{Domain-Specific Surrogates.}  
In many scientific applications, domain knowledge can be encoded directly into the surrogate model. For example, Bayesian networks or physics-informed neural networks can be designed to respect known factors, symmetries, or mechanistic relationships inherent to the problem \cite{dalibardBOATBuildingAutoTuners2017, hoangDecentralizedHighDimensionalBayesian2018}. Such models can dramatically improve sample efficiency and interpretability, but require careful construction and validation to ensure consistency with both the data and the underlying scientific principles.

\section{Gaussian Process \TECHNICAL\,\THEORY}
\label{sec:surrogate:gp}

In this section, we introduce the Gaussian process (GP) \cite{williams2006gaussian}, the most widely used surrogate model in Bayesian optimisation. GPs provide a principled and flexible framework for modelling unknown functions, supporting expressive prior assumptions and enabling analytic Bayesian inference.

\subsection{Foundations and Intuition}

A GP adopts the function-space view of the hypothesis space, placing a prior directly over the set of possible functions $f: \mathcal{X} \to \mathbb{R}$. This prior encodes beliefs about properties such as smoothness, periodicity, and other structural aspects of the objective function.

\begin{definition}[Gaussian Process]
Consider a \emph{random function} $f$ from the domain $\mathcal{X}$ to $\mathbb{R}$. We say $f$ satisfies a \emph{Gaussian process} on $\mathcal{X}$, if for any finite subset $\{x_1, \ldots, x_n\} \subset \mathcal{X}$, the joint distribution of $(f(x_1), \ldots, f(x_n))$ is multivariate Gaussian:
\[
(f(x_1), \ldots, f(x_n)) \sim \mathcal{N}(\mathbf{m}, \mathbf{K}),
\]
where $\mathbf{m} = (m(x_1), \ldots, m(x_n))$ is the mean vector, and $\mathbf{K}$ is the $n \times n$ covariance matrix with entries $K_{ij} = k(x_i, x_j)$. The functions $m: \mathcal{X} \to \mathbb{R}$ and $k: \mathcal{X} \times \mathcal{X} \to \mathbb{R}$ are called the \emph{mean function} and \emph{kernel function} (or covariance function), respectively. We then denote the distribution of the random function $f$ as
\[
f \sim \mathcal{GP}(m, k)
\]
\end{definition}

Intuitively, a GP is a distribution over functions, where the kernel function $k(x, x')$ encodes the similarity between points $x$ and $x'$, and the mean function $m(x)$ encodes prior beliefs about the expected behaviour of $f(x)$. A GP generalises the multivariate Gaussian distribution from finite-dimensional vectors to arbitrary (possibly infinite) index sets $\mathcal{X}$.

\subsection{GP as Surrogate Model}
\label{sec:surrogate:gp:inference}

In this section, we describe how GPs are used as surrogate models in Bayesian optimisation. We first present the general result for conditioning a GP on arbitrary jointly Gaussian variables, and then specialise to the standard regression case with i.i.d. Gaussian observation noise.

\subsubsection{Gaussian Process Conditioned on Gaussian Evidence}

A fundamental property of GPs is that conditioning on any jointly Gaussian random variable yields another GP. More generally, suppose we have a GP prior for $f$ and observe a vector of variables $\mathbf{z}$, which are jointly Gaussian with $f$ but not necessarily direct noisy observations of $f$. Let $\mathbf{x} = (x_1, \ldots, x_n)$ and $\hat{\mathbf{x}} = (\hat{x}_1, \ldots, \hat{x}_m)$ be two sets of input locations. Define
\begin{align*}
    \mathbf{f} &= (f(x_1), \ldots, f(x_n))^\top, \\
    \hat{\mathbf{f}} &= (f(\hat{x}_1), \ldots, f(\hat{x}_m))^\top, \\
    \mathbf{z} &\sim \mathcal{N}(\boldsymbol{\mu}_z, \mathbf{K}_{zz}),
\end{align*}
where $\mathbf{z}$ can have arbitrary covariance $\mathbf{K}_{zz}$ and cross-covariance $\mathbf{K}_{\hat{\mathbf{f}}z}$ with $\hat{\mathbf{f}}$.

The joint distribution is
\[
\begin{bmatrix}
\mathbf{z} \\
\hat{\mathbf{f}}
\end{bmatrix}
\sim \mathcal{N}\left(
\begin{bmatrix}
\boldsymbol{\mu}_z \\
\boldsymbol{\mu}_{\hat{f}}
\end{bmatrix},
\begin{bmatrix}
\mathbf{K}_{zz} & \mathbf{K}_{z\hat{f}} \\
\mathbf{K}_{\hat{f}z} & \mathbf{K}_{\hat{f}\hat{f}}
\end{bmatrix}
\right)
\]
where $\mathbf{K}_{z\hat{f}}$ and $\mathbf{K}_{\hat{f}z}$ are the mutual covariance matrices. Conditioning on $\mathbf{z}$, the posterior for $\hat{\mathbf{f}}$ is
\[
\hat{\mathbf{f}} \mid \mathbf{z} \sim \mathcal{N}\left(
\boldsymbol{\mu}_{\hat{f}} + \mathbf{K}_{\hat{f}z} \mathbf{K}_{zz}^{-1} (\mathbf{z} - \boldsymbol{\mu}_z), \;
\mathbf{K}_{\hat{f}\hat{f}} - \mathbf{K}_{\hat{f}z} \mathbf{K}_{zz}^{-1} \mathbf{K}_{z\hat{f}}
\right)
\]
This general result applies to any jointly Gaussian variables. The classical GP regression is a special case, where $\mathbf{z}$ are noisy observations of $f$ with i.i.d. Gaussian noise, which is discussed next.

\subsubsection{GP Conditioned on Noisy Observations}

In the standard GP regression setting, we observe noisy measurements $\mathbf{y} = (y_1, \ldots, y_n)^\top$ at inputs $\mathbf{x} = (x_1, \ldots, x_n)$, where $y_i = f(x_i) + \varepsilon_i$ and $\varepsilon_i \sim \mathcal{N}(0, \sigma^2)$ i.i.d. The joint distribution of observed and predicted function values at new inputs $\hat{\mathbf{x}} = (\hat{x}_1, \ldots, \hat{x}_m)$ is:
\[
\begin{bmatrix}
\mathbf{y} \\
\hat{\mathbf{f}}
\end{bmatrix}
\sim \mathcal{N}\left(
\begin{bmatrix}
\mathbf{m} \\
\hat{\mathbf{m}}
\end{bmatrix},
\begin{bmatrix}
\mathbf{K}_{\mathbf{x}\mathbf{x}} + \sigma^2 \mathbf{I} & \mathbf{K}_{\mathbf{x}\hat{\mathbf{x}}} \\
\mathbf{K}_{\hat{\mathbf{x}}\mathbf{x}} & \mathbf{K}_{\hat{\mathbf{x}}\hat{\mathbf{x}}}
\end{bmatrix}
\right),
\]
where $\mathbf{m} = (m(x_1), \ldots, m(x_n))^\top$, $\hat{\mathbf{m}} = (m(\hat{x}_1), \ldots, m(\hat{x}_m))^\top$, and $\mathbf{K}_{\mathbf{x}\hat{\mathbf{x}}}$ is the $n \times m$ matrix with entries $k(x_i, \hat{x}_j)$.

Conditioning on the observed data, the posterior predictive distribution for $\hat{\mathbf{f}}$ is:
\[
\hat{\mathbf{f}} \mid \mathbf{x}, \mathbf{y}, \hat{\mathbf{x}} \sim \mathcal{N}(\boldsymbol{\mu}_{\hat{x}}, \boldsymbol{\Sigma}_{\hat{x}}),
\]
where
\begin{align*}
\boldsymbol{\mu}_{\hat{x}} &= \hat{\mathbf{m}} + \mathbf{K}_{\hat{\mathbf{x}}\mathbf{x}} [\mathbf{K}_{\mathbf{x}\mathbf{x}} + \sigma^2 \mathbf{I}]^{-1} (\mathbf{y} - \mathbf{m}), \\
\boldsymbol{\Sigma}_{\hat{x}} &= \mathbf{K}_{\hat{\mathbf{x}}\hat{\mathbf{x}}} - \mathbf{K}_{\hat{\mathbf{x}}\mathbf{x}} [\mathbf{K}_{\mathbf{x}\mathbf{x}} + \sigma^2 \mathbf{I}]^{-1} \mathbf{K}_{\mathbf{x}\hat{\mathbf{x}}}.
\end{align*}

The i.i.d. additive noise assumption is strong but often reasonable in practice. The posterior mean provides the best estimate of the function at new points, while the posterior covariance quantifies uncertainty---crucial for exploration in Bayesian optimisation.

\subsection{Mean Function}

The mean function $m(x)$ encodes prior beliefs about the average value of $f(x)$. Common choices include:

\begin{itemize}
    \item \textbf{Zero mean:} $m(x) = 0$. This is the default when no prior information is available.
    \item \textbf{Constant mean:} $m(x) = c$, where $c$ is a constant hyperparameter.
    \item \textbf{Linear mean:} $m(x) = \beta^\top x$, where $\beta$ is a vector of coefficients.
    \item \textbf{Linear combination of basis functions:} $m(x) = \sum_{j} \beta_j \phi_j(x)$, where $\{\phi_j\}$ are user-chosen basis functions and $\beta_j$ are coefficients.
\end{itemize}

The parameters of the mean function (e.g., $c$, $\beta$) are treated as hyperparameters. They can be chosen manually or automatically adapted from data (see Section~\ref{sec:surrogate:hyperparam_ada}).

\subsection{Kernel Function}

The kernel function $k(x, x')$ encodes assumptions about smoothness, periodicity, and other structural properties of $f$. It must be symmetric and positive semi-definite to ensure the resulting covariance matrix is valid.

\begin{definition}[Positive Semi-Definite Kernel]
A function $k: \mathcal{X} \times \mathcal{X} \to \mathbb{R}$ is a \emph{positive semi-definite (PSD) kernel} if, for any finite set $\{x_1, \ldots, x_n\} \subset \mathcal{X}$, the matrix $K$ with entries $K_{ij} = k(x_i, x_j)$ is symmetric and positive semi-definite, i.e., for any $\mathbf{a} \in \mathbb{R}^n$, $\mathbf{a}^\top K \mathbf{a} \geq 0$.
\end{definition}

\subsubsection{Basic Choices}

Common kernel functions include:

\begin{itemize}
    \item \textbf{Squared Exponential (RBF) Kernel:}
    \[
    k(x, x') = \sigma_f^2 \exp\left(-\frac{\|x - x'\|^2}{2\ell^2}\right),
    \]
    where $\sigma_f^2$ is the signal variance and $\ell$ is the length-scale. This kernel encodes the assumption that the function is infinitely differentiable (very smooth), and the length-scale controls how quickly correlation decays with distance.
    \item \textbf{Matern Kernel:}
    \[
    k(x, x') = \sigma_f^2 \frac{2^{1-\nu}}{\Gamma(\nu)} \left( \frac{\sqrt{2\nu} \|x - x'\|}{\ell} \right)^\nu K_\nu \left( \frac{\sqrt{2\nu} \|x - x'\|}{\ell} \right),
    \]
    where $K_\nu$ is the modified Bessel function, and $\nu > 0$ is a smoothness parameter. The Matern kernel allows control over the degree of smoothness; e.g., $\nu=1/2$ yields the exponential kernel, and as $\nu\to\infty$ it approaches the RBF kernel.
    \item \textbf{Linear Kernel:}
    \[
    k(x, x') = \sigma_b^2 + x^\top x',
    \]
    which encodes linear functions and is suitable when the underlying function is expected to be approximately linear.
    \item \textbf{Periodic Kernel:}
    \[
    k(x, x') = \sigma_f^2 \exp\left( -\frac{2\sin^2(\pi |x - x'| / p)}{\ell^2} \right),
    \]
    where $p$ is the period. This kernel is suitable for modelling periodic functions, such as seasonal effects.
\end{itemize}

Each kernel has associated hyperparameters (e.g., length-scale, variance) that control the behaviour of the GP and are typically learned from data (see Section~\ref{sec:surrogate:hyperparam_ada}).

\subsubsection{Non-Numeric Kernels}

GPs are not limited to numeric data; kernels can be defined for structured or discrete objects, enabling GPs to model complex domains.

\paragraph{String Kernel}  
String kernels measure the similarity between sequences (such as DNA, RNA, or protein sequences) by counting shared substrings or motifs. For example, the \emph{spectrum kernel} is defined as
\[
k(s, t) = \sum_{u \in \mathcal{A}^k} \#_u(s) \cdot \#_u(t),
\]
where $\#_u(s)$ is the number of occurrences of substring $u$ of length $k$ in string $s$, and $\mathcal{A}$ is the alphabet. Lodhi et al.~\cite{lodhiTextClassificationUsing2000} introduced efficient algorithms for computing such kernels, enabling GPs to handle high-dimensional sequence data for text classification and bioinformatics.  For application of the string kernel in Bayesian optimization, we refer the readers to \textcite{mossBOSSBayesianOptimization2020}.

\paragraph{Fisher Kernel}  
The Fisher kernel leverages a generative probabilistic model $p(x|\eta)$, with parameters $\eta$, to define similarity as
\[
k(x, x') = \nabla_\eta \log p(x|\eta_0)^\top F^{-1} \nabla_\eta \log p(x'|\eta_0),
\]
where $F$ is the Fisher information matrix and $\eta_0$ is a reference parameter value. This kernel is suitable when a generative model captures important structure in the data. By specifying the generative probabilistic model $p(x|\eta)$, the Fisher kernel can handle design spaces consisting of graphs, sequences, or other non-vectorial data, allowing the GP to exploit domain-specific similarity measures.

\subsubsection{Constructing New Kernels from Base Kernels}

Kernels are closed under addition, multiplication, and scaling, enabling the construction of complex kernels from simpler ones. 

\begin{itemize}
    \item \textbf{Sum:} If $k_1$ and $k_2$ are PSD kernels, then $k(x, x') = k_1(x, x') + k_2(x, x')$ is PSD. This allows the model to capture the additive effects of different features or groups of variables.
    \item \textbf{Product:} If $k_1$ and $k_2$ are PSD, then $k(x, x') = k_1(x, x') \cdot k_2(x, x')$ is also PSD. This enables modelling interactions between different aspects of the input.
    \item \textbf{Scaling:} $k(x, x') = c \cdot k_1(x, x')$ with $c \geq 0$ is PSD.
\end{itemize}

These properties allow practitioners to design kernels that capture multiple aspects of the data, such as periodicity and smoothness, or to encode domain-specific structure. They are crucial for modelling heterogeneous design spaces (e.g., combining categorical and continuous variables), capturing multiple properties (e.g., smoothness and periodicity), or handling high-dimensional spaces by factorizing the kernel across groups of variables. For heterogeneous domains, one can use a \emph{product kernel} $k(x, x') = k_\mathrm{num}(x_\mathrm{num}, x'_\mathrm{num}) \cdot k_\mathrm{cat}(x_\mathrm{cat}, x'_\mathrm{cat})$, where $x_\mathrm{num}$ and $x_\mathrm{cat}$ denote the numeric and categorical components, respectively. For high-dimensional spaces, kernels can be factorized or built as sums/products over groups of features, allowing practitioners to encode prior knowledge about relevant interactions. Some studies also explore the autonomous search of composite kernel structures based on these properties from a series of base kernels \cite{duvenaudStructureDiscoveryNonparametric2013}.

\subsection{Implementation and Discussion}

\subsubsection{Complexity Analysis and Approximate Inference}

The primary computational bottleneck is the inversion of the $n \times n$ covariance matrix, which scales as $\mathcal{O}(n^3)$ in time and $\mathcal{O}(n^2)$ in memory. For sequential data, incremental updates and matrix factorisation techniques (e.g., Cholesky decomposition) can be used to speed up computations. For large datasets, sparse GP approximations or inducing point methods are employed to reduce complexity.

Exact inference in GPs is computationally expensive for large $n$ (typically $\mathcal{O}(n^3)$ time and $\mathcal{O}(n^2)$ memory). Approximate inference techniques, such as variational inference \cite{hensmanMCMCVariationallySparse2015, tranVariationalGaussianProcess2016} and sparse GPs \cite{snelsonSparseGaussianProcesses2005, lazaro-gredillaSparseSpectrumGaussian2010}, typically reduce the computational cost to roughly $\mathcal{O}(nm^2)$ time and $\mathcal{O}(nm)$ memory, for some $m \ll n$. A detailed discussion of these techniques is beyond the scope of this tutorial.

\subsection{Equivalence to Bayesian Inference of Linear Regression}

Bayesian linear regression with a feature map $\phi(x)$ and Gaussian prior over weights is equivalent to a GP with kernel $k(x, x') = \phi(x)^\top \Sigma \phi(x')$. Specifically, let $y = \omega^\top \phi(x) + \varepsilon$, with $\omega \sim \mathcal{N}(0, \Sigma)$ and $\varepsilon \sim \mathcal{N}(0, \sigma^2)$. The predictive distribution for new inputs is Gaussian, with mean and covariance matching those of a GP with the specified kernel. This equivalence underpins the close relationship between kernel methods and Bayesian linear models, and extends to infinite-width neural networks via the Neural Tangent Kernel.

\paragraph{Mathematical Derivation}
Let $\mathbf{y} = \Phi \omega + \boldsymbol{\varepsilon}$, where $\Phi$ is the $n \times d$ design matrix with $\Phi_{ij} = \phi_j(x_i)$, $\omega \sim \mathcal{N}(0, \Sigma)$, and $\boldsymbol{\varepsilon} \sim \mathcal{N}(0, \sigma^2 \mathbf{I})$. The marginal distribution of $\mathbf{y}$ is
\[
\mathbf{y} \sim \mathcal{N}(0, \Phi \Sigma \Phi^\top + \sigma^2 \mathbf{I})
\]
For a new input $\hat{x}$, let $\hat{\phi} = \phi(\hat{x})$. The joint distribution of $(\mathbf{y}, \hat{y})$ is
\[
\begin{bmatrix}
\mathbf{y} \\
\hat{y}
\end{bmatrix}
\sim \mathcal{N}\left(
\mathbf{0},
\begin{bmatrix}
\Phi \Sigma \Phi^\top + \sigma^2 \mathbf{I} & \Phi \Sigma \hat{\phi} \\
\hat{\phi}^\top \Sigma \Phi^\top & \hat{\phi}^\top \Sigma \hat{\phi}
\end{bmatrix}
\right)
\]
By the formula for conditional Gaussians, the predictive distribution for $\hat{y}$ given $\mathbf{y}$ is
\[
\hat{y} \mid \mathbf{y} \sim \mathcal{N}\left(
\hat{\phi}^\top \Sigma \Phi^\top [\Phi \Sigma \Phi^\top + \sigma^2 \mathbf{I}]^{-1} \mathbf{y},\;
\hat{\phi}^\top \Sigma \hat{\phi} - \hat{\phi}^\top \Sigma \Phi^\top [\Phi \Sigma \Phi^\top + \sigma^2 \mathbf{I}]^{-1} \Phi \Sigma \hat{\phi}
\right)
\]
which matches the GP predictive equations with kernel $k(x, x') = \phi(x)^\top \Sigma \phi(x')$.

\section{Implementation \TECHNICAL}

\subsection{Feature Representation}

The effectiveness of surrogate models, particularly parametric ones, often hinges on the choice of feature representation. In scientific domains, domain-specific representations can dramatically improve model performance. For example, in chemistry, molecular graphs can be embedded using graph neural networks, descriptors, or fingerprints; in biology, protein sequences may be represented using sequence embeddings or structural features. The choice of representation should reflect the relevant properties of the domain and the nature of the optimisation task.

\subsection{Hyperparameter Adaptation}
\label{sec:surrogate:hyperparam_ada}

Surrogate models, especially GPs, involve several hyperparameters, including those governing the kernel (e.g., length-scale, variance), the mean function, and the noise model. The choice of hyperparameters significantly affects the surrogate's predictive accuracy and uncertainty quantification.

\subsubsection{Maximising the Marginal Likelihood}

A principled approach to hyperparameter selection is to maximise the marginal likelihood, which is the probability of the observed data under the model, integrating over the latent function values. For a surrogate model with hyperparameters $\lambda$, this means
\begin{equation}
\lambda^* = \arg\max_\lambda \; p(\mathbf{y} \mid \mathbf{x}, \lambda)
\label{eq:surrogate:max_marginal_likeli}
\end{equation}
where $p(\mathbf{y} \mid \mathbf{x}, \lambda)$ is the likelihood of the observed data given the hyperparameters. This approach applies to any surrogate model with hyperparameters. Taking the GP surrogate (see Section~\ref{sec:surrogate:gp:inference}) for example, given observations with i.i.d. additive Gaussian noise, Eq.~\eqref{eq:surrogate:max_marginal_likeli} becomes maximizing:
\[
\log p(\mathbf{y} \mid \mathbf{x}, \lambda)
= -\frac{1}{2}(\mathbf{y} - \mathbf{m})^\top \bigl(\mathbf{K}_{\mathbf{x}\mathbf{x}} + \sigma^2 \mathbf{I}\bigr)^{-1} (\mathbf{y} - \mathbf{m})
  - \frac{1}{2} \log \det\bigl(\mathbf{K}_{\mathbf{x}\mathbf{x}} + \sigma^2 \mathbf{I}\bigr)
  - \frac{n}{2} \log(2\pi),
\]
where $n$ is the dimension (length) of $\mathbf{y}$.

\subsubsection{Hierarchical Bayes}

Alternatively, a fully Bayesian treatment places priors over the hyperparameters and integrates them out, yielding a hierarchical model. In the Bayesian framework, the predictive distribution for a new input $\hat{x}$ is
\[
p(\hat{y} \mid \hat{x}, \mathbf{x}, \mathbf{y}) = \int p(\hat{y} \mid \hat{x}, \mathbf{x}, \mathbf{y}, \lambda) \, p(\lambda \mid \mathbf{x}, \mathbf{y}) \, d\lambda,
\]
where $p(\lambda \mid \mathbf{x}, \mathbf{y})$ is the posterior over hyperparameters. In practice, this integration is often intractable and is approximated using Markov Chain Monte Carlo (MCMC) or variational inference. Hierarchical Bayes provides more robust uncertainty estimates, particularly when data is scarce or hyperparameters are poorly identified.

\section{Coding the Surrogate Model \CODE}

This section provides a practical introduction to surrogate modelling within Bayesian Optimization (BO) software, focusing on Gaussian Process (GP) models. First, we show how surrogate models are typically implemented and configured in BO frameworks, exemplified by HEBO. Furthermore, we demonstrate how to create and customize Gaussian Process surrogates from scratch using two popular Python libraries. The coding patterns we present are invariant across applications and problem settings, so we do not provide case-by-case studies.

\subsection{Surrogate Modelling in BO Software}

Most BO software naturally employs Gaussian Processes (see Section~\ref{sec:surrogate:gp}) as surrogate models, where the hyperparameters are inferred from the available data. Through the software interface, users can typically configure the surrogate model by passing specific arguments; however, the degree of customization is often limited by the interface design. 

\subsection{Surrogate Modelling in HEBO}

In the following, we use HEBO \cite{cowen-riversHEBOPushingLimits2022} as an illustrative example for surrogate modelling: 

\begin{lstlisting}[language=Python]
from hebo.optimizers.hebo import HEBO
from hebo.design_space.design_space import DesignSpace

# Define the design space (example)
space = DesignSpace().parse({
    "name": "x1", "type": "num", "lb": 0, "ub": 1
})

# Instantiate HEBO with a specific surrogate model
opt = HEBO(
    space=space,
    model_name="gp",  # Surrogate model name (default: Gaussian Process)
    model_config={"lr": 0.01, "num_epochs": 100}  # Configuration for the surrogate.
)
\end{lstlisting}

For the Gaussian Process, which is the default surrogate, HEBO automatically constructs a kernel suitable for heterogeneous variables. Specifically, it employs a Matern kernel for numeric variables and combines it with appropriate kernels for categorical variables when necessary. Kernel hyperparameters, such as lengthscales, are initialized based on the dataset. This design is practical for most scenarios, offering a robust and flexible surrogate without requiring users to specify detailed configurations.

HEBO’s GP implementation is powered by GPyTorch \cite{gardnerGPyTorchBlackboxMatrixMatrix2018}, which fits kernel and mean function hyperparameters from data. Additionally, HEBO enables users to customize the GP further by passing a dictionary via \texttt{model\_config}. For example, one can specify the learning rate (\texttt{lr}) and the number of epochs (\texttt{num\_epochs}) for hyperparameter optimization. Detailed configuration of the GP in HEBO is beyond the scope of this tutorial\footnote{For readers interested in in-depth customization of GP surrogates and possessing relevant coding expertise, we refer to the HEBO source code at \url{https://github.com/huawei-noah/HEBO/tree/master/HEBO/hebo}.}, and the default HEBO settings are generally suitable for most use cases without additional configuration.

HEBO also supports a variety of non-GP surrogate models. For instance, \texttt{rf} (Random Forest) is suitable for non-smooth objective functions; \texttt{deep\_ensemble}, \texttt{psgld}, and \texttt{mcbn} are neural network-based models designed for probabilistic regression \cite{liPreconditionedStochasticGradient2016, teyeBayesianUncertaintyEstimation2018}. Moreover, HEBO offers specialized models such as \texttt{gumbel}, \texttt{masked\_deep\_ensemble}, and \texttt{fe\_deep\_ensemble}, which incorporate feature selection mechanisms. While this tutorial does not delve into the technical details of these surrogates---topics that require extensive machine learning background---it is useful to acknowledge their existence. This awareness enables users to tackle a broader range of optimization problems beyond the scope of traditional Gaussian Processes.

\subsubsection{Surrogate Modelling in Bgolearn}
\label{sec:surrogate:coding:gp_impl:bgolearn}

\begin{center}
\CODE\,\url{https://github.com/zwyu-ai/BO-Tutorial-for-Sci/blob/main/coding_illustrations/gp_bo_bgolearn.ipynb}
\end{center}

The \texttt{Bgolearn} \cite{cao2026bgolearn} package provides an ensemble surrogate modelling framework for Bayesian optimization. In this example, we demonstrate how to construct a GP surrogate model using history data files.

\begin{lstlisting}[language=Python]
import pandas as pd
data = pd.read_excel('./data/train.xlsx')
vs = pd.read_excel('./data/visual_samples.xlsx')

import Bgolearn.BGOsampling as BGOS
Bgolearn = BGOS.Bgolearn() 

Mymodel = Bgolearn.fit(
    data_matrix=data.iloc[:,:-1],
    Measured_response=data.iloc[:,-1],
    virtual_samples=vs,
    Classifier='GaussianProcess'  # Explicit specification
)
\end{lstlisting}

The \texttt{Bgolearn} package supports several surrogate models, each with different strengths and typical use cases (see \url{https://bgolearn.netlify.app/surrogate_models}):

\begin{itemize}
    \item \textbf{Gaussian Process} -- Default and most commonly used model.
    \item \textbf{Random Forest} -- Effective for discrete or categorical features.
    \item \textbf{Support Vector Regression} -- Robust to noisy data.
    \item \textbf{Multi-Layer Perceptron} -- Neural network-based approach.
    \item \textbf{AdaBoost} -- Ensemble learning method.
\end{itemize}

\subsection{Creating Your Own Gaussian Processes}
\label{sec:surrogate:coding:gp_impl}

The flexibility of surrogate models within BO software is often limited by the interface and typically encapsulated deep within the implementation, making it difficult to inspect or customize. However, it is possible to fully examine, control, and even tailor surrogate modelling. In this section, we demonstrate how to achieve this in Python.

Specifically, we introduce two libraries---Scikit-Learn \cite{scikit-learn} and GPyTorch \cite{gardnerGPyTorchBlackboxMatrixMatrix2018}---to create Gaussian Process surrogates (see Section~\ref{sec:surrogate:gp}) for Bayesian optimization. Scikit-Learn is well-suited for small to moderate-sized datasets and rapid prototyping due to its unified interface and ease of use. It abstracts many technical details, allowing researchers to focus on modelling concepts rather than implementation intricacies. Conversely, GPyTorch is designed for scalability and customizability. It supports GPU-accelerated inference and training, making it ideal for large-scale or high-dimensional problems. Furthermore, GPyTorch offers fine-grained control over GP architectures, kernel composition, and training procedures, which is essential for advanced research and engineering applications. This flexibility, however, comes with increased complexity and a steeper learning curve.

This tutorial serves as an introductory guide, aiming to provide conceptual clarity and practical orientation rather than exhaustive technical coverage. The code examples presented are deliberately minimal and intended for illustration only. Readers with a strong interest in Gaussian Process modelling or those seeking to deploy GPs in production or at scale are encouraged to consult the official documentation and community resources for Scikit-Learn (\url{https://scikit-learn.org/stable/modules/gaussian_process.html}) and GPyTorch (\url{https://gpytorch.ai/}). Both libraries are actively maintained and offer extensive examples, API references, and tutorials. Replicating such comprehensive material here is neither practical nor desirable; instead, our focus is to introduce fundamental concepts and direct readers to authoritative sources for further study.

\subsubsection{GP Modelling using Scikit-Learn}
\label{sec:surrogate:coding:gp_impl:sklearn}

\begin{center}
\CODE\,\url{https://github.com/zwyu-ai/BO-Tutorial-for-Sci/blob/main/coding_illustrations/gp_bo_sklearn.ipynb}
\end{center}

Scikit-Learn provides a user-friendly interface for Gaussian Process regression through the \texttt{GaussianProcessRegressor} class. The workflow is streamlined and accessible: users first specify a kernel function encoding prior assumptions about the smoothness and structure of the underlying function. The \texttt{GaussianProcessRegressor} is then instantiated with the kernel and other relevant parameters, such as noise level. Fitting the model to data is performed via the \texttt{fit} method, which automatically optimizes kernel hyperparameters by maximizing the marginal likelihood. Predictions on new data are made using the \texttt{predict} method, which returns both mean predictions and uncertainty estimates. This design abstracts much of the complexity of GP inference, allowing users to focus on model selection and interpretation rather than implementation details. The interface is consistent with other Scikit-Learn estimators, facilitating integration into broader machine learning workflows. The following example demonstrates fitting a GP to noisy data and making predictions:

\begin{lstlisting}[language=Python]
import numpy as np
from sklearn.gaussian_process import GaussianProcessRegressor
from sklearn.gaussian_process.kernels import Matern, RBF, ConstantKernel

# Generate synthetic training data
X_train = np.array([0.1, 0.4, 0.6, 0.7, 0.8]).reshape(-1, 1)
y_train = np.sin(2 * np.pi * X_train).ravel() + np.random.normal(0, 0.1, X_train.shape[0])

# Define kernel: constant * RBF + Matern (as an example)
kernel = ConstantKernel(1.0) * RBF(length_scale=0.2) + Matern(length_scale=0.2, nu=1.5)

# Instantiate GP regressor
gp = GaussianProcessRegressor(kernel=kernel, alpha=0.1**2, n_restarts_optimizer=10)

# Fit to data
gp.fit(X_train, y_train)

# Predict at new points
X_test = np.linspace(0, 1, 100).reshape(-1, 1)
y_pred, y_std = gp.predict(X_test, return_std=True)

# Plot results (optional)
import matplotlib.pyplot as plt
plt.figure()
plt.plot(X_test, np.sin(2 * np.pi * X_test), 'k--', label='True Function')
plt.plot(X_train, y_train, 'ro', label='Observed')
plt.plot(X_test, y_pred, 'b-', label='GP Mean')
plt.fill_between(X_test.ravel(), y_pred - 2*y_std, y_pred + 2*y_std, color='blue', alpha=0.2, label='95% CI')
plt.legend()
plt.show()
\end{lstlisting}

Scikit-Learn’s GP module emphasizes simplicity and consistency. It is highly accessible for beginners, with automatic kernel hyperparameter optimization during fitting. The \texttt{predict} method provides both mean predictions and uncertainty estimates. For further details and advanced usage, readers should consult the official documentation at \url{https://scikit-learn.org/stable/modules/gaussian_process.html}.

\subsubsection{GP Modelling using GPyTorch}
\label{sec:surrogate:coding:gp_impl:gpytorch}

\begin{center}
\CODE\,\url{https://github.com/zwyu-ai/BO-Tutorial-for-Sci/blob/main/coding_illustrations/gp_bo_gpytorch.ipynb}
\end{center}

In contrast, GPyTorch offers a more modular and extensible framework for Gaussian Process modelling, tailored to advanced research and large-scale applications. Built on PyTorch, GPyTorch is particularly well-suited for scalable Gaussian Process modelling, enabling extensive customization of GP architectures.

The typical workflow begins by subclassing \texttt{gpytorch.models.ExactGP} to define a custom GP model, specifying mean and covariance (kernel) modules within the constructor. A likelihood object---often a Gaussian likelihood that models observation noise---is instantiated separately and linked to the model. Training proceeds via a manual optimization loop: the model and likelihood are set to training mode, and the marginal log likelihood is maximized using an optimizer such as Adam. This loop requires explicit computation of gradients and parameter updates, offering fine-grained control over the learning process. For prediction, the model and likelihood switch to evaluation mode, and the predictive distribution is computed for new inputs. The following example illustrates this basic workflow:

\begin{lstlisting}[language=Python]
import torch
import gpytorch
import numpy as np

# Generate synthetic data
X_train = torch.tensor([0.1, 0.4, 0.6, 0.7, 0.8]).unsqueeze(-1)
y_train = torch.sin(2 * np.pi * X_train).squeeze() + 0.1 * torch.randn(X_train.size(0))

# Define GP model
class GPModel(gpytorch.models.ExactGP):
    def __init__(self, train_x, train_y, likelihood):
        super().__init__(train_x, train_y, likelihood)
        self.mean_module = gpytorch.means.ConstantMean()
        self.covar_module = gpytorch.kernels.ScaleKernel(
            gpytorch.kernels.RBFKernel()
        )
    def forward(self, x):
        mean_x = self.mean_module(x)
        covar_x = self.covar_module(x)
        return gpytorch.distributions.MultivariateNormal(mean_x, covar_x)

# Instantiate likelihood and model
likelihood = gpytorch.likelihoods.GaussianLikelihood()
model = GPModel(X_train, y_train, likelihood)

# Training loop
model.train()
likelihood.train()
optimizer = torch.optim.Adam(model.parameters(), lr=0.1)
mll = gpytorch.mlls.ExactMarginalLogLikelihood(likelihood, model)

for i in range(100):
    optimizer.zero_grad()
    output = model(X_train)
    loss = -mll(output, y_train)
    loss.backward()
    optimizer.step()

# Switch to eval mode for prediction
model.eval()
likelihood.eval()
X_test = torch.linspace(0, 1, 100).unsqueeze(-1)
with torch.no_grad(), gpytorch.settings.fast_pred_var():
    pred = model(X_test)
    y_pred = pred.mean
    y_std = pred.stddev

# Plot results (optional)
import matplotlib.pyplot as plt
plt.figure()
plt.plot(X_train.numpy(), y_train.numpy(), 'ro', label='Observed')
plt.plot(X_test.numpy(), y_pred.numpy(), 'b-', label='GP Mean')
plt.plot(X_test.numpy(), np.sin(2 * np.pi * X_test.numpy()), 'k--', label='True Function')
plt.fill_between(X_test.numpy().ravel(),
                 (y_pred - 2*y_std).numpy(),
                 (y_pred + 2*y_std).numpy(),
                 color='blue', alpha=0.2, label='95% CI')
plt.legend()
plt.show()
\end{lstlisting}

This approach allows extensive customization, including GPU acceleration, approximate inference, and bespoke kernels. It presumes familiarity with PyTorch’s computational model and is best suited for readers with an engineering or machine learning background. For comprehensive guidance, refer to the official documentation at \url{https://gpytorch.ai/}.

%% file: v2/decision.tex
\part{Acting and Decision Making}
\label{part:decision}

The concept of decision making can be viewed as a mapping from knowledge to action. In the context of scientific discovery, decision making translates our current beliefs about a hypothesis into actions---specifically, the selection of subsequent experiments. The overarching aim is to acquire more knowledge about the unknown and to achieve higher objective values.

In sequential optimisation, this process is formalised as a \emph{policy} $\pi$, which maps the accumulated data $\mathcal{D}_t$ to the design $x_{t+1} \in \mathcal{X}$ for the next experiment. For both sequential model-based optimisation (SMBO) and Bayesian Optimisation (BO), this policy $\pi$ determines the next action using the posterior model $\mathcal{M}(\hat{\mathbf{f}}|\hat{\mathbf{x}}, \mathcal{D}_t)$. In the sections that follow, we explore the concepts, theoretical foundations, and coding implementations that underlie the construction of such a policy within the BO framework.

\section{Acquisition Function \ELEMENTARY}
\label{sec:decision:acq_func}

Given a surrogate model for the unknown objective $f: \mathcal{X} \to \mathbb{R}$, we now turn our attention to the crucial aspect of decision making: namely, given the current model $\mathcal{M}(\hat{\mathbf{f}}|\hat{\mathbf{x}}, \mathcal{D}_t)$, how should we select the next design $x_{t+1}$? To address this question, BO reframes decision making as the optimisation of a tractable, white-box function known as the \emph{acquisition function}.

\subsection{Motivation and Definition}

Experiments aimed at efficient discovery must be designed according to sound principles. The central idea of BO aligns closely with the Bayesian philosophy of scientific discovery discussed in Section~\ref{sec:intro:bayesian_discovery}: experiments should be designed not only by leveraging our current knowledge, but also by acknowledging and incorporating the uncertainty arising from our ignorance \cite{garnettBayesianOptimization2023}. In this context, the acquisition function acts as a principled criterion for evaluating the potential utility of each candidate point, quantifying the expected benefit of evaluating the objective at that point given the current surrogate model and observed data. It provides a systematic approach to selecting the next experiment, inherently balancing the trade-off between \emph{exploration} (gathering information in poorly understood regions) and \emph{exploitation} (focusing on regions likely to yield high objective values).

\begin{definition}[Acquisition Function]
The acquisition function, denoted $\alpha(x|\mathcal{D})$, is a known function that can be efficiently computed from the surrogate model and the current data $\mathcal{D}$. It assigns a score to each candidate point $x \in \mathcal{X}$, reflecting the expected benefit of evaluating $f(x)$ given the current knowledge. Given the acquisition function, the decision-making process becomes the maximisation of the acquisition function:
\begin{equation}
x_{t+1} = \arg\max_{x \in \mathcal{X}} \alpha(x|\mathcal{D}_t)
\end{equation}
\end{definition}

The acquisition function is constructed upon the surrogate model and is designed to be much less computationally demanding than evaluating the true objective function. In practice, this means that the computational cost associated with optimising the acquisition function is negligible compared to the expense of conducting real experiments. As a result, the process of decision making---i.e., maximising the acquisition function---is far more tractable than directly optimising the costly, black-box objective. In essence, by maximising the acquisition function, we transform the complex problem of experimental design into a manageable optimisation task. Furthermore, this approach is firmly grounded in the principles of Bayesian decision theory, which we discuss in greater detail in Section~\ref{sec:decision:bayesian_decision_theory}.

\subsection{Trading off Exploration and Exploitation}

A distinguishing feature of Bayesian Optimisation is that the surrogate model $\mathcal{M}$ provides not only predictions but also uncertainty estimates, as a result of Bayesian inference. This uncertainty can be strategically exploited to actively explore the unknown objective function. In this setting, decision making must strike a careful balance between two competing objectives: \emph{exploration}, which involves sampling points where the model is uncertain in order to reduce uncertainty and potentially discover superior solutions; and \emph{exploitation}, which involves sampling points where the model predicts high objective values, thereby seeking to improve performance based on current knowledge.

The acquisition function is specifically designed to mediate this exploration--exploitation trade-off. Its mathematical form typically incorporates both the predicted mean and the uncertainty (variance) of the surrogate model at each candidate point. By doing so, the acquisition function encourages sampling in regions that are either highly promising (exploitation) or highly uncertain (exploration). This dual focus enables BO to efficiently optimise the unknown objective by judiciously allocating experimental resources.

\subsection{Common Choices of Acquisition Functions}
\label{sec:decision:common_acq_choices}

\begin{figure}[tbh]
\centering
\includegraphics[trim={0em 12em 0em 13em}, clip=true, width=0.75\linewidth]{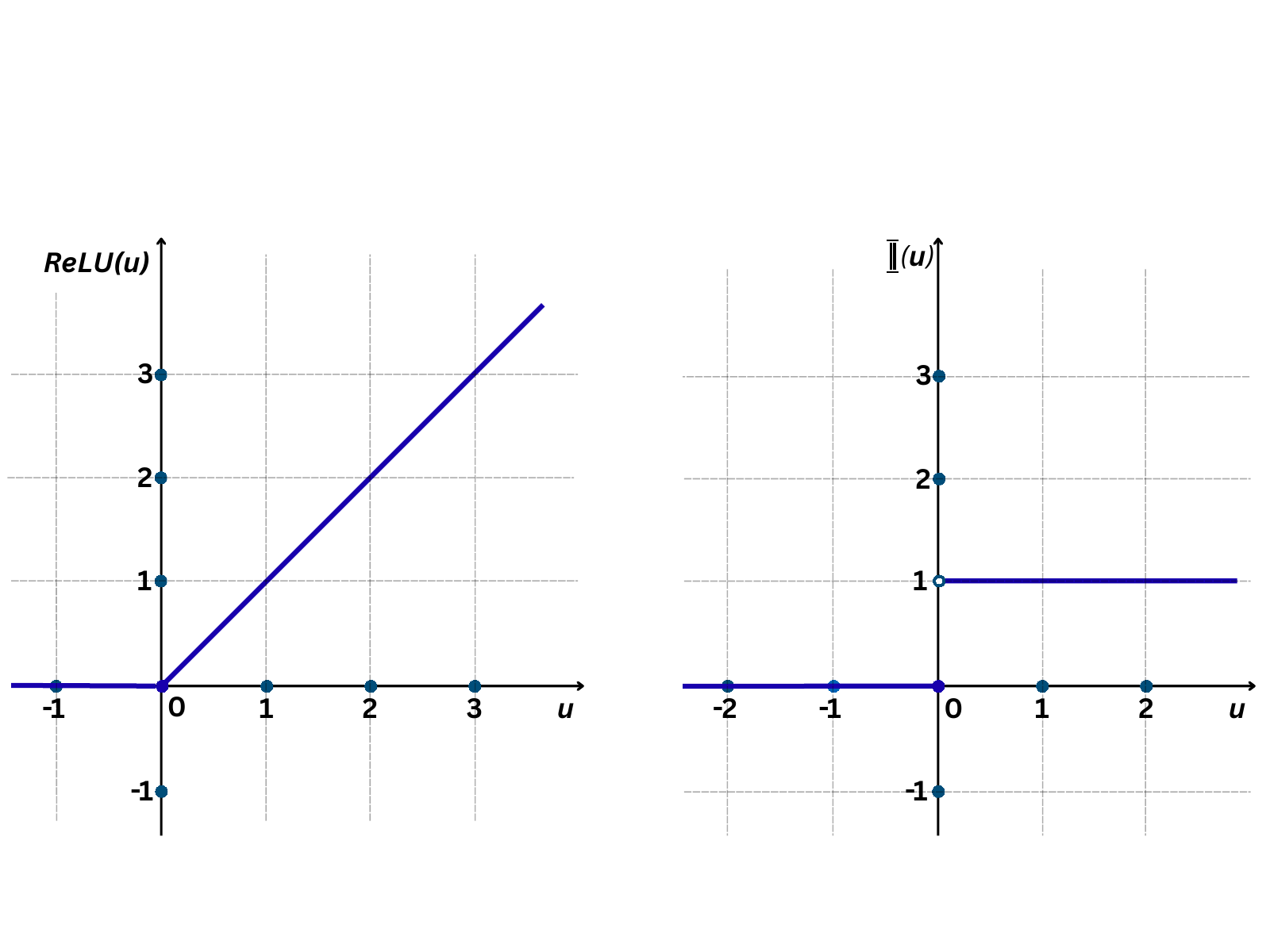}
\caption{\textbf{ReLU $(u)$ \& $\mathbb{I}(u)$} Maximising each of these functions implies encouraging their internal expressions to be positive. In the case of the Heaviside function, this corresponds to driving the argument above zero to trigger a non-zero response. For the ReLU function, the goal is to make the argument as large and positive as possible to maximise the acquisition value.}
\label{fig:functions_acq}
\end{figure}

Common choices of acquisition functions in Bayesian optimisation encode different strategies for balancing exploration of uncertain regions and exploitation of promising areas suggested by the surrogate model. Popular choices include Probability of Improvement (PI), which prioritises regions likely to outperform the current best; Expected Improvement (EI), which accounts for both probability and magnitude of improvement; and Upper Confidence Bound (UCB), which explicitly trades off mean prediction and uncertainty. More recent criteria, such as information-theoretic acquisition functions \cite{wangMaxvalueEntropySearch2017}, are not included in this tutorial. Figure~\ref{fig:surrogate:surrogate_acquisitions} illustrates typical shapes of these acquisition functions over a one-dimensional input space.

\paragraph{Expected Improvement (EI)}
The expected improvement has been one of the most commonly used acquisition functions since the early works of \textcite{mockusBayesMethodsSeeking1975, jonesEfficientGlobalOptimization1998}:
\begin{align*}
    \alpha_{\text{EI}}(x\mid\mathcal{D}) 
    = \mathbb{E}_{f \sim \mathcal{M}(\cdot \mid\cdot,\mathcal{D})}\left[\operatorname{ReLU}\bigl(f(x) - f(x^{+})\bigr)\right],
\end{align*}
where $x^+$ is the best performing design in the data so far. The ReLU function is defined as $\operatorname{ReLU}(u) = \max\{0,u\}$ and is illustrated in Figure~\ref{fig:functions_acq}. Here, the subscript $f \sim \mathcal{M}(\cdot \mid\cdot,\mathcal{D})$ indicates that the evaluation of $f$ complies with the posterior surrogate $\mathcal{M}(\hat{\mathbf{f}}|\hat{\mathbf{x}} ,\mathcal{D})$ given data $\mathcal{D}$.

\paragraph{Probability of Improvement (PI)}
PI is another classic acquisition function, first introduced by \textcite{kushnerNewMethodLocating1964}:
\begin{align*}
    \alpha_{\text{PI}}(x\mid\mathcal{D}) 
    = \mathbb{E}_{f \sim \mathcal{M}(\cdot \mid\cdot,\mathcal{D})}\left[\mathbb{I}\left(f(x) - f(x^+)\right)\right],
\end{align*}
where $\mathbb{I}(u)$ is the left-continuous Heaviside step function,
\[
\mathbb{I}(u) = 
\begin{cases}
0 & \text{if } u \leq 0, \\
1 & \text{if } u > 0.
\end{cases}
\]
PI measures the probability that sampling at $x$ will yield an improvement over the current best observation $f(x^+)$. While intuitive and straightforward, PI can be overly exploitative unless an improvement threshold is introduced.

\paragraph{Upper Confidence Bound (UCB)}
UCB is a canonical acquisition function inspired by multi-armed bandit problems \cite{auerUsingConfidenceBounds2002, lattimoreBanditAlgorithms2020}:
\begin{align*}
    \alpha_{\text{UCB}}(x\mid\mathcal{D}) 
    = \mu(x) + \beta\,\sigma(x),
\end{align*}
where $\mu(x)$ and $\sigma(x)$ are the posterior mean and standard deviation of $f(x)$, which are estimated from the surrogate model at $x$. The parameter $\beta > 0$ controls the trade-off between exploration and exploitation: larger values of $\beta$ encourage more exploration, while smaller values favour exploitation. UCB is appealing for its simplicity and strong theoretical guarantees in certain settings.

\paragraph{Thompson Sampling (TS)}
Thompson Sampling \cite{russoTutorialThompsonSampling2020} is an alternative acquisition strategy rooted in bandit algorithms. At each iteration, one samples a function $f$ from the posterior surrogate model, and selects the maximiser:
\begin{align*}
    \alpha_{\text{TS}}(x\mid\mathcal{D}) = f(x), \qquad f \sim \mathcal{M}(\cdot \mid\cdot,\mathcal{D}).
\end{align*}
Thompson Sampling inherently balances exploration and exploitation by sampling according to the model's current uncertainty. Its conceptual simplicity and strong theoretical guarantees make it attractive, especially in parallel or distributed settings.

\begin{figure}
    \centering
    \includegraphics[width=1.0\linewidth]{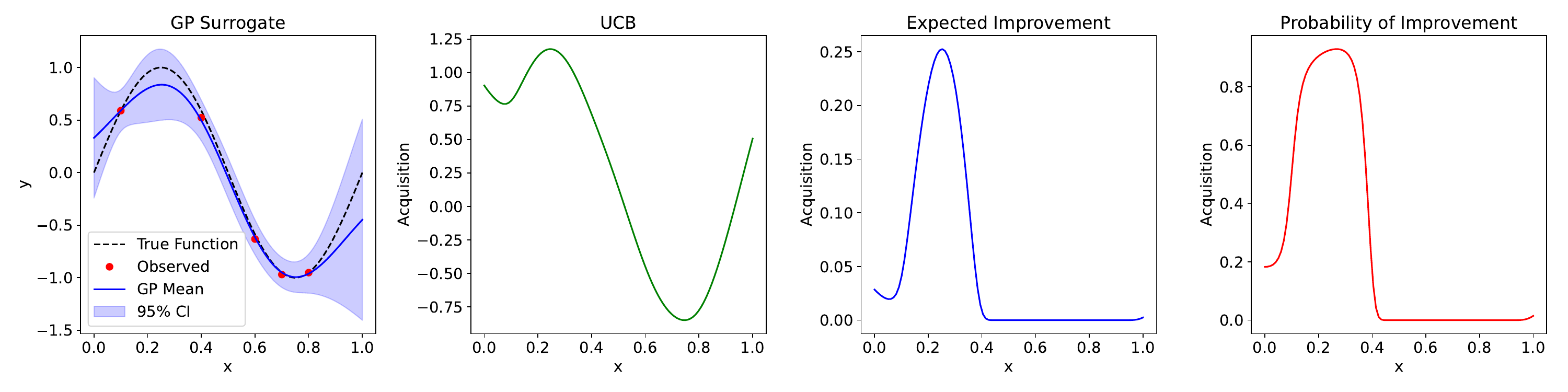}
    \caption{Visualization of the Gaussian Process (GP) surrogate model and acquisition functions used in Bayesian Optimization. The true objective function, $f(x) = \sin(2\pi x)$, is shown as a dashed black line. Observed data points are marked as red circles, where the Gaussian noise satisfies $\mathcal{N}(0, 0.1^2)$. The GP surrogate model is depicted as the blue curve (mean prediction), with the shaded region representing the 95\% confidence interval ($\mu \pm 2\sigma$). The acquisition functions---Upper Confidence Bound (UCB), Expected Improvement (EI), and Probability of Improvement (PI)---are plotted in green, blue, and red, respectively, to illustrate their behaviour in guiding the selection of the next evaluation point. We do not plot Thompson Sampling as it is randomly drawn from the posterior surrogate.}
    \label{fig:surrogate:surrogate_acquisitions}
\end{figure}

It is important to note that the above acquisition functions are \emph{myopic}---they consider only the next decision based on the current estimation of the objective, without anticipating the outcomes of future experiments. In reality, the current decision can influence subsequent decisions; thus, it may be desirable to consider the long-term effects of actions. However, Bayesian Optimisation typically employs these myopic acquisition functions, as they have been found to be sufficiently effective and computationally efficient in practice. The rationale is that the surrogate model is itself an approximation of the true objective, so the cost of overly precise planning---predicting long-term effects given an approximate model---often outweighs the potential benefit. Nonetheless, non-myopic acquisition functions can be valuable in certain scenarios, particularly when the experimental budget is extremely limited (i.e., when only a very small number of trials are possible). For a more general and principled perspective on acquisition functions and decision making, we provide a theoretical discussion in Section~\ref{sec:decision:bayesian_decision_theory}.

\section{Bayesian Decision Theory for Optimisation \THEORY}
\label{sec:decision:bayesian_decision_theory}

This section provides the theoretical motivation and justification for the acquisition function.

\subsection{Utility and Bayesian Decision Making}

Given the entire optimisation history $\mathcal{D}_t := \{(x_i, y_i)\}_{i=1}^t$, we define the utility as a function of the $s$-step look-ahead, $u(\mathcal{D}_{t+s})$. For example, a typical utility is the best objective value ever observed after $s$ additional experiments:
\[
u(\mathcal{D}_{t+s}) := \max_{t+1 \leq i \leq t+s} y_i - \max_{1 \leq i \leq t} y_i.
\]
Other possible utility functions may also consider the reduction in model uncertainty (information gain), which is the basis for Bayesian Experimental Design (BED) \cite{rainforthModernBayesianExperimental2023}. Since our aim here is to elucidate the mathematical foundation of acquisition functions, we do not cover the advanced design of utility functions in this tutorial.

Bayesian decision making, given the model, requires us to choose $x_{t+1}$ so as to maximise the expected posterior utility:
\begin{equation}
x_{t+1} = \arg \max_{x \in \mathcal{X}}\mathbb{E}[u(\mathcal{D}_{t+s})|\mathcal{D}_t, x]
\label{eq:decision_making_utiltiy}
\end{equation}
This formulation makes explicit that future experiments are involved and must be predicted using the current model, taking into account both future model adaptation and the subsequent policy. In other words, the optimal decision at each step depends on the expected outcomes of all future decisions and the evolution of the model as new data are collected. This recursive dependency is a hallmark of sequential decision making and highlights the complexity inherent in optimal experimental design.

\subsection{From Utility to Acquisition Function}
From Eq.~\ref{eq:decision_making_utiltiy}, we see that decision making is, at its core, the optimisation of a function, which we denote as
\[
\alpha(x|\mathcal{D}_t) := \mathbb{E}[u(\mathcal{D}_{t+s})|\mathcal{D}_t, x].
\]
This function serves as the acquisition function, encapsulating the expected utility of each possible next action. Thus, the acquisition function is not arbitrary, but is directly motivated by the principle of maximising expected utility in Bayesian decision theory. Moreover, each commonly used acquisition function can be understood as an implicit treatment of a particular choice of utility function, i.e., a concrete way of quantifying the value of future data or decisions.

\paragraph{Myopic Acquisition Function}
If the step look-ahead is $s = 1$, the acquisition function simplifies significantly, as it considers only the immediate effect of the next experiment. This leads to the so-called ``myopic'' acquisition functions, which are widely used in practice due to their computational tractability. Myopic acquisition functions do not require explicit modelling of future decisions or model updates beyond the next experiment, making them both practical and efficient.

\paragraph{Non-myopic Acquisition Function}
If the look-ahead $s > 1$, the acquisition function becomes substantially more challenging to compute, as it involves predicting future observations, updating the model, and accounting for subsequent decision-making behaviour. While non-myopic acquisition functions offer the potential advantage of more globally optimal planning \cite{yueWhyNonmyopicBayesian2020}, their implementation is often computationally intractable, even for a two-step look-ahead \cite{wuPracticalTwoStepLookahead2019}. Various approximate methods have been proposed (e.g., \cite{streltsovNonmyopicUtilityFunction1999,gonzalezGLASSESRelievingMyopia2016}), but these approaches are seldom used in practice due to their complexity and high computational cost. Nonetheless, in scenarios with extremely limited experimental budgets, non-myopic strategies may offer significant benefits by allowing more informed and farsighted planning.

\section{Coding the Acquisition Function \CODE}

\subsection{Specifying Acquisition Functions in BO Software}

The acquisition function (AF) plays a central role in Bayesian Optimisation (BO) algorithms, as it determines the strategy by which the next evaluation points are selected. In this section, we present the sample code to specify an acquisition function in representative BO software.

\subsubsection{Acquisition Functions in HEBO}
\label{sec:decision:coding:acq_impl:HEBO}

In practical software implementations, the choice of acquisition function is typically encapsulated within the optimiser interface, and may be set either by default or explicitly by the user. For example, in the HEBO framework \cite{cowen-riversHEBOPushingLimits2022}, the acquisition function is selected automatically according to the optimiser class, but can also be customised by specifying the acquisition class and its configuration.

\begin{lstlisting}[language=Python]
from hebo.optimizers.bo import BO
from hebo.optimizers.hebo import HEBO

# Construct your DesignSpace (replace with actual space construction)
space = ...  # e.g. DesignSpace().parse(...)

# BO: uses the default acquisition (LCB, lower confidence bound, for minimisation)
opt_bo = BO(space)

# HEBO: high-level optimizer; by default uses MACE, a multi-objective synthesis of EI, PI, and LCB
opt_hebo = HEBO(space)
\end{lstlisting}

The HEBO package also supports customisation of the acquisition function. HEBO supports a variety of acquisition functions, each corresponding to a different strategy for balancing exploration and exploitation. The main acquisition classes available in \texttt{hebo.acquisitions.acq} are as follows:

\begin{itemize}
    \item \texttt{LCB} (Lower Confidence Bound): This is the minimisation-oriented version of the popular UCB acquisition (see Section~\ref{sec:decision:common_acq_choices}). In HEBO, \texttt{LCB} is defined as $\mu(x) - \kappa \sigma(x)$, where $\mu(x)$ and $\sigma(x)$ are the predictive mean and standard deviation at $x$, and $\kappa$ is a trade-off parameter. This encourages exploration of uncertain regions while seeking low objective values.
    \item \texttt{Mean}: Selects points with the lowest predicted mean, i.e., pure exploitation without consideration of uncertainty.
    \item \texttt{Sigma}: Selects points with the highest predictive uncertainty, i.e., pure exploration.
    \item \texttt{EI} (Expected Improvement): Measures the expected improvement (see Section~\ref{sec:decision:common_acq_choices}) over the current best observation, balancing exploitation and exploration in a principled manner.
    \item \texttt{logEI}: The logarithm of the expected improvement, which can be more numerically stable in some cases.
    \item \texttt{WEI} and \texttt{Log\_WEI}: Weighted expected improvement and its logarithmic variant, which can be used for multi-objective or constrained settings.
    \item \texttt{MES} (Max-value Entropy Search): Focuses on reducing the uncertainty about the location of the global minimum \cite{wangMaxvalueEntropySearch2017}.
    \item \texttt{MOMeanSigmaLCB}: A multi-objective acquisition that considers mean, standard deviation, and LCB jointly.
    \item \texttt{MACE}: The default in HEBO, this is a multi-objective acquisition function that synthesises several canonical criteria (EI, PI, LCB) into a unified acquisition, thus balancing multiple exploration--exploitation strategies.
\end{itemize}

It is important to note that in HEBO, all acquisition functions are implemented for minimisation tasks. For example, \texttt{LCB} is the lower confidence bound, as opposed to the upper confidence bound more commonly used for maximisation.

The following code demonstrates how to instantiate HEBO or BO with explicit acquisition classes, including canonical single-objective functions such as EI, and the multi-objective MACE:

\begin{lstlisting}[language=Python]
from hebo.optimizers.bo import BO
from hebo.optimizers.hebo import HEBO
from hebo.acquisitions.acq import LCB, Mean, Sigma, EI, MACE, GeneralAcq

space = ...  # your DesignSpace

# 1) Use LCB (lower confidence bound) with a custom kappa value for exploration-exploitation balance
opt_bo_lcb = BO(space, acq_cls=LCB, acq_conf={'kappa': 3.0})

# 2) Use Mean (pure exploitation) or Sigma (pure exploration)
opt_bo_mean  = BO(space, acq_cls=Mean)
opt_bo_sigma = BO(space, acq_cls=Sigma)

# 3) Use Expected Improvement (EI), a canonical acquisition function
opt_bo_ei = BO(space, acq_cls=EI)

# 4) Use MACE, the multi-objective synthesis acquisition (default for HEBO)
opt_hebo_mace = HEBO(space, acq_cls=MACE)
\end{lstlisting}

\subsubsection{Acquisition Functions in Bgolearn}
\label{sec:decision:coding:acq_impl:bgolearn}
\begin{center}

\CODE\,\url{https://github.com/zwyu-ai/BO-Tutorial-for-Sci/blob/main/coding_illustrations/gp_bo_bgolearn.ipynb}
\end{center}

Bgolearn provides a unified framework for materials discovery, within which acquisition functions can be readily invoked. A simple implementation is shown below:

\begin{lstlisting}[language=Python]
import pandas as pd
data = pd.read_excel('./data/train.xlsx')
vs = pd.read_excel('./data/visual_samples.xlsx')

import Bgolearn.BGOsampling as BGOS
Bgolearn = BGOS.Bgolearn() 

Mymodel = Bgolearn.fit(
    data_matrix=data.iloc[:, :-1],
    Measured_response=data.iloc[:, -1],
    virtual_samples=vs,
)

Mymodel.EI()  # Expected Improvement
# Mymodel.Knowledge_G()  # Knowledge Gradient
# Mymodel.PoI()  # Probability of Improvement
\end{lstlisting}

Different acquisition functions can be accessed by modifying the corresponding method calls. For further details, please refer to the official documentation: \url{https://bgolearn.netlify.app/acquisition_functions}.

\subsection{Implementing Acquisition Functions from Your Surrogate Models}
\label{sec:decision:coding:acq_impl}

In this section, we demonstrate how to implement common acquisition functions---Expected Improvement (EI), Probability of Improvement (PI), and Upper Confidence Bound (UCB)---using the Gaussian Process surrogates constructed in Section~\ref{sec:surrogate:coding:gp_impl}. These acquisition functions are popular to Bayesian optimisation, as discussed mathematically in Section~\ref{sec:decision:common_acq_choices}. The code examples below build directly on the GP implementations using Scikit-Learn and GPyTorch; thus, we focus on the acquisition function logic, referencing the surrogate model code as needed.

\subsubsection{Acquisition Functions for Scikit-Learn Gaussian Process}
\label{sec:decision:coding:acq_impl:sklearn}

\begin{center}
\CODE\,\url{https://github.com/zwyu-ai/BO-Tutorial-for-Sci/blob/main/coding_illustrations/gp_bo_sklearn.ipynb}
\end{center}

Given a fitted \texttt{GaussianProcessRegressor}, the posterior mean $\mu(x)$ and standard deviation $\sigma(x)$ at candidate points $x$ can be obtained using the \texttt{predict} method with \texttt{return\_std=True}. The following code implements UCB, EI, and PI as Python functions, which can be evaluated at any candidate points:

\begin{lstlisting}[language=Python]
import numpy as np
from scipy.stats import norm

def acq_ucb(x, gp, beta=2.0):
    """Upper Confidence Bound acquisition."""
    mu, std = gp.predict(x, return_std=True)
    return mu + beta * std

def acq_ei(x, gp, y_best, xi=0.0):
    """Expected Improvement acquisition."""
    mu, std = gp.predict(x, return_std=True)
    std = std.reshape(-1)
    mu = mu.reshape(-1)
    with np.errstate(divide='warn'):
        imp = mu - y_best - xi
        Z = imp / std
        ei = imp * norm.cdf(Z) + std * norm.pdf(Z)
        ei[std == 0.0] = 0.0
    return ei

def acq_pi(x, gp, y_best, xi=0.0):
    """Probability of Improvement acquisition."""
    mu, std = gp.predict(x, return_std=True)
    std = std.reshape(-1)
    mu = mu.reshape(-1)
    with np.errstate(divide='warn'):
        Z = (mu - y_best - xi) / std
        pi = norm.cdf(Z)
        pi[std == 0.0] = 0.0
    return pi
\end{lstlisting}

The parameter \texttt{y\_best} should be set to the best (e.g., lowest) observed value so far, consistent with a minimisation objective. The optional parameter \texttt{xi} controls the improvement threshold for EI/PI, and \texttt{beta} controls the exploration-exploitation trade-off for UCB.

\subsubsection{Acquisition Functions for GPyTorch Gaussian Process}
\label{sec:decision:coding:acq_impl:gpytorch}

\begin{center}
\CODE\,\url{https://github.com/zwyu-ai/BO-Tutorial-for-Sci/blob/main/coding_illustrations/gp_bo_gpytorch.ipynb}
\end{center}

For GPyTorch models, the workflow is analogous: after switching the model and likelihood to evaluation mode, the predictive mean and variance are computed for candidate points. The acquisition functions can then be implemented as follows:

\begin{lstlisting}[language=Python]
import torch
from scipy.stats import norm

def acq_ucb_torch(x, model, beta=2.0):
    """UCB for GPyTorch (x should be torch tensor, shape [N, D])."""
    model.eval()
    with torch.no_grad():
        pred = model(x)
        mu = pred.mean
        std = pred.stddev
        return mu + beta * std

def acq_ei_torch(x, model, y_best, xi=0.0):
    """EI for GPyTorch (returns numpy array)."""
    model.eval()
    with torch.no_grad():
        pred = model(x)
        mu = pred.mean.cpu().numpy()
        std = pred.stddev.cpu().numpy()
        imp = mu - y_best - xi
        with np.errstate(divide='warn'):
            Z = imp / std
            ei = imp * norm.cdf(Z) + std * norm.pdf(Z)
            ei[std == 0.0] = 0.0
        return ei

def acq_pi_torch(x, model, y_best, xi=0.0):
    """PI for GPyTorch (returns numpy array)."""
    model.eval()
    with torch.no_grad():
        pred = model(x)
        mu = pred.mean.cpu().numpy()
        std = pred.stddev.cpu().numpy()
        with np.errstate(divide='warn'):
            Z = (mu - y_best - xi) / std
            pi = norm.cdf(Z)
            pi[std == 0.0] = 0.0
        return pi
\end{lstlisting}

These implementations are fully compatible with the GPyTorch workflow introduced previously. Note that candidate points \texttt{x} should be supplied as PyTorch tensors, and the acquisition values can be used for selection or visualisation as required.

%% file: v2/algorithm.tex
\part{Algorithmic Workflow and Implementation}
\label{part:alg}

This part is dedicated to a systematic exposition of the workflow, implementation, and practical extensions of Bayesian Optimisation (BO), building upon the foundational concepts established in earlier sections: specifically, the framing of scientific discovery as a black-box optimisation problem, the design and use of surrogate models, and sequential decision-making via acquisition functions.

The organisation of this part follows a logical progression from core algorithmic principles to real-world scientific applications. First, we present a compact yet comprehensive overview of the elementary BO workflow, detailing its iterative loop and key considerations for each sub-module. Next, we focus on the canonical GP-UCB algorithm---a widely studied instantiation of BO---and formalise its theoretical guarantees (i.e., sub-linear regret). Subsequently, we address the distinctive challenges of BO in scientific discovery by reviewing technical extensions tailored to real experimental settings. We then translate these concepts into actionable code, providing implementations using state-of-the-art BO libraries (e.g., HEBO) as well as custom workflows built from scratch with Gaussian process frameworks (Scikit-Learn and GPyTorch). Finally, we experimentally validate the efficacy of BO on four representative scientific-discovery tasks spanning catalysis, materials science, and organic synthesis.

\section{Bayesian Optimisation Algorithm \ELEMENTARY}

Here, we present a compact description of the basic BO workflow. The algorithm receives an optimisation problem, including the objective domain (i.e., the design space) and an objective evaluator (e.g., manual experimentation or an external interface), and returns the best solution found. We first introduce the overall workflow and then discuss key considerations for the involved sub-modules.

\subsection{Algorithmic Workflow}

The Bayesian optimisation workflow proceeds as follows. The algorithm begins by selecting a set of initial observations to construct an initial dataset. At each iteration, the surrogate model is updated with all available data, and the acquisition function is constructed based on the surrogate’s predictions and associated uncertainties. The acquisition function is then maximised to propose the next candidate point, which is subsequently evaluated using the true objective function (e.g., through experiment or simulation). The new observation is added to the dataset, and this process repeats until a stopping criterion is met. Finally, the algorithm recommends a solution based on the accumulated data and/or the surrogate model. This workflow is summarised in Algorithm~\ref{alg:basic_bo}, and the main optimisation loop is illustrated in Figure~\ref{fig:BO_loop}.

\begin{algorithm}
    \caption{Bayesian Optimisation (Basic)}
  \label{alg:basic_bo}
  \begin{algorithmic}[1]
    \REQUIRE Design space \(\mathcal{X}\), objective evaluator $\operatorname{evaluate}(f, x)$
    \STATE Choose a finite subset $\mathbf{x}_{init} = (x_1,\ldots,x_{n\_init}) \subset \mathcal{X}$ for \emph{initial observations}.
    \STATE Evaluate $\mathbf{y}_{init} \gets \operatorname{evaluate}(f, \mathbf{x}_{init})$
    \STATE Construct dataset $\mathcal{D} \gets \{(x_i, y_i)\}_{i=1}^{n\_init}$
    \REPEAT
      \STATE Fit \emph{surrogate model} \(\mathcal{M}(\hat{\mathbf{f}}|\hat{\mathbf{x}}, \mathcal{D})\) from data $\mathcal{D}$
      \STATE Construct \emph{acquisition function} \(\alpha(x\mid\mathcal{D})\) from $\mathcal{D}$ and the surrogate \(\mathcal{M}(\hat{\mathbf{f}}|\hat{\mathbf{x}}, \mathcal{D})\)
    \STATE Select the next point by \emph{maximising the acquisition function}: \(x' \gets \arg\max_{x\in\mathcal{X}} \alpha(x\mid\mathcal{D})\)
      \STATE Evaluate the objective function: \(y' \gets \mathrm{evaluate}(f, x')\)
      \STATE Augment data: \(\mathcal{D} \gets \mathcal{D} \cup \{(x',y')\}\)
    \UNTIL{\emph{Stopping criteria} are satisfied}
    \STATE \emph{Recommended final solution} \(x^*\) from $\mathcal{D}$ and \(\mathcal{M}(\hat{\mathbf{f}}|\hat{\mathbf{x}}, \mathcal{D})\).
    \RETURN \(x^*\)
  \end{algorithmic}
\end{algorithm}

\begin{figure}[tbh]
\centering
\includegraphics[trim={0em 0em 0em 0em}, clip=true, width=0.8\linewidth]{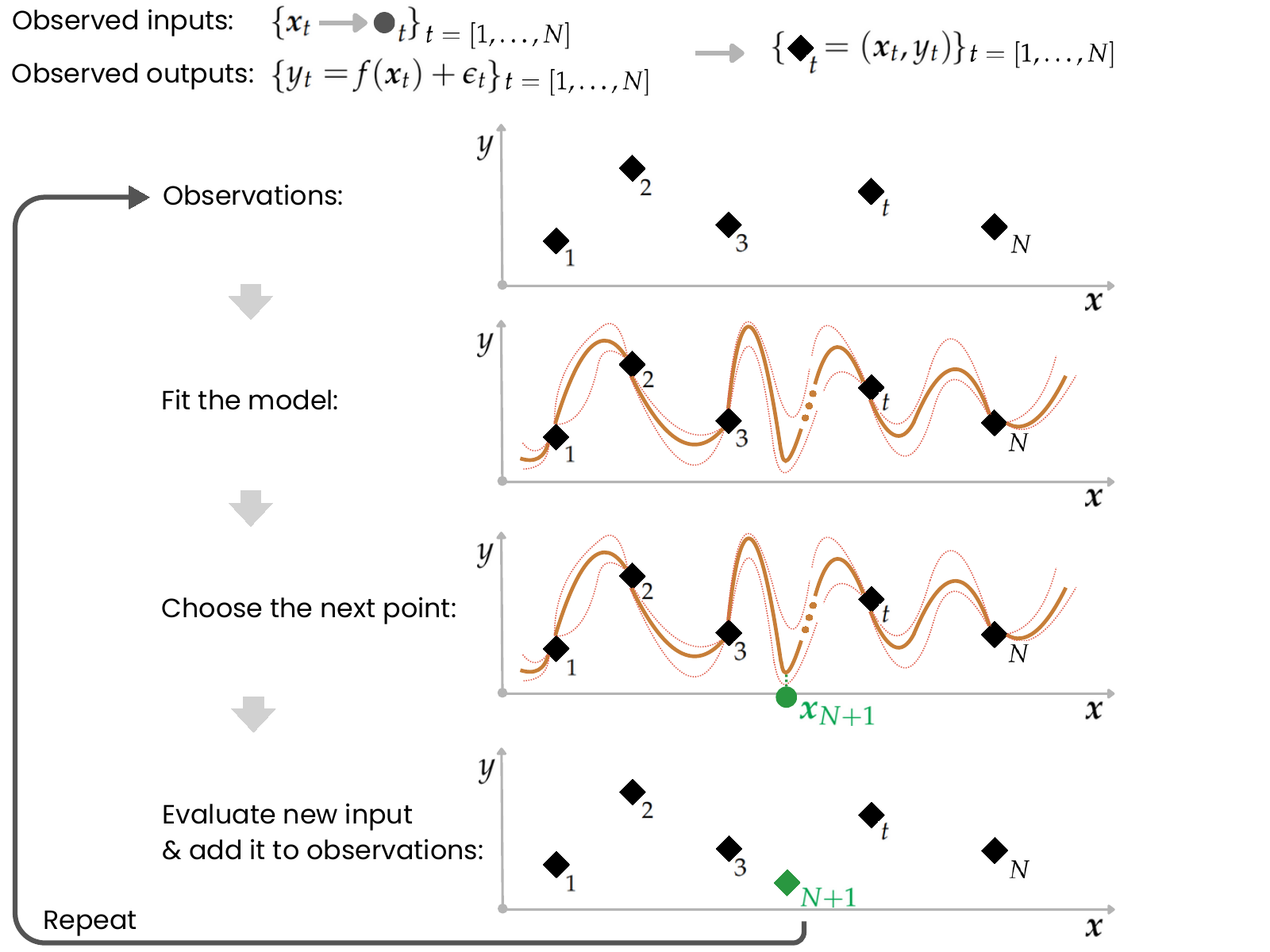}
\caption{\textbf{Bayesian Optimisation Loop:} At round $N$, a surrogate model (e.g., a Gaussian process) is fitted to the observed data points ($\{\boldsymbol{x}_t, y_t\}_{t=1}^{N}$). The next point $\boldsymbol{x}_{N+1}$ is selected by optimising an acquisition function derived from the model posterior. The new observation $y_{N+1}$ is then evaluated and added to the dataset, iteratively refining the model. The solid brown line depicts the mean prediction of the model, while the dashed lines illustrate the associated variance at each input point $\boldsymbol{x}$.}
\label{fig:BO_loop}
\end{figure}

It is important to note that the above describes the basic version of Bayesian optimisation. Numerous extensions exist that modify this workflow, such as batched or parallel acquisition (selecting multiple points per iteration), human-in-the-loop interaction (where humans guide or modify the search), or the use of more sophisticated surrogate models and acquisition strategies. These extensions can further enhance the efficiency and flexibility of Bayesian optimisation in various applications.

\subsection{Consideration of Sub-Modules}

The BO workflow described above comprises several key sub-modules \cite{hoffmanModularMechanismsBayesian2014}. The two central components---the surrogate model and the acquisition function---have been introduced in Part~\ref{part:surrogate} and Part~\ref{part:decision}. In addition, several auxiliary sub-modules are essential for a complete BO implementation, as outlined below.

\paragraph{The strategy for choosing initial observations}
Before entering the main optimisation loop, it is necessary to initialise the dataset with a set of initial observations. This is because the surrogate model requires data to build its initial estimate of the objective function. While it is theoretically possible to start with an empty dataset if the surrogate model has a well-defined prior, this is uncommon in practice. Typically, initial observations are collected to better inform the model and avoid poor early decisions. The simplest strategy is to randomly sample a small number of points from the design space (e.g., uniform sampling). More sophisticated approaches include space-filling designs \cite{jonesEfficientGlobalOptimization1998}, such as Latin hypercube sampling, which aim to cover the design space more evenly. Alternatively, expert or human suggestions can be used to select initial points, especially when prior knowledge about promising regions is available.

\paragraph{The maximisation tool for the acquisition function}
At each iteration, the acquisition function must be maximised to determine the next point to evaluate---this is often referred to as the \emph{inner optimisation problem} \cite{wilsonMaximizingAcquisitionFunctions2018}. The choice of maximisation method depends on the nature of the design space. For small, finite spaces, exhaustive search is feasible and guarantees the optimal value. For low-dimensional continuous spaces, discretising the space into a finite grid can be effective. When the acquisition function is differentiable (as is the case with certain surrogates), gradient-based numerical optimisation methods \cite{nocedalNumericalOptimization2006} (often with multiple restarts to improve global search) can efficiently locate maxima. For more complex or mixed-variable spaces (e.g., combinations of discrete and continuous variables) or for multi-objective acquisition, heuristic search methods such as evolutionary algorithms or random search are commonly employed \cite{cowen-riversHEBOPushingLimits2022}.

\paragraph{The stopping criteria}
The optimisation process must be terminated at an appropriate point, and various stopping criteria can be applied. The simplest and most widely-used approach is to specify a fixed budget, such as a maximum number of iterations or function evaluations. More sophisticated criteria can be designed based on the observed data, the surrogate model, or the acquisition function. For instance, a data-based criterion might monitor the empirical improvement over recent steps and halt if the improvement falls below a threshold, indicating convergence. A model-based criterion could measure the surrogate model’s uncertainty and stop when the model becomes sufficiently confident. Acquisition-based criteria might terminate the loop if the maximum value of the acquisition function drops below a certain threshold, signalling diminishing returns. In more advanced settings, the decision of whether to continue can itself be formalised using Bayesian decision theory \cite{daiBayesianOptimizationMeets2019}.

\paragraph{The final recommendation}
Upon termination of the optimisation loop, the algorithm must recommend a final solution. The most common approach is to select the best observed point in the data (i.e., the input with the highest or lowest observed objective value, depending on the task). Alternatively, the surrogate model can be used to predict the optimum (e.g., the input with the best predicted mean), which may be beneficial if the observations are noisy or sparse.

\section{The GP-UCB Algorithm and Theoretical Guarantees \TECHNICAL\,\THEORY}
\label{sec:alg:gp-ucb}

In this section, we present a canonical and widely studied instance of the Bayesian Optimisation (BO) algorithm: \emph{GP-UCB}. This algorithm employs a Gaussian process (GP) as the surrogate model (see Section~\ref{sec:surrogate:gp}) and the Upper Confidence Bound (UCB) as the acquisition function (see Section~\ref{sec:decision:common_acq_choices}). GP-UCB is not only a standard baseline in practice, but also the subject of extensive theoretical analysis, providing strong regret guarantees for BO~\cite{srinivasInformationTheoreticRegretBounds2012}. We first describe the algorithm and its implementation, then summarise its theoretical guarantees.

\subsection{GP-UCB Algorithm}

The GP-UCB algorithm assumes that the black-box objective function $f(x)$ is drawn from a Gaussian process prior with mean function $m(x)$ and covariance function $k(x, x')$, and that observations are corrupted by independent, identically distributed (i.i.d.) Gaussian noise. That is, for each query $x_t$, we observe $y_t = f(x_t) + \epsilon_t$ with $\epsilon_t \sim \mathcal{N}(0, \sigma^2)$.

Given a history of decisions $\mathbf{x} = (x_1, \ldots, x_t)$ and their corresponding (possibly noisy) observations $\mathbf{y} = (y_1, \ldots, y_t)^\top$, the posterior predictive distribution for a new candidate $\hat{x}$ is
\begin{equation}
\mathcal{M}(\hat{f}|\hat{x}; \mathbf{x}, \mathbf{y})
= \mathcal{N} \left(
m(\hat{x}) + \mathbf{k}_{\hat{x}\mathbf{x}} [\mathbf{K}_{\mathbf{x}\mathbf{x}} + \sigma^2 \mathbf{I}]^{-1} (\mathbf{y} - \mathbf{m}),\;
k(\hat{x}, \hat{x}) - \mathbf{k}_{\hat{x}\mathbf{x}} [\mathbf{K}_{\mathbf{x}\mathbf{x}} + \sigma^2 \mathbf{I}]^{-1} \mathbf{k}_{\mathbf{x}\hat{x}}
\right),
\end{equation}
where:
\begin{itemize}
    \item $\mathbf{k}_{\hat{x}\mathbf{x}} = [k(\hat{x}, x_1), \ldots, k(\hat{x}, x_t)]$ is the row vector of covariances between $\hat{x}$ and observed points, and $\mathbf{k}_{\mathbf{x}\hat{x}}$ is its transpose.
    \item $\mathbf{K}_{\mathbf{x}\mathbf{x}}$ is the $t \times t$ covariance matrix with $[\mathbf{K}_{\mathbf{x}\mathbf{x}}]_{ij} = k(x_i, x_j)$,
    \item $\mathbf{m} = [m(x_1), \ldots, m(x_t)]^\top$.
\end{itemize}
The posterior mean and variance are thus:
\begin{align}
    \mu(\hat{x}|\mathbf{x}, \mathbf{y}) &=  m(\hat{x}) + \mathbf{k}_{\hat{x}\mathbf{x}} [\mathbf{K}_{\mathbf{x}\mathbf{x}} + \sigma^2 \mathbf{I}]^{-1} (\mathbf{y} - \mathbf{m}) \label{eq:gp-ucb:post-mean}\\
    \sigma^2(\hat{x}|\mathbf{x}, \mathbf{y}) &= k(\hat{x}, \hat{x}) - \mathbf{k}_{\hat{x}\mathbf{x}} [\mathbf{K}_{\mathbf{x}\mathbf{x}} + \sigma^2 \mathbf{I}]^{-1} \mathbf{k}_{\mathbf{x}\hat{x}}
    \label{eq:gp-ucb:post-var}
\end{align}

The core idea of GP-UCB is to select, at each iteration, the next point by maximising an upper confidence bound constructed from the GP posterior:
\begin{equation}
    \alpha_{\beta_t}(\hat{x}|\mathbf{x}, \mathbf{y}) := \mu(\hat{x}|\mathbf{x}, \mathbf{y}) + \sqrt{\beta_t\, \sigma^2(\hat{x}|\mathbf{x}, \mathbf{y})},
    \label{eq:gp-ucb:acq}
\end{equation}
where $\beta_t \ge 0$ is an exploration coefficient, which may be fixed or adapted over time. This acquisition function encourages exploration of uncertain regions (via the posterior standard deviation) and exploitation of regions with high predicted mean.

A basic version of the GP-UCB algorithm is summarised in Algorithm~\ref{alg:gp-ucb}. The algorithm is typically initialised with an empty dataset (or a small set of initial observations) and proceeds for a fixed budget $T$ of iterations, with a dynamically chosen $\beta_t$.

\begin{algorithm}[h]
\caption{GP-UCB Algorithm}
\label{alg:gp-ucb}
\begin{algorithmic}[1]
\REQUIRE Input space \( \mathcal{X} \subseteq \mathbb{R}^d \); prior mean function $m$ (often \( m \equiv 0 \)), covariance function $k$, noise variance \( \sigma^2 \); confidence sequence \( \{\beta_t\}_{t \geq 1} \); budget $T$
\STATE Initialize data \( \mathcal{D}_0 = \emptyset \)
\FOR{ \( t = 1, 2, \ldots, T\) }
    \STATE Compute posterior mean $\mu(\hat{x}|\mathbf{x}, \mathbf{y})$ and variance $\sigma^2(\hat{x}|\mathbf{x}, \mathbf{y})$ using Equations~\eqref{eq:gp-ucb:post-mean} and~\eqref{eq:gp-ucb:post-var}
    \STATE Select $x_t = \arg \max_{\hat{x} \in \mathcal{X}} \left(\mu(\hat{x}|\mathbf{x}, \mathbf{y}) + \sqrt{\beta_t\, \sigma^2(\hat{x}|\mathbf{x}, \mathbf{y})}\right)$
    \STATE Observe \( y_t = f(x_t) + \epsilon_t \), with \( \epsilon_t \sim \mathcal{N}(0, \sigma^2) \)
    \STATE Update data: $\mathbf{x} \gets (x_1,\ldots,x_t)$, $\mathbf{y} \gets (y_1, \ldots, y_t)^\top$
\ENDFOR
\RETURN The best observed $x^* = \arg\max_{x_i \in \mathbf{x}} y_i$
\end{algorithmic}
\end{algorithm}

\subsection{Theoretical Guarantee: Regret Bounds}

To assess the performance of BO algorithms, the \emph{cumulative regret} $R_T$ over $T$ steps is a common measure (see Section~\ref{sec:sci_as_opt:theory}). For GP-UCB, strong theoretical guarantees are available: under mild conditions, the cumulative regret grows sublinearly with $T$, meaning that the average regret per step vanishes as $T$ increases. This implies that GP-UCB efficiently identifies near-optimal solutions as more function evaluations are performed.

The regret analysis for GP-UCB is based on information-theoretic arguments. The following two theorems summarise the main results for (i) finite and (ii) compact continuous design spaces. The full derivation is established by \textcite{srinivasInformationTheoreticRegretBounds2012} and is not included in this tutorial.

\begin{theorem}[Regret Bound for Finite Design Spaces]
\label{thm:gp-ucb:regret-finite}
Let \( \mathcal{X} \) be a finite set with cardinality \( |\mathcal{X}| < \infty \). Suppose \( f \sim \mathcal{GP}(m, k) \) and fix confidence level \( \delta \in (0,1) \). Define
\[
\beta_t = 2 \log \left( \frac{|\mathcal{X}| t^2 \pi^2}{6 \delta} \right).
\]
Then, with probability at least \( 1 - \delta \), the cumulative regret of GP-UCB satisfies
\[
R_T \leq \sqrt{C_1 T \beta_T \gamma_T} \quad \forall T \geq 1,
\]
where \( C_1 = \frac{8}{\log(1 + \sigma^{-2})} \) and $\gamma_T$ is the maximum information gain after $T$ steps.
\end{theorem}

\begin{theorem}[Regret Bound for Compact Design Spaces]
\label{thm:gp-ucb:regret-compact}
Let \( \mathcal{X} \subseteq [0, r]^d \) be compact and convex, and suppose \( f \sim \mathcal{GP}(m, k) \) with $k$ possessing continuous fourth-order derivatives. For confidence \( \delta \in (0,1) \), set
\[
\beta_t = 2 \log \left( \frac{t^2 2 \pi^2}{3 \delta} \right) + 2 d \log \left( t^2 d b r \sqrt{\log(4 d a / \delta)} \right),
\]
where $a, b > 0$ are constants from the smoothness condition
\[
\mathbb{P}\left\{ \sup_{x \in \mathcal{X}} \left| \frac{\partial f}{\partial x_j} \right| > L \right\} \leq a e^{-(L/b)^2}, \quad j = 1, \ldots, d.
\]
Then, with probability at least \( 1 - \delta \),
\[
R_T \leq \sqrt{C_1 T \beta_T \gamma_T} + 2 \quad \forall T \geq 1,
\]
where \( C_1 \) is as in Theorem~\ref{thm:gp-ucb:regret-finite} and $\gamma_T$ is the maximum information gain.
\end{theorem}

Theorems~\ref{thm:gp-ucb:regret-finite} and~\ref{thm:gp-ucb:regret-compact} show that, with high probability, the cumulative regret $R_T$ of GP-UCB grows at most as $\mathcal{O}(\sqrt{T \beta_T \gamma_T})$. The information-gain term $\gamma_T$ depends on the kernel $k$ and the complexity of the underlying function $f$. For commonly used kernels (e.g., squared exponential), $\gamma_T$ grows sublinearly with $T$, so the average regret per step $R_T / T \to 0$ as $T \to \infty$. This establishes that GP-UCB is \emph{no-regret}: the algorithm eventually identifies near-optimal solutions, and the search becomes increasingly efficient over time.

The choice of the exploration parameter $\beta_t$ plays a crucial role: it must be large enough to ensure sufficient exploration (and thus guarantee the regret bounds), but not so large as to excessively prioritise exploration over exploitation. The theoretically derived schedules for $\beta_t$ (as above) are sufficient for regret guarantees, but in practice, $\beta_t$ is often treated as a tunable hyperparameter and may be set heuristically or via cross-validation.

In summary, GP-UCB provides both a practical and theoretically justified approach to Bayesian Optimisation, balancing exploration and exploitation in a principled manner and enjoying strong guarantees on its optimisation performance.

\section{Review of Technical Extensions for Scientific Discovery \TECHNICAL}

Optimisation problems encountered in scientific discovery are often characterised by diverse challenges, rendering the implementation of Bayesian Optimisation (BO) indirect and non-trivial. In this section, we review technical extensions that are particularly valuable for scientific discovery, focusing on batched decision making, handling heteroscedasticity and non-stationarity, contextual decision making, and integrating human expertise within the optimisation loop.

\subsection{Batched Decision Making for High-Throughput Experiments}

In numerous scientific and engineering applications, evaluating multiple candidate designs in parallel is feasible and, in many cases, desirable. This scenario---referred to as \emph{batched} or \emph{parallel} Bayesian optimisation---naturally arises in high-throughput experimental platforms. By selecting a batch of $q$ points to evaluate in each round, rather than a single point, one can utilise available resources more efficiently and accelerate the overall optimisation process.

The central challenge in batched Bayesian optimisation is to select a set of $q$ candidates, $\{x_{t,1}, \ldots, x_{t,q}\}$, that collectively offer maximal value for the optimisation objective. Ideally, the batch should strike a balance between (i) exploiting the current surrogate model to sample promising candidates, and (ii) exploring uncertain regions to improve the global model. Moreover, the batch should avoid redundant or highly correlated points, ensuring that each experiment yields complementary information.

\paragraph{Batched Acquisition Functions}
A natural approach is to develop a batched acquisition function $\alpha(\hat{\mathbf{x}}|\mathcal{D})$ for an arbitrary batch $\hat{\mathbf{x}} = (\hat{x}_1,\ldots,\hat{x}_q) \in \mathcal{X}^q$ of $q$ decision points. Extending standard acquisition functions to the batched setting is, however, non-trivial, as the value of evaluating a set of points jointly is not, in general, the sum of their individual acquisition values. For instance, the \emph{q-Expected Improvement} ($q$-EI)~\cite{ginsbourgerMultipointsCriterionDeterministic2008, ginsbourgerKrigingWellSuitedParallelize2010} generalises the standard expected improvement to batches by considering the expected improvement from the best value among the $q$ new observations, taking into account their joint distribution under the GP posterior. Nevertheless, $q$-EI becomes intractable for $q > 2$ and typically requires approximate inference \cite{chevalierFastComputationMultiPoints2013}. Other notable batched acquisition functions include the \emph{parallel knowledge gradient}~\cite{wuParallelKnowledgeGradient2016} and \emph{batched entropy search}~\cite{shahParallelPredictiveEntropy2015}, which directly account for the information gain or expected utility from evaluating a set of points together. These approaches are generally more computationally demanding than their sequential counterparts, but can yield more informative batches.

\paragraph{Batching via Acquisition Maximisation}
Alternatively, batching can be achieved by employing a single-point acquisition function, where batching is implemented through maximisation of the acquisition function. One practical approach is to generate a batch by optimising the acquisition function multiple times, each time conditioning on the pending (fantasised) outcomes of previous batch members~\cite{snoek2012practical, gonzalezBatchBayesianOptimization2016}. Recent methods, such as HEBO~\cite{cowen-riversHEBOPushingLimits2022}, use population-based optimisers (e.g., evolutionary algorithms) to maximise the acquisition function over the input space, naturally producing a diverse set of high-acquisition candidates. This approach is particularly well-suited to high-dimensional or mixed-value design spaces.

\subsection{Heteroscedasticity and Non-Stationarity}

Traditional BO often assumes simple i.i.d. Gaussian noise and a stationary covariance for uncertainty estimation, particularly when using a Gaussian process surrogate function. These assumptions are known as \emph{homoscedasticity} and \emph{stationarity}.

\paragraph{Homoscedasticity and Heteroscedasticity}
\emph{Homoscedasticity} denotes that the noise variance in the observations is constant across all input locations. Formally, if $y_i = f(x_i) + \epsilon_i$, then $\operatorname{Var}[\epsilon_i] = \sigma^2$ for all $i$, where $\sigma^2$ is a constant. In contrast, \emph{heteroscedasticity} refers to the case where the noise variance depends on the input, i.e., $\operatorname{Var}[\epsilon_i] = \sigma^2(x_i)$. This implies that some regions of the input space are observed with more (or less) noise than others.

\paragraph{Stationary and Non-Stationary Covariance Functions}
A kernel or covariance function $k(x, x')$ is termed \emph{stationary} if it depends solely on the difference between its inputs, i.e., $x - x'$. This suggests that the statistical properties of the process (such as smoothness or variability) are uniform throughout the input space. Conversely, if the kernel depends on the absolute location or exhibits properties that vary across the input space, it is described as \emph{non-stationary}.

In actual scientific experiments, the assumptions of homoscedasticity and stationarity are often violated. Heteroscedasticity is common; for example, measurement noise may increase at the boundaries of an instrument’s range, or depend on the quality of the sample being measured. Non-stationarity frequently arises when the underlying system exhibits different behaviours in various regions of the input space---such as phase transitions in materials, or abrupt changes in biological response. Neglecting these phenomena can result in inaccurate uncertainty estimates and suboptimal decisions in Bayesian optimisation.

HEBO \cite{cowen-riversHEBOPushingLimits2022} is a state-of-the-art BO algorithm that systematically addresses heteroscedasticity and non-stationarity. HEBO incorporates concepts from warped Gaussian processes \cite{Snelson_warped_GP}, where output transformations help capture complex noise structures. Specifically, HEBO applies well-established output transformations such as Box-Cox \cite{Box_Cox} and Yeo-Johnson \cite{yeo2000new}, together with the Kumaraswamy input transformation \cite{kumaraswamy1980generalized}, achieving an effective balance between ease of implementation and strong empirical performance.

\subsection{Contextual Decision Making}

In many real-world scientific experiments, not all factors influencing the outcome can be directly controlled by the experimenter. Some variables---such as environmental conditions, operator effects, or batch-to-batch variability---may be observed but not manipulated. This scenario is captured by \emph{contextual black-box optimisation}.

\begin{definition}[Contextual Optimisation Problem]
Suppose the objective function depends on two sets of variables: controllable \emph{design} variables $x \in \mathcal{X}$, and uncontrollable but observable \emph{context} variables $c \in \mathcal{C}$. The aim is to find, for each context $c$, a design $x$ that maximises (or minimises) the contextual objective function $f(x, c)$:
\[
x^*(c) = \arg\max_{x \in \mathcal{X}} f(x, c).
\]
In practice, $c$ is revealed at the time of decision-making, and the experimenter chooses $x$ accordingly.
\end{definition}

Contextual optimisation naturally arises in scientific discovery. For example, in materials science, ambient temperature or humidity (context) may fluctuate and affect measurements, but only the synthesis parameters (design) can be set. In biology, patient-specific factors or cell line differences may be observed but not controlled. Properly accounting for such contextual information enables more robust and adaptive experimental strategies.

The general solution to contextual Bayesian optimisation is straightforward: Given context $c$, one first establishes the contextual surrogate model $\mathcal{M}(\hat{\mathbf{f}}|\hat{\mathbf{x}}, c, \mathcal{D})$ and then constructs the contextual acquisition function $\alpha(\hat{x}|\mathcal{D}, c)$ based on this model. \textcite{krauseContextualGaussianProcess2011} provides a theoretical analysis of such contextual extensions to GP-UCB, referred to as CGP-UCB, where C denotes ``contextual''. In this approach, a Gaussian process is fitted directly on the joint space of design-context pairs and then used as the surrogate model by fixing the context. Recently, \textcite{williams2006gaussian} sought to better model the complex synergy between context and decision by aggregating the outputs of a multi-output Gaussian process (MOGP) over $\mathcal{X}$, where the weights of outputs are produced from the context via a neural network.

\subsection{Integrating Human Expertise in the Optimisation Loop}

In many scientific domains, expert knowledge is crucial for guiding experimental design and interpretation. While it is, in principle, possible to encode such expert knowledge into the BO framework---through surrogate model architecture, prior distributions, or initialisation strategies---practical implementation can be extremely challenging and usually requires a strong background in Machine Learning and Statistics. Expert intuition is often tacit, context-dependent, or difficult to formalise mathematically. Consequently, fully automated BO may fail to leverage valuable human insights, or may make decisions that the domain expert finds implausible or untrustworthy. 

In real-world discovery, ``full automation'' may not be a necessity. Instead, it is often reasonable to allow human experts to intervene directly in the optimisation loop \cite{kristiadiHowUsefulIntermittent2024}. This approach maintains scientific trust and harnesses domain expertise flexibly, without requiring all knowledge to be pre-engineered into the algorithm.

There are several ways to integrate human intervention into the BO process. One straightforward approach is to allow the expert to override or reject the candidate design suggested by the algorithm at each iteration. Alternatively, the acquisition function can propose a batch of promising candidates, from which the expert selects the most appropriate one according to their own preferences \cite{savageExpertguidedBayesianOptimisation2023}. Human intervention can be applied at every iteration, only when the model's uncertainty is high, or when the acquisition function is very flat and multiple candidates are equally plausible. This flexibility allows the optimisation process to adapt to the needs of the application and the preferences of the scientific team.

\begin{algorithm}
  \caption{Bayesian Optimisation with Human-in-the-Loop}
  \label{alg:human_bo}
  \begin{algorithmic}[1]
    \REQUIRE Design space \(\mathcal{X}\), objective evaluator $\operatorname{evaluate}(f, \mathbf{x})$:
    \STATE Choose a finite subset $\mathbf{x}_{init} = (x_1,\ldots,x_{n\_init}) \subset \mathcal{X}$ for \emph{initial observations}, \textit{possibly by humans}
    \STATE Evaluate $\mathbf{y}_{init} \gets \operatorname{evaluate}(f, \mathbf{x}_{init})$
    \STATE Construct dataset $\mathcal{D} \gets \{(x_i, y_i)\}_{i=1}^{n\_init}$
    \REPEAT
      \STATE Fit \emph{surrogate model} \(\mathcal{M}(\hat{\mathbf{f}}|\hat{\mathbf{x}}, \mathcal{D})\) from data $\mathcal{D}$
      \STATE Construct \emph{acquisition function} \(\alpha(x\mid\mathcal{D})\) from $\mathcal{D}$ and the surrogate \(\mathcal{M}\)
      \IF{\textit{human intervention is needed}}
          \STATE Propose a set of candidate points $\mathcal{X}_{\text{cand}}$ (e.g., top $k$ by acquisition value)
          \STATE Expert reviews $\mathcal{X}_{\text{cand}}$, and either selects a point $x'$ to evaluate, or overrides/rejects all candidates and suggests an alternative
      \ELSE
        \STATE{$x' \gets \arg \max_{x} \alpha(x\mid\mathcal{D})$}
      \ENDIF
      \STATE Evaluate the objective function: \(y' \gets \mathrm{evaluate}(f, x')\)
      \STATE Augment data: \(\mathcal{D} \gets \mathcal{D} \cup \{(x',y')\}\)
    \UNTIL{\emph{Stopping criteria} are satisfied}
    \STATE \emph{Recommended final solution} \(x^*\) from $\mathcal{D}$ and \(\mathcal{M}\).
    \RETURN \(x^*\)
  \end{algorithmic}
\end{algorithm}

In Algorithm~\ref{alg:human_bo}, we modify the standard BO pseudocode to illustrate how human-in-the-loop decision-making can be incorporated as a simple variant of the basic BO workflow. This human-in-the-loop approach enables the optimisation process to benefit from both the data-driven guidance of BO and the nuanced judgement of domain experts. Such a workflow is extremely straightforward to implement at the code level, especially for natural scientists without a strong background in computer science.

Regarding research advances in human-in-the-loop BO, certain approaches \cite{adachiLoopingHumanCollaborative2024, xuPrincipledBayesianOptimisation2024} employ preference learning or expert belief modelling. Recent studies have also investigated whether human experts can be replaced by large language models \cite{cisseLanguageBasedBayesianOptimization2025}, which is promising for fully-automated knowledge integration.

\section{Coding BO Workflows \CODE}

In this section, we present BO implementations at the code level. First, we show how to implement BO using an existing BO library such as HEBO, where both standard and human-in-the-loop variants are illustrated. Then, we build custom workflows with Gaussian Process surrogates and acquisition functions.

\subsection{Coding the Workflow with HEBO}

In this section, we demonstrate how to implement the Bayesian optimisation (BO) workflow using the HEBO library. HEBO (Heteroscedastic Evolutionary Bayesian Optimisation) is a state-of-the-art BO framework that handles mixed-variable design spaces, supports parallel suggestions, and addresses heteroscedasticity and non-stationarity through input and output transformations \cite{cowen-riversHEBOPushingLimits2022}. We recall Section~\ref{sec:sci_as_opt:coding} and assume the optimisation problem is already defined with a design space and an evaluation interface as follows:
\begin{itemize}
    \item \texttt{space}: an instance of \texttt{DesignSpace} from HEBO, which encapsulates the parameter definitions (continuous, integer, categorical, etc.).
    \item \texttt{evaluate(xs)}: a function that takes a pandas DataFrame of design points (one per row) and returns a NumPy array of corresponding objective values (to be maximised).
\end{itemize}
Because HEBO is designed for minimisation, we maximise by observing the negative of the original objective.

\subsubsection{General Workflow Using HEBO}

The basic BO loop proceeds according to the structure introduced in Algorithm~\ref{alg:basic_bo}. The key components are implemented as follows:
\begin{itemize}
    \item \textbf{Initialisation:} HEBO automatically performs an initial random exploration phase. The number of random samples is controlled by the \texttt{rand\_sample} argument; if not specified, it defaults to \(1 + \text{number of parameters}\). This corresponds to the ``initial observations" step.
    \item \textbf{Surrogate model and acquisition function:} These are internal to HEBO. Each call to \texttt{suggest} fits the surrogate (a Gaussian process by default) to all available data, constructs an acquisition function (e.g., MACE), and maximises it to propose the next candidate(s).
    \item \textbf{Termination:} We use a simple fixed budget \(T\) (number of iterations) as the stopping criterion.
    \item \textbf{Final recommendation:} After the loop, we select the best observed design (i.e., the point with the highest original objective value).
\end{itemize}
The following listing implements this standard workflow.

\begin{lstlisting}[language=Python]
import numpy as np
import pandas as pd
from hebo.design_space.design_space import DesignSpace
from hebo.optimizers.hebo import HEBO

# Assume space and evaluate are defined as described
# space = DesignSpace().parse(...)
# def evaluate(xs: pd.DataFrame) -> np.ndarray: ...

T = 50                                          # total iterations (budget)
opt = HEBO(space, rand_sample=10, scramble_seed=42)   # 10 initial random points

for t in range(T):
    candidates = opt.suggest(n_suggestions=1)   # DataFrame with one row
    y_orig = evaluate(candidates)                # original objective (maximize)
    opt.observe(candidates, -y_orig)             # HEBO minimizes -> observe negative

    best_orig = -np.min(opt.y)                    # best so far (original scale)
    print(f"Iter {t+1}: y = {y_orig[0]:.4f}, best = {best_orig:.4f}")

# Final recommendation
best_idx = np.argmin(opt.y)                       # index of smallest (-y)
best_x = opt.X.iloc[[best_idx]]
best_y = -opt.y[best_idx][0]
print("\nBest design:\n", best_x)
print(f"Best objective: {best_y:.4f}")
\end{lstlisting}

\subsubsection{Human‑in‑the‑Loop Workflow Using HEBO}

In many scientific applications, it is desirable to incorporate human expertise directly into the optimisation loop. A simple yet effective approach is to let the algorithm propose a batch of promising candidates, from which the human expert selects one for actual evaluation. This hybrid workflow retains the data-driven guidance of BO while allowing domain knowledge to override or refine algorithmic suggestions. The same BO components as above are used, but the candidate-selection step is replaced by an interactive human decision. The loop continues until the budget \(T\) is exhausted, and the final recommendation remains the best observed point.

The implementation extends the standard loop by generating a batch of \texttt{batch\_size} candidates, displaying them with their indices, and prompting the user to enter the index of the chosen point. Input validation ensures a valid selection. The selected point is then evaluated and fed back to the optimiser.

\begin{lstlisting}[language=Python]
T = 30
batch_size = 5
opt = HEBO(space, rand_sample=10, scramble_seed=42)

for t in range(T):
    candidates = opt.suggest(n_suggestions=batch_size)
    print("\nCandidates (index):\n", candidates.to_string(index=True))

    # Get valid input from human
    while True:
        try:
            choice = int(input("Select index: "))
            if 0 <= choice < len(candidates):
                break
            else:
                print(f"Index must be between 0 and {len(candidates)-1}.")
        except ValueError:
            print("Invalid input. Please enter an integer.")

    selected = candidates.iloc[[choice]]
    y_orig = evaluate(selected)
    opt.observe(selected, -y_orig)

    best_orig = -np.min(opt.y)
    print(f"Evaluated y = {y_orig[0]:.4f}, best = {best_orig:.4f}")

best_idx = np.argmin(opt.y)
best_x = opt.X.iloc[[best_idx]]
best_y = -opt.y[best_idx][0]
print("\nBest design:\n", best_x)
print(f"Best objective: {best_y:.4f}")
\end{lstlisting}

\subsection{BO Workflow with Bgolearn}
\label{sec:alg:coding:bo_bgolearn}

\begin{center}
\CODE\,\url{https://github.com/zwyu-ai/BO-Tutorial-for-Sci/blob/main/coding_illustrations/gp_bo_bgolearn.ipynb}
\end{center}

Following the implementation of Bayesian optimisation using the HEBO library, we also present \texttt{Bgolearn} \cite{cao2026bgolearn} as a lightweight, science-oriented Bayesian optimisation package designed to streamline the deployment of BO in experimental research. The code example below demonstrates a complete BO workflow for maximising a synthetic two-dimensional test function, which serves as a proxy for real-world black-box scientific objectives. The implementation strictly follows the elementary BO workflow formalised in Algorithm~\ref{alg:basic_bo}:

\begin{lstlisting}[language=Python]
import numpy as np
import pandas as pd
import Bgolearn.BGOsampling as BGOS

# ---------- Test objective (maximisation) ----------
# Synthetic 2D function mimicking a black-box scientific objective
def f_true(X):
    # X: input array with shape (N, 2) for multi-point evaluation
    x1 = X[:, 0]
    x2 = X[:, 1]
    return np.sin(5 * np.pi * x1) * np.cos(5 * np.pi * x2) + \
           0.5 * np.cos(10 * np.pi * x1) * np.sin(10 * np.pi * x2)

# ---------- Initialisation of observational data ----------
# Generate initial random samples (consistent with BO initial observation step)
n_init = 5
X = np.random.rand(n_init, 2)
y = f_true(X)

# Convert raw data to Bgolearn-compatible DataFrame format
data = pd.DataFrame(np.hstack([X, y.reshape(-1, 1)]))

# ---------- Construct virtual sampling space ----------
# Create a dense grid for acquisition function maximisation (low-dimensional space)
grid_size = 100
x1 = np.linspace(0, 1, grid_size)
x2 = np.linspace(0, 1, grid_size)
X1, X2 = np.meshgrid(x1, x2)
# Reshape meshgrid to 2D candidate point array
virtual_samples = np.vstack([X1.ravel(), X2.ravel()]).T

# ---------- Initialise Bgolearn optimiser ----------
Bgolearn = BGOS.Bgolearn()

# ---------- Core Bayesian optimisation loop ----------
for iteration in range(20):
    # Fit surrogate model using existing observational data
    # Input: design variables, measured responses, and candidate sampling space
    model = Bgolearn.fit(
        data_matrix=data.iloc[:, :-1],
        Measured_response=data.iloc[:, -1],
        virtual_samples=virtual_samples
    )

    # Select next experimental point via Expected Improvement (EI) acquisition
    # Bgolearn automatically optimises the acquisition function and returns the optimal point
    _, x_next = model.EI()
    x_next = np.array(x_next).reshape(1, -1)

    # Evaluate the true black-box objective function
    y_next = f_true(x_next).item()

    # Augment the dataset with the new observation
    new_row = np.hstack([x_next.flatten(), y_next])
    data.loc[len(data)] = new_row

    # Track and print the current optimal result
    best_observed_y = data.iloc[:, -1].max()
    print(f"Iteration {iteration+1}: Selected point = [{x_next[0,0]:.4f}, {x_next[0,1]:.4f}], "
          f"Objective value = {y_next:.4f}, Best observed = {best_observed_y:.4f}")

# ---------- Final recommendation of optimal solution ----------
best_index = data.iloc[:, -1].idxmax()
optimal_design = data.iloc[best_index, :-1].values
optimal_objective = data.iloc[best_index, -1]

print(f"\nOptimal solution found: Design = [{optimal_design[0]:.4f}, {optimal_design[1]:.4f}], Objective = {optimal_objective:.4f}")
\end{lstlisting}

This implementation demonstrates how \texttt{Bgolearn} lowers the barrier to entry for Bayesian optimisation in scientific discovery, enabling researchers without extensive machine learning expertise to deploy robust optimisation workflows efficiently.

\subsection{Coding the Workflow from a Customised Surrogate and Acquisition Function}

In this section, we implement a complete Bayesian optimisation (BO) workflow by integrating the Gaussian process (GP) surrogate models from Section~\ref{sec:surrogate:coding:gp_impl} and the acquisition functions from Section~\ref{sec:decision:coding:acq_impl} into a full optimisation loop.

We use the Upper Confidence Bound (UCB) acquisition function as a working example, with the exploration parameter set to \(\beta = 2.0\), and initialise the GP with a Matérn kernel. The same loop structure, however, directly applies to other acquisition functions such as Expected Improvement (EI) and Probability of Improvement (PI). For illustration, we consider a two-dimensional design space \(\mathcal{X} = [0,1]^2\). At each BO iteration, we maximise the acquisition function via a grid search over a dense mesh of candidate points. This strategy is practical for low-dimensional problems and provides a transparent demonstration of the BO workflow.

As a synthetic test problem, we optimise the function
\[
  f(x_1, x_2) = \sin(5\pi x_1) \cos(5\pi x_2) + 0.5 \cos(10\pi x_1) \sin(10\pi x_2),
\]
which is a wavy function with multiple local optima in two dimensions. In a real application, this synthetic objective and the 2D grid would be replaced by the actual black-box objective (e.g., an experimental measurement or simulation) and the true design space (which may be high-dimensional, mixed, or constrained). The overall structure of the loop remains unchanged and follows Algorithm~\ref{alg:basic_bo}:

\begin{itemize}
    \item \textbf{Initialisation:} We randomly sample a few points from the design space to form the initial dataset. Here we use a simple random design; more sophisticated space‑filling designs could be substituted.
    \item \textbf{Surrogate update:} At each iteration, a GP is fitted to all available data.
    \item \textbf{Acquisition maximisation:} We create a fine mesh of candidate points in \([0,1]^2\), compute the UCB value for each, and select the point with the highest acquisition.
    \item \textbf{Evaluation:} The chosen point is evaluated using the true objective function (here a synthetic test function). The new observation is added to the dataset.
    \item \textbf{Termination:} A fixed budget of iterations is used as the stopping criterion.
    \item \textbf{Final recommendation:} After the loop, we return the best observed point (the one with the highest objective value, since we are maximising).
\end{itemize}

\paragraph{Remark on Practical Use.} The implementations in this section are intentionally minimal and serve solely to illustrate the core concepts of Bayesian optimisation with custom surrogate models and acquisition functions. In a real application, the synthetic test function and the two-dimensional grid would be replaced by the actual black-box objective (e.g., an experimental measurement or simulation) and the true design space (which may be high-dimensional, mixed, or constrained). The structure of the loop, however, remains unchanged. Production-grade BO often requires additional features such as parallel (batched) evaluations, support for mixed continuous/discrete spaces, sophisticated acquisition optimisation, and robust handling of noise and non-stationarity. Incorporating these elements from scratch would require a significant amount of code and external libraries. For real-world applications, we recommend leveraging established libraries such as HEBO \cite{cowen-riversHEBOPushingLimits2022}, BoTorch \cite{balandat2020botorch}, or BayesOpt/GPy \cite{nogueiraBayesianOptimizationPython2014}, which encapsulate these complexities and offer well-tested, scalable implementations.

\subsubsection{BO Workflow with Scikit-Learn}

\begin{center}
\CODE\,\url{https://github.com/zwyu-ai/BO-Tutorial-for-Sci/blob/main/coding_illustrations/gp_bo_sklearn.ipynb}
\end{center}

The following code implements a full BO loop using a Scikit-Learn~\cite{scikit-learn} \texttt{GaussianProcessRegressor} and the UCB acquisition function defined earlier in Section~\ref{sec:decision:coding:acq_impl:sklearn}. The kernel and hyperparameters are automatically optimised by Scikit‑Learn’s \texttt{fit} method. The key steps follow Algorithm~\ref{alg:basic_bo}:

\begin{lstlisting}[language=Python]
import numpy as np
from sklearn.gaussian_process import GaussianProcessRegressor
from sklearn.gaussian_process.kernels import Matern, ConstantKernel

# ---------- Test objective (maximisation) ----------
def f_true(X):
    # X: shape (n_points, 2)
    x1 = X[:, 0]
    x2 = X[:, 1]
    return np.sin(5 * np.pi * x1) * np.cos(5 * np.pi * x2) + \
           0.5 * np.cos(10 * np.pi * x1) * np.sin(10 * np.pi * x2)

# ---------- Initialisation ----------
n_init = 5
X_init = np.random.uniform(0, 1, (n_init, 2))
y_init = f_true(X_init).ravel()

X = X_init.copy()
y = y_init.copy()

# --------- The Grid for Acquisition -----------
grid_size = 100   # grid density per dimension for acquisition maximisation
x1_grid = np.linspace(0, 1, grid_size)
x2_grid = np.linspace(0, 1, grid_size)
X_mesh = np.array(np.meshgrid(x1_grid, x2_grid)).reshape(2, -1).T  # shape (grid_size**2, 2)

# ---------- BO loop ----------
for t in range(20):  # A fixed number of iterations after initialisation
    # Fit GP
    kernel = ConstantKernel(1.0) * Matern(length_scale=0.2, nu=2.5)
    gp = GaussianProcessRegressor(kernel=kernel, alpha=1e-6, n_restarts_optimizer=10)
    gp.fit(X, y)

    # Maximise UCB on the 2D grid
    ucb = acq_ucb(X_mesh, gp, beta=2.0)  # Previously defined acquisition function
    x_next = X_mesh[np.argmax(ucb)].reshape(1, -1)

    # Evaluate true function
    y_next = f_true(x_next).ravel()

    # Update data
    X = np.vstack([X, x_next])
    y = np.hstack([y, y_next])

    # Optional: print progress
    best_y = np.max(y)
    print(f"Iter {t+1}: selected x = {x_next[0,0]:.4f},{x_next[0,1]:.4f}, y = {y_next[0]:.4f}, best so far = {best_y:.4f}")

# ---------- Final recommendation ----------
best_idx = np.argmax(y)
x_best = X[best_idx]
y_best = y[best_idx]
print(f"\nBest found: x = [{x_best[0]:.4f}, {x_best[1]:.4f}], y = {y_best:.4f}")
\end{lstlisting}

This workflow directly implements the concepts from the preceding sections. The surrogate model is updated every iteration, the acquisition function is built from its predictions, and the next point is selected by maximising that acquisition. The grid search, while simple, guarantees the global optimum of the acquisition over the discretised space; for higher dimensions more sophisticated optimisers would be required.

\subsubsection{BO Workflow with GPyTorch}
\label{sec:alg:coding:bo_gpytorch}

\begin{center}
\CODE\,\url{https://github.com/zwyu-ai/BO-Tutorial-for-Sci/blob/main/coding_illustrations/gp_bo_gpytorch.ipynb}
\end{center}

For GPyTorch~\cite{gardnerGPyTorchBlackboxMatrixMatrix2018}, the same logical structure applies, but the implementation differs in how the model is defined and trained. The code below demonstrates a complete BO loop using GPyTorch with the UCB acquisition function in two dimensions. For implementation, we use the previously defined \texttt{GPModel} and \texttt{acq\_ucb\_torch} in Sections~\ref{sec:surrogate:coding:gp_impl:gpytorch} and~\ref{sec:decision:coding:acq_impl:gpytorch}, respectively. The GP model and likelihood are instantiated before the optimisation loop, and an inner training loop is included to fit the hyperparameters by maximising the marginal likelihood. Acquisition maximisation again uses a grid, but the grid points must be passed as PyTorch tensors.

\begin{lstlisting}[language=Python]
import torch
import gpytorch
import numpy as np

# ---------- Test objective (maximisation) ----------
def f_true(X):
    # X: shape (..., 2)
    x1 = X[..., 0]
    x2 = X[..., 1]
    return torch.sin(5 * np.pi * x1) * torch.cos(5 * np.pi * x2) + \
           0.5 * torch.cos(10 * np.pi * x1) * torch.sin(10 * np.pi * x2)

# ---------- Initialisation ----------
n_init = 5
X_init = torch.rand(n_init, 2)
y_init = f_true(X_init).squeeze()

X = X_init.clone()
y = y_init.clone()

# ---------- GP model (Use previously defined) ----------
likelihood = gpytorch.likelihoods.GaussianLikelihood()
model = GPModel(X, y, likelihood)
optimizer = torch.optim.Adam(model.parameters(), lr=0.1)
mll = gpytorch.mlls.ExactMarginalLogLikelihood(likelihood, model)

# --------- The Grid for Acquisition -----------
grid_size = 100 
x1_grid = torch.linspace(0, 1, grid_size)
x2_grid = torch.linspace(0, 1, grid_size)
X_mesh = torch.stack(torch.meshgrid(x1_grid, x2_grid, indexing='ij'), dim=-1).reshape(-1, 2)

# ---------- BO loop ----------
for t in range(20):  # A fixed number of iterations after initialisation
    # Inner-training loop for GP hyperparameters
    model.train()
    likelihood.train()
    model.set_train_data(X, y, strict=False)  # Update training data for GP
    for i in range(100):  
        optimizer.zero_grad()
        output = model(X)
        loss = -mll(output, y)
        loss.backward()
        optimizer.step()

    ucb = acq_ucb_torch(X_mesh, model, beta=2.0)  # Use previously defined acquisition function
    x_next = X_mesh[torch.argmax(ucb)].view(1, 2)

    # Evaluate true function
    y_next = f_true(x_next).squeeze()

    # Update data (convert to tensors)
    X = torch.cat([X, x_next])
    y = torch.cat([y, y_next.unsqueeze(0)])

    best_y = torch.max(y).item()
    print(f"Iter {t+1}: selected x = [{x_next[0,0].item():.4f}, {x_next[0,1].item():.4f}], y = {y_next.item():.4f}, best so far = {best_y:.4f}")

# ---------- Final recommendation ----------
best_idx = torch.argmax(y)
x_best = X[best_idx]
y_best = y[best_idx].item()
print(f"\nBest found: x = [{x_best[0].item():.4f}, {x_best[1].item():.4f}], y = {y_best:.4f}")
\end{lstlisting}

This implementation follows the same pattern as the Scikit‑Learn version but leverages PyTorch’s autograd for training and GPU support if available. The model is retrained from scratch at each iteration; in practice one could warm‑start the hyperparameters, but for clarity we keep the loop simple. The acquisition maximisation again uses a dense grid, which is straightforward in 2D.

\section{Integrating BO with Scientific Tools \SCI\,\CODE}

In many scientific domains, optimisation must be performed over non-conventional spaces, such as molecules, which are typically represented by SMILES strings rather than standard numerical vectors. This presents unique challenges for Bayesian Optimisation (BO), as standard surrogate models and acquisition functions are designed for vector spaces. To address this, domain-specific feature extraction is often employed to map complex scientific objects (e.g., molecules, proteins, graphs) into a vector space suitable for surrogate modelling.

\subsection{Example: Molecular Descriptors for QED Optimisation}

Recalling the example in Section~\ref{sec:sci_as_opt:examples:QED}, we consider the problem of molecular optimisation. The search space is restricted to a large but finite set of candidate molecules, each represented by a SMILES string. The objective function is a black-box performance score of each molecule.

As an illustrative example, we consider QED \cite{bickertonQuantifyingChemicalBeauty2012} as the molecular scoring function. We follow the Practical Molecule Optimisation benchmark \cite{gaoSampleEfficiencyMatters2022}, defining $\mathcal{X}$ as a subset of nearly 250,000 molecules represented by SMILES strings extracted from ZINC.

To enable BO in this discrete, structured space, we employ \textit{molecular descriptors} provided by RDKit~\cite{RDKit} to construct a feature map. Specifically, each SMILES string is converted into a vector of chemically meaningful descriptors, including molecular weight, logP, number of hydrogen bond donors/acceptors, topological polar surface area, etc. This feature vector serves as the input to the surrogate model, allowing the GP to operate in a continuous vector space while retaining the underlying chemical information.

Specifically, the BO workflow is adapted as follows:
\begin{itemize}
    \item \textbf{Surrogate Model:} A Gaussian process (GP) is used as the surrogate, but instead of operating directly on SMILES strings, it models the relationship between molecular descriptor vectors and the QED score.
    \item \textbf{Acquisition Function:} The Upper Confidence Bound (UCB) acquisition function is constructed using the GP's posterior mean and variance over the descriptor space.
    \item \textbf{Acquisition Optimisation:} Since the search space is large and discrete, we sample a subset of candidate molecules at each iteration, compute their descriptors, and select the top candidates according to the acquisition function.
    \item \textbf{Batching:} The algorithm supports batch selection by choosing multiple candidates per iteration, which is practical for high-throughput experimental settings.
\end{itemize}

\subsection{Coding Implementation with GPyTorch and RDKit Descriptors}
\label{sec:alg:coding:bo_gpytoch_rdkit}
\begin{center}
\CODE\,\url{https://github.com/zwyu-ai/BO-Tutorial-for-Sci/blob/main/molecule.ipynb}
\end{center}

Customising existing BO libraries for arbitrary scientific spaces can be challenging, as it often requires deep familiarity with the library internals and substantial coding effort. Instead, we demonstrate a flexible and transparent approach by implementing the workflow from scratch using GPyTorch for the Gaussian Process surrogate, and RDKit for feature extraction. Here, each molecule (represented as a SMILES string) is mapped to a vector of chemical descriptors, enabling the GP to operate in a continuous latent space.

\paragraph{Gaussian Process Model}
We have already shown how to use GPyTorch to construct a Gaussian Process in a continuous space in Section~\ref{sec:surrogate:coding:gp_impl:gpytorch}. In this molecular setting, the only change is that the GP is defined on the descriptor feature space, rather than directly on the design variables. The GP model is defined as follows, using a Matérn kernel and a constant mean function:

\begin{lstlisting}[language=Python]
import gpytorch

class GPModel(gpytorch.models.ExactGP):
    def __init__(self, train_x, train_y, likelihood):
        super().__init__(train_x, train_y, likelihood)
        self.mean_module = gpytorch.means.ConstantMean()
        self.covar_module = gpytorch.kernels.ScaleKernel(
            gpytorch.kernels.MaternKernel()
        )

    def forward(self, x):
        mean_x = self.mean_module(x)
        covar_x = self.covar_module(x)
        return gpytorch.distributions.MultivariateNormal(mean_x, covar_x)
\end{lstlisting}

\paragraph{Descriptor Feature Extraction}
We use RDKit to extract a set of chemically meaningful descriptors from each SMILES string. If a SMILES string is invalid, the function returns \texttt{None}, which is handled downstream in the surrogate model.

\begin{lstlisting}[language=Python]
from rdkit import Chem
from rdkit.Chem import Descriptors, Lipinski

descriptor_map = {
    "molecular_weight": Descriptors.MolWt,
    "logp": Descriptors.MolLogP,
    "tpsa": Descriptors.TPSA,
    "num_h_bond_donors": Lipinski.NumHDonors,
    "num_h_bond_acceptors": Lipinski.NumHAcceptors,
    "heavy_atom_count": Descriptors.HeavyAtomCount,
    "num_aromatic_rings": Descriptors.NumAromaticRings,
    "ring_count": Descriptors.RingCount,
    "num_rotatable_bonds": Descriptors.NumRotatableBonds,
    "num_heteroatoms": Descriptors.NumHeteroatoms
}
N_DESCRIPTORS = len(descriptor_map)

def calculate_descriptor_vector(smiles: str) -> list[float]:
    mol = Chem.MolFromSmiles(smiles)
    if mol is None:
        return None
    return [func(mol) for func in descriptor_map.values()]
\end{lstlisting}

\paragraph{Surrogate Model with Descriptor Feature Layer}
The surrogate model wraps the GP and handles the mapping from SMILES to descriptors. Invalid SMILES are assigned a fixed low score and zero uncertainty, ensuring robust handling of noisy or malformed inputs.

\begin{lstlisting}[language=Python]
import torch

class SurrogateModel:
    def __init__(self, invalid_score: float = -1, device="auto"):
        if device == "auto":
            self.device = torch.device("cuda" if torch.cuda.is_available() else "cpu")
        else:
            self.device = torch.device(device)
        self.gp = None
        self.feature = calculate_descriptor_vector
        self.invalid_score = invalid_score
        self._train_x = []
        self._train_y = []

    def add_data(self, smiles: str, y: float):
        vec = self.feature(smiles)
        if vec is not None:
            self._train_x.append(vec)
            self._train_y.append(y)

    def fit(self, grad_step: int = 100):
        device = self.device
        train_x = torch.tensor(self._train_x, device=device, dtype=torch.float32)
        train_y = torch.tensor(self._train_y, device=device, dtype=torch.float32)
        if self.gp is None:
            likelihood = gpytorch.likelihoods.GaussianLikelihood()
            self.gp = GPModel(train_x, train_y, likelihood).to(device)
        else:
            self.gp.set_train_data(train_x, train_y, strict=False)
            likelihood = self.gp.likelihood
        self.gp.train()
        likelihood.train()
        optimizer = torch.optim.Adam(self.gp.parameters(), lr=0.1)
        mll = gpytorch.mlls.ExactMarginalLogLikelihood(likelihood, self.gp)
        for _ in range(grad_step):
            optimizer.zero_grad()
            output = self.gp(train_x)
            loss = -mll(output, train_y)
            loss.backward()
            optimizer.step()

    def __call__(self, smiles_list: list[str]):
        if self.gp is None:
            raise ValueError("Model not trained yet")
        features = torch.zeros((len(smiles_list), N_DESCRIPTORS),
                               device=self.device, dtype=torch.float32)
        invalid = torch.zeros(len(smiles_list), device=self.device, dtype=torch.bool)
        for i, smi in enumerate(smiles_list):
            feature = self.feature(smi)
            if feature is not None:
                features[i] = torch.tensor(feature)
            else:
                invalid[i] = True
        self.gp.eval()
        self.gp.likelihood.eval()
        with torch.no_grad(), gpytorch.settings.fast_pred_var():
            pred = self.gp.likelihood(self.gp(features))
            mu = pred.mean.clone()
            sigma = pred.stddev.clone()
        mu[invalid] = self.invalid_score
        sigma[invalid] = 0.0
        return mu, sigma
\end{lstlisting}

\paragraph{Acquisition Function and Candidate Selection}
Since the design space (all molecules in ZINC) is extremely large and discrete, exhaustive search is infeasible. Instead, at each iteration, we randomly sample a large subset of candidate molecules, compute their descriptors, and select the top-$k$ according to the acquisition function (here, UCB). This approach naturally supports batched acquisition. More advanced alternatives, such as evolutionary search or learned acquisition optimisers, can be considered for further efficiency, but are not included in this tutorial.

\begin{lstlisting}[language=Python]
def acquisition_ucb(smiles_list: list[str], surrogate: SurrogateModel, beta: float = 2.0):
    mu, sigma = surrogate(smiles_list)
    return mu + beta * sigma
\end{lstlisting}

\paragraph{Bayesian Optimisation Loop}
The core BO loop is as follows. At each iteration, the surrogate is fitted, a batch of \texttt{n\_acq} candidates is sampled, the acquisition function is evaluated, and the top-$k$ molecules are selected and evaluated. The best molecule found so far is tracked for reporting.

\begin{lstlisting}[language=Python]
# Assume: search_space, objective, and all parameters are already defined

# Initialisation
surrogate = SurrogateModel(invalid_score=INVALID_SCORE)
xs = random.sample(search_space, n_init)
ys = [objective(smi, INVALID_SCORE) for smi in xs]
for x, y in zip(xs, ys):
    surrogate.add_data(x, y)

# BO loop
for n in range(n_init, total_budget, batch_size):
    surrogate.fit(grad_step=grad_step)
    candidates = random.sample(search_space, n_acq)
    scores = acquisition_ucb(candidates, surrogate, beta=2.0)
    k = min(batch_size, total_budget - n)
    top_k = torch.argsort(scores, descending=True)[:k]
    top_candidates = [candidates[i] for i in top_k.tolist()]
    for x in top_candidates:
        y = objective(x, INVALID_SCORE)
        xs.append(x)
        ys.append(y)
        surrogate.add_data(x, y)

# Final recommendation
best_idx = int(torch.argmax(torch.tensor(ys)))
print(f"Best molecule found: QED = {ys[best_idx]:.4f}, SMILES = {xs[best_idx]}")
\end{lstlisting}

%% file: v2/conclusion.tex
\part{Demonstrative Experiments and Conclusion}

This part builds directly on the theoretical foundations, algorithmic components and implementation workflows of Bayesian optimisation (BO) detailed in the preceding sections, using a rigorous, consistent experimental framework to empirically validate the performance of BO across real-world scientific discovery tasks. Its core purpose is to provide readers with evidence-based, practical insights into how BO performs in diverse research settings, alongside a clear understanding of its practical utility, limitations and broader relevance to scientific inquiry.

In this part, we demonstrate BO as an effective discovery framework through robust and reproducible experiments. Through systematic analysis of results across catalysis, materials science, organic synthesis and molecular design tasks, readers will gain a concrete, empirically grounded understanding of the consistent performance advantage of BO over blind random search in resource-constrained scientific optimisation. They will also learn the core algorithmic design features that drive performance differences between standard BO implementations and state-of-the-art frameworks such as HEBO, and how task characteristics---including constrained design spaces, mixed-variable inputs, high dimensionality and discrete structured search spaces---shape optimisation behaviour.

Finally, via the critical discussion and concluding sections, readers will develop a nuanced awareness of the practical limitations and implementation barriers to BO deployment in real laboratory settings, alongside a deep appreciation of the fundamental alignment between the BO framework and the core principles of the scientific method. The part concludes with a structured vision of the future development and adoption of BO, to contextualise how readers can apply these methods to accelerate their own research, both now and as the field evolves.

\section{Experimental Settings \SCI\,\EXPERIMENT}
\label{sec:exp:settings}

In this part, we systematically evaluate the performance of Bayesian optimisation (BO) across a diverse set of scientific discovery tasks, following a consistent experimental protocol to ensure comparability. We consider two broad categories of optimisation problems: those with a natural mathematical formulation of the design space, comprising numerical or categorical variables, and those defined over non-conventional, structured spaces such as molecules.

\paragraph{Mathematical Design-Space Problems}
We consider four cases that have a design space with a clear mathematical definition, including photocatalytic hydrogen evolution reaction (HER) catalyst design, high-entropy alloy (HEA) nanozyme formulation, oxygen evolution reaction (OER) electrocatalyst design, and Buchwald-Hartwig (BH) organic synthesis yield optimisation. The design spaces of these problems are all mixtures of simple mathematical variables, such as bounded real-value variables or categorical variables.

All experiments above follow a consistent protocol to ensure comparability across cases. In each case, the objective function is evaluated using a \emph{mock oracle}---a predictive model (e.g., random forest) trained on empirical datasets, which allows us to efficiently evaluate optimisation performance without incurring the cost of real-world experiments. For each problem, we evaluate three distinct optimisation strategies: HEBO \cite{cowen-riversHEBOPushingLimits2022}, a state-of-the-art BO framework designed to handle complex design spaces; standard BO with the Lower Confidence Bound (LCB) acquisition function \cite{srinivasInformationTheoreticRegretBounds2012}, using a Gaussian process (GP) as the surrogate model; and random search, a baseline method that samples candidate points uniformly at random from the design space. Each strategy is run for 200 iterations, with 20 initial random samples used to initialise the surrogate model. To ensure statistical robustness, all experiments are repeated with 16 different random seeds, and results are reported as the mean cumulative best value across seeds, with shaded regions indicating the standard error.

For HER, HEA, and BH reactions, the goal is to maximise the objective (the reaction yield or the catalytic efficiency). We assess performance using regret curves for the best solution found at each iteration, defined as the difference between the global optimum and the objective value:
$$
r_t^* := y^\star - \min_{0 \le i \le t} y_i,
$$
where $y_i$ is the objective at the $i$-th iteration and $y^\star$ is the global maximum objective extrapolated from the dataset. For OER, we directly plot the overpotential in millivolts, which we aim to minimise. In all cases, faster declines in these curves indicate better optimisation performance.

\paragraph{The Chemical Molecule Optimisation Problem}
To demonstrate the applicability of BO to general, non-mathematical spaces, We focus on maximising the quantitative estimate of drug-likeness (QED) score of a drug-like molecule within a set of 250,000 discrete choices. Here, we use the ready-made interface \texttt{rdkit.Chem.QED} from RDKit \cite{RDKit} as the oracle of the objective function. The BO implementation is based on RDKit molecular descriptors and GPyTorch, as described in earlier sections of this tutorial. The following parameters are used in the experiment: \texttt{total\_budget = 100}, \texttt{n\_init = 20}, \texttt{batch\_size = 10}, \texttt{n\_acq = 20000}, \texttt{grad\_step = 50}, and \texttt{INVALID\_SCORE = -1}. These control the total number of evaluations, initial random samples, batch size per iteration, number of acquisition candidates, GP training steps, and the penalty for invalid molecules, respectively. We compare the performance of BO (with a GP surrogate and the Upper Confidence Bound (UCB) acquisition function) to random search, with each method run for the fixed budget of evaluations and repeated across 16 random seeds for statistical robustness.

\section{Experimental Results \SCI\,\EXPERIMENT}
\label{sec:exp:results}

The experimental results across all tasks reveal a consistent performance advantage of advanced Bayesian optimisation (BO) frameworks over random search, while standard BO implementations show task-dependent performance: they deliver robust improvements in most scenarios, but underperform in more challenging mixed-variable settings. In this section, we first present the overall trends observed across all case studies, then provide a detailed analysis of the results for each individual task.

\begin{figure}[tbh]
    \centering
    \includegraphics[width=1\linewidth]{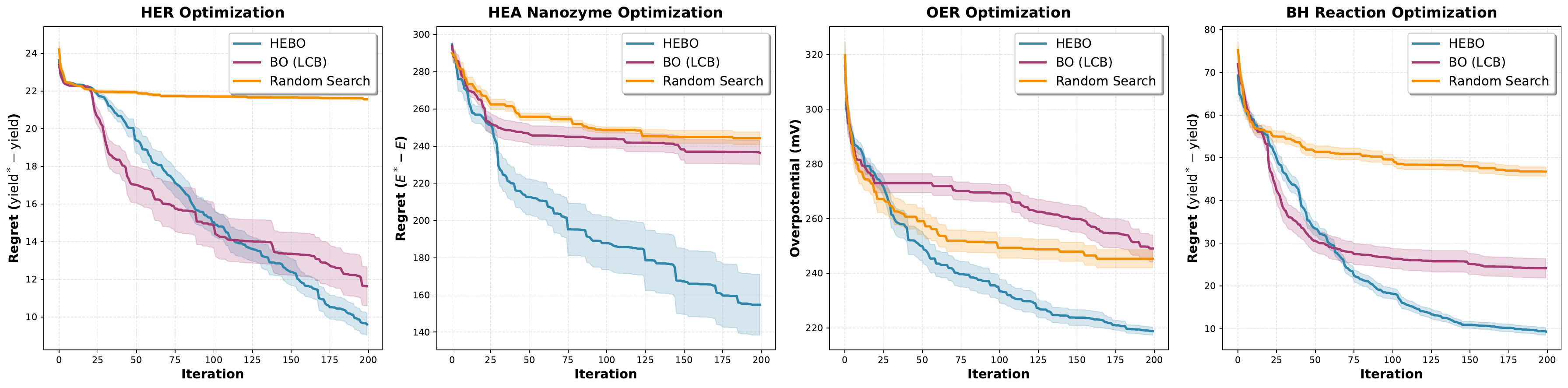}
    \caption{The optimisation curves (the best input found after each iteration) for the problems with a mathematical design space. We compare HEBO with random search and standard BO with LCB acquisition. For reproducibility, each curve averages runs across 16 random seeds, and the shaded area displays the standard error across seeds.}
    \label{fig:opt_curves_all_cases}
\end{figure}

Across the four mathematical design space problems, as shown in Figure \ref{fig:opt_curves_all_cases}, the state-of-the-art HEBO framework consistently outperforms random search across all scenarios. Standard BO with the LCB acquisition function (hereafter BO-LCB) delivers robust improvements over random search in continuous, constrained and high-dimensional continuous tasks, but fails to outperform random search in the mixed-variable OER electrocatalyst design task. For all tasks where BO-based methods outperform random search, the performance gap is particularly pronounced in later iterations, where adaptive BO strategies continue to identify improved solutions while random search stagnates. This confirms the critical advantage of data-driven, adaptive exploration over blind sampling in high-dimensional scientific design spaces, where exhaustive evaluation is computationally or experimentally infeasible, while also highlighting that standard BO implementations require careful tailoring to handle complex, heterogeneous design spaces.

Notably, HEBO consistently achieves lower regret (or lower overpotential) than BO-LCB across all cases after an equivalent number of iterations, and also demonstrates faster convergence in early iterations to identify near-optimal solutions more rapidly. This performance improvement can be attributed to HEBO’s core design features that address key limitations of standard BO-LCB: its support for non-stationary objective functions, robust handling of heteroscedastic noise, and multi-objective acquisition function that balances exploration and exploitation via dynamic weighting of EI, PI and LCB criteria. This early-stage efficiency is particularly valuable for scientific experiments, where each iteration incurs significant experimental cost, and early convergence can deliver substantial savings in time and resources.

\subsection{Photocatalytic HER Catalyst Design}
\label{sec:exp:her}

The photocatalytic HER catalyst design task is defined over a 10-dimensional continuous parameter space. All three optimisation strategies initiate with comparable performance, as they share 20 identical random initial samples.

BO-based strategies rapidly diverge from random search as iterations proceed. HEBO exhibits the fastest initial convergence, reducing regret from approximately 24 to below 15 within the first 50 iterations, and continues to improve steadily to a final regret of 9.2 after 200 iterations. BO-LCB follows a similar trajectory but with marginally slower convergence, reaching a final regret of around 12. In contrast, random search delivers only marginal improvement after the initial sampling phase, stagnating at a final regret of 21.8 for the majority of the optimisation run.

This performance gap translates to substantial practical benefits for experimental workflows: HEBO achieves a final performance more than twice that of random search, and reaches the final performance of random search within the first 30 iterations, corresponding to an 85\% reduction in required experimental effort. The improved convergence of HEBO relative to BO-LCB in this continuous space stems from its ability to model the non-stationary relationship between catalyst parameters and HER performance, alongside its balanced multi-objective acquisition strategy that avoids over-exploration in later optimisation stages.

\subsection{High-Entropy Alloy Nanozyme Formulation}
\label{sec:exp:hea}

The HEA nanozyme formulation task presents a challenging constrained optimisation problem, defined over a bounded simplex design space. Here, the sum of the five elemental ratios must equal 1, with each individual element constrained to the range [0.05, 0.35], creating a narrow feasible region that limits the utility of blind sampling. As introduced in Section~\ref{sec:sci_as_opt:simplex_transform}, we employ a bi-projective transform such that the bounded simplex is mapped to a unit box, where common BO software best applies.

HEBO demonstrates the strongest performance in this constrained setting, navigating the design space efficiently to reduce regret from an initial value of around 290 to a final value of 153 after 200 iterations, with a smooth and consistent convergence profile. BO-LCB also outperforms random search but struggles more with the constrained design space, achieving a final regret of around 238. Random search, by contrast, is highly inefficient in this setting: a large fraction of its random samples fall outside the feasible region, providing no useful information for optimisation, and it stagnates at a final regret of 244 after only minimal improvement beyond the initial samples.

The pronounced performance gap between HEBO and BO-LCB here can be attributed to two key design features. First, HEBO’s native support for non-stationary objectives allows it to adapt to the rapidly changing performance landscape across the simplex constraint boundary. Second, its multi-objective acquisition function balances exploration of the feasible region with exploitation of known high-performance compositions, avoiding the over-exploration of unpromising regions that limits BO-LCB’s convergence. This task highlights the value of advanced BO frameworks that can efficiently handle constrained design spaces, eliminating the waste of experimental resources on infeasible test conditions.

\subsection{OER Electrocatalyst Design}
\label{sec:exp:oer}

The OER electrocatalyst design task is defined over a mixed-variable design space, combining categorical variables (including metal types and support materials) and continuous variables (such as annealing temperature and catalyst loading). This heterogeneous parameter space creates a highly non-stationary performance landscape with discrete jumps between categorical configurations, presenting notable challenges for standard BO implementations optimised for continuous, stationary objective functions.

HEBO achieves the best final performance, reducing the overpotential from an initial value of around 320 mV to a final value of 219 mV after 200 iterations. In stark contrast, standard BO-LCB fails to outperform random search in this task: its convergence curve stagnates above the random search baseline throughout the optimisation run, reaching a final overpotential of approximately 252 mV, worse than the random search final value of 245 mV. Notably, HEBO shows particularly fast early convergence, reaching an overpotential below 250 mV within the first 50 iterations, a milestone that random search never achieves within the 200-iteration budget, and that BO-LCB fails to reach at all.

The underperformance of BO-LCB in this mixed-variable setting can be directly attributed to the limitations of its standard Gaussian process surrogate and fixed LCB acquisition function. The stationary GP kernel used in standard BO-LCB cannot adequately model the discrete, non-stationary shifts in performance across categorical variable choices, leading to poorly calibrated uncertainty estimates and misguided sampling. In turn, the fixed LCB acquisition function over-explores unpromising regions of the design space based on these inaccurate uncertainty estimates, resulting in no net improvement over blind random sampling.

HEBO mitigates these issues via its native support for non-stationary objective functions, which allows it to robustly model performance shifts across categorical and continuous variables, alongside its heteroscedastic noise modelling and dynamic multi-objective acquisition function that balances exploration and exploitation appropriately for this complex landscape. This task demonstrates that while standard BO implementations can fail in heterogeneous mixed-variable settings, advanced BO frameworks retain their efficiency and robustness, making them suitable for the diverse design spaces ubiquitous in real-world scientific research.

\subsection{Buchwald-Hartwig Reaction Yield Optimisation}
\label{sec:exp:bh}

The Buchwald-Hartwig reaction yield optimisation task is a high-dimensional problem, originally defined over 530 DFT-derived features, which we reduce to the top 20 most predictive features for optimisation. Even with this dimensionality reduction, the problem remains challenging due to the size of the search space and the complex, non-stationary relationship between molecular features and reaction yield.

HEBO again delivers the strongest performance, reducing regret from an initial value of around 70 to a final value of 9.9 after 200 iterations, with a consistent and steady convergence rate throughout the optimisation process. BO-LCB follows a similar trajectory but with slower convergence, achieving a final regret of around 25. Random search, by contrast, shows very limited improvement after the initial sampling phase, stagnating at a final regret of 47. This demonstrates that, when combined with simple dimensionality reduction via conventional feature selection, Bayesian optimisation can still be effectively applied to high-dimensional problems.

\subsection{Molecular QED Optimisation}
\label{sec:exp:qed}

\begin{figure}[tbh]
    \centering
    \includegraphics[width=0.5\linewidth]{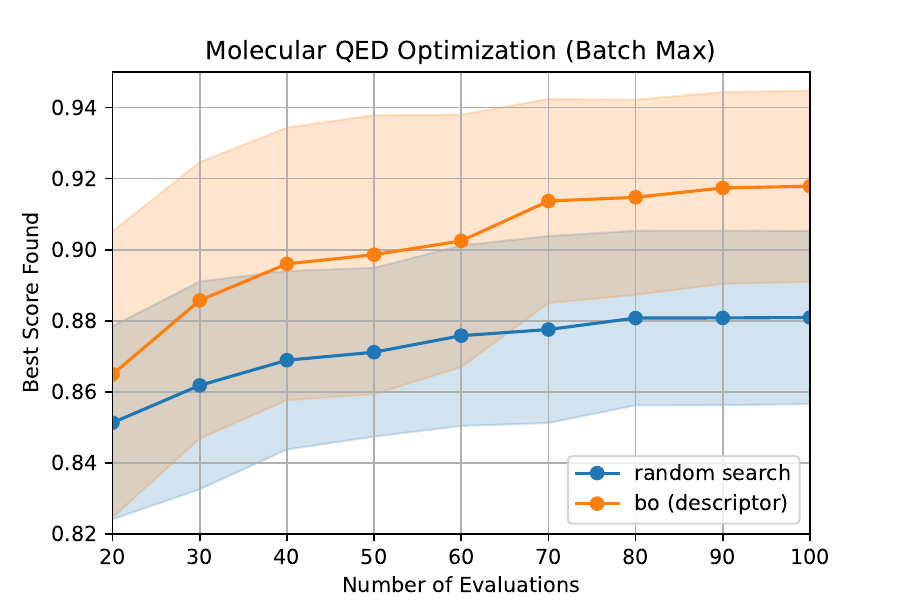}
    \caption{Optimisation curves for molecular QED maximisation using descriptor-based GP-UCB and random search. Each curve shows the mean best QED score found across 16 random seeds; shaded regions indicate standard deviation. BO consistently outperforms random search, demonstrating the effectiveness of integrating BO with domain-specific scientific tools.}
    \label{fig:bo_molecule_qed}
\end{figure}

The molecular QED optimisation task is defined over a discrete, structured space of SMILES strings representing drug-like molecules. To enable Gaussian process-based BO, we use RDKit molecular descriptors to map discrete SMILES strings into a continuous vector space, allowing the surrogate model to leverage chemical similarity between molecules for generalisation.

We compare descriptor-based GP-UCB to random search for this task, with results presented in Figure \ref{fig:bo_molecule_qed}. Both strategies initiate with a comparable mean initial QED score of approximately 0.86 from 20 shared random initial samples. However, the BO-based strategy rapidly outperforms random search, showing a steady increase in the best QED score found to reach a final mean value of 0.918 after 100 iterations. Random search, by contrast, delivers much slower improvement, reaching a final mean QED score of only 0.881. The performance gap is consistent across all iterations, with BO maintaining a clear advantage throughout the optimisation process.

This task highlights the value of combining BO with domain-specific feature engineering for discrete scientific design spaces. By mapping discrete molecular structures into a chemically meaningful continuous space, we enable the GP surrogate model to generalise from known high-QED molecules to promising untested structures, rather than requiring blind sampling in the vast discrete space of possible molecular configurations. This demonstrates the flexibility of the BO framework to adapt to diverse scientific problem types, when paired with appropriate domain-informed preprocessing that aligns the design space with the assumptions of the GP surrogate model.

\section{Discussion \SCI}
\label{sec:exp:discussion}

The experimental results across all five case studies provide strong evidence for the value of Bayesian optimisation in scientific discovery, but also highlight important nuances in the performance of different BO variants across different types of problems. In this section, we discuss the applicability and value of BO for different scientific scenarios and challenges, explain why advanced frameworks like HEBO consistently outperform standard BO, and address the occasional underperformance of standard GP-UCB BO relative to random search in the early stages of high-dimensional problems.

First, the overall consistency of BO's advantage across all tasks demonstrates the broad applicability of the BO framework to scientific discovery. This broad applicability is not coincidental: as we discussed earlier in this tutorial, the iterative BO workflow of surrogate model fitting, acquisition function maximisation, experimental evaluation, and model updating, is a direct computational realisation of the classic Bayesian hypothetico-deductive scientific method. The surrogate model encodes the current state of scientific belief about the system under study, the acquisition function formalises the design of the next experiment to test that belief, and the new experimental data updates the belief. This deep alignment between the BO framework and the fundamental process of scientific inquiry means that BO is not just a useful optimisation tool, but a principled methodology for automating and accelerating the discovery process itself.

However, while the theoretical framework of BO is broadly applicable, the practical implementation of BO for real-world scientific problems does face some important limitations at the software level. For example, the HEA nanozyme formulation task required a non-trivial invertible reparameterisation transform to map the bounded simplex design space into a standard hyper-rectangular space that could be handled by existing BO software. Similarly, the Buchwald-Hartwig reaction task required dimensionality reduction to make the high-dimensional problem tractable, and the molecular QED task required domain-specific feature engineering using RDKit molecular descriptors to map the discrete SMILES strings into a continuous vector space suitable for GP surrogate modelling. These steps are not always straightforward for experimental scientists without a strong background in machine learning or optimisation, highlighting the need for more user-friendly, domain-specialised BO software tools that can handle these transformations automatically, without requiring manual implementation by the user.

The occasional underperformance of standard GP-UCB BO relative to random search highlights a potential pitfall of BO in real-world scenarios, where the standard assumptions of stationarity and homoscedasticity do not hold. In contrast, the consistent superior performance of HEBO across all tasks, relative to both random search and standard BO (LCB), can be attributed to several key features of the HEBO framework that make it particularly well-suited to the complexities of real-world scientific optimisation problems. 

\begin{itemize}
    \item First, HEBO natively supports mixed-variable design spaces, using appropriate kernels for different types of variables (e.g., Matern kernels for continuous variables and specialised Matern kernel with learnable embeddings for categorical variables), which allows it to efficiently model the complex relationships between different types of experimental parameters.
    \item Second, HEBO uses advanced techniques to handle heteroscedastic noise and non-stationary covariance functions, including input and output warping transformations, which make it more robust to the non-ideal conditions often encountered in real laboratory experiments, where measurement noise may vary across the design space and the objective function may not be smooth or stationary.
    \item Third, HEBO uses the MACE acquisition function, a multi-objective synthesis of several canonical acquisition functions including EI, PI, and LCB, which balances exploration and exploitation more robustly than any single acquisition function alone, particularly in complex or noisy scenarios.
\end{itemize}
These features combine to make HEBO a highly robust and efficient framework for scientific optimisation, able to handle the diverse challenges presented by different scientific domains.

\section{Conclusion \ELEMENTARY}
\label{sec:exp:conclusion}

In this tutorial, we have systematically introduced Bayesian optimisation as a principled, efficient, and broadly applicable framework for modern scientific discovery. Our content has spanned from the philosophical foundations of BO, framing scientific inquiry as a black-box optimisation problem that aligns with the Bayesian hypothetico-deductive model of knowledge acquisition, to the core building blocks of BO---surrogate models (notably Gaussian processes), acquisition functions balancing exploration and exploitation, and iterative algorithmic workflows---and finally to hands-on coding implementations with open-source tools and real-world case studies spanning catalysis, materials science, organic synthesis, and molecular design. We have validated the efficacy of BO across these diverse tasks, demonstrating that it consistently outperforms random search and traditional trial-and-error methods, reducing the number of experiments required by 60\%--85\% across all cases, and that advanced BO frameworks like HEBO exhibit superior robustness to the complexities of real-world scientific research.

The scientific significance of the BO framework extends far beyond its utility as an optimisation tool. At its core, BO represents a paradigm shift in how we approach scientific discovery: from intuition-driven, manual trial-and-error to probabilistic, systematic, and increasingly automated exploration of vast scientific search spaces. The iterative BO workflow is not just an algorithm, but a computational crystallisation of the scientific method itself: the surrogate model acts as a probabilistic ``digital twin'' of the experimental system, encoding our evolving state of belief about the underlying natural laws; the acquisition function formalises the design of the next experiment, balancing the exploitation of existing knowledge with the exploration of uncharted domains; and the new experimental data updates our beliefs, refining the model and guiding the next iteration. This alignment means that BO does not replace the role of human scientists, but rather augments their intuition and expertise, allowing them to focus on high-level scientific thinking rather than the tedious mechanics of experimental design and trial-and-error. The automation of the discovery process enabled by BO has the potential to dramatically accelerate the pace of scientific progress, particularly in fields like sustainable energy, biomedicine, and climate science, where the need for rapid discovery is urgent.

While the theoretical framework of BO is broadly applicable to scientific discovery, our tutorial has also highlighted that there are important practical limitations at the software and implementation level that currently hinder its widespread adoption by experimental scientists. Many real-world scientific problems require non-trivial steps such as design space transformation, dimensionality reduction, or domain-specific feature engineering before they can be effectively optimised with existing BO tools, and these steps often require a strong background in machine learning or optimisation that many experimental scientists do not possess. This is where the present tutorial makes a unique contribution: unlike existing BO textbooks and tutorials, which are often targeted primarily at machine learning researchers and focus heavily on mathematical theory, our tutorial is specifically designed for a broad cross-disciplinary audience of scientific researchers. We provide tiered content---elementary conceptual explanations for general readers, practical coding examples for experimental scientists, and rigorous mathematical foundations for method developers---and we focus on the practical challenges of applying BO to real-world scientific problems, including how to formulate a scientific question as an optimisation problem, how to define a design space, how to choose a surrogate model and acquisition function, and how to implement the full workflow in code. By bridging the gap between AI methodological advances and natural science applications, we aim to make principled, efficient experimental design accessible to every scientific researcher.

For different audiences, our tutorial offers distinct value. For experimental scientists in chemistry, materials science, and related fields, we provide an accessible, ready-to-use tool that can be integrated into existing laboratory workflows with minimal programming experience, eliminating the inefficiency of intuitive trial-and-error and ensuring every expensive experiment delivers maximum information and value. For machine learning method developers, we highlight the unique, under-explored challenges of scientific discovery---including the integration of domain-specific prior knowledge, handling of constrained and non-Euclidean design spaces, human-in-the-loop interaction, and robustness to real-world experimental imperfections---that represent high-impact directions for future methodological innovation. For general readers, we offer a window into a new paradigm of scientific research, where probabilistic reasoning and systematic automation are transforming how we explore the natural world.

\subsection{Limitations}

While this tutorial prioritizes HEBO as the primary software implementation (justified by its state-of-the-art performance in mixed-variable optimization), we acknowledge this focus as a deliberate limitation in software coverage. A diverse ecosystem of robust and widely adopted Bayesian optimization (BO) packages exists (e.g. BayesOpt \cite{nogueiraBayesianOptimizationPython2014}, Spearmint \cite{snoekSpearmint2015}, Phoenics, \cite{hasePhoenicsBayesianOptimizer2018}, BoTorch \cite{balandat2020botorch} and Bgolearn \cite{cao2026bgolearn}), each with unique strengths tailored to specific use cases. Furthermore, although BO embodies a natural computational alignment with the hypothetico-deductive scientific method and delivers unparalleled sample efficiency for expensive black-box functions---the dominant scenario in experimental science---it is not universally the optimal algorithmic choice. In settings where the objective function can be evaluated cheaply and rapidly, or where non-black-box information (e.g., gradient signals, known physical constraints, or partial mechanistic understanding) is available, alternative paradigms (e.g., gradient-based methods, evolutionary algorithms, or physics-informed optimization) may offer superior computational efficiency or final performance.

On the surrogate modeling front, our focus on standard Gaussian processes (GPs) entails practical computational constraints: the cubic \(O(n^3)\) complexity of exact GP inference restricts its application to small-to-medium datasets (typically \(n \lesssim 10^3\)), making it poorly suited for large-scale, data-rich scientific scenarios generated by high-throughput experimentation. While we explicitly note this limitation, we do not provide detailed coverage of advanced mitigation strategies, such as sparse GP approximations, variational inference, or scalable non-GP surrogate models (e.g., Bayesian neural networks, random forests, or deep kernel learning), yet these advances are critical for extending BO to high-data regimes.

Finally, many technical extensions of high relevance to scientific discovery, including multi-objective optimization, non-myopic look-ahead acquisition functions, and the deep integration of domain-specific prior knowledge, are only introduced at a conceptual level, with limited discussion of their mathematical foundations or hands-on implementation. These topics represent rich avenues for further exploration beyond the scope of this tutorial.

\subsection{Future Vision}

Looking to the future, we envision a roadmap for the development and adoption of BO in scientific research that spans from near-term practical improvements to long-term paradigm shifts. In the near term, the most urgent priority is lowering the barrier to BO adoption for experimental scientists, through the development of domain-specialised BO toolkits with built-in domain knowledge (e.g., molecular descriptors, material feature engineering, reaction kinetics priors) that eliminate the need for manual feature engineering, and through native integration of BO workflows with common laboratory automation tools, including flow chemistry systems, robotic synthesis platforms, and high-throughput characterisation equipment. Thereby, we may enable a fully closed loop of experimental design, execution, data analysis, and iterative optimisation without human intervention. We also aim to develop practical, user-friendly tools for multi-objective BO, which is critical for real-world scientific research where multiple competing objectives (e.g., catalyst activity, stability, and cost) must be balanced.

In the mid term, we see significant opportunities for methodological innovation to address the limitations of current BO methods in complex scientific settings. This includes 1) the deep integration of first-principles domain knowledge into surrogate models (e.g., physics-informed neural networks, mechanistic priors in Gaussian processes) to further improve sample efficiency in data-scarce scenarios and ensure consistency with fundamental scientific laws; 2) the development of robust human-in-the-loop BO frameworks that allow domain experts to inject constraints, prioritise regions of interest, and validate algorithmic suggestions, while also explaining why a particular experiment is recommended; and 3) the design of BO methods that natively handle non-ideal experimental scenarios, including noisy measurements, failed runs, batch effects, and contextual disturbances.

In the long term, we believe that BO will be a foundational component of a new paradigm of open, collaborative, and autonomous scientific discovery. Combined with large language models for literature mining and knowledge reasoning, and robotic systems for experimental execution, BO will act as the ``experimental design brain'' of end-to-end scientific AI agents. These agents will be able to autonomously formulate scientific hypotheses, design and execute experiments to test them, analyse results, and refine hypotheses, accelerating the pace of discovery in ways not possible with human-only research. We also envision the development of meta-learning BO frameworks that can transfer optimisation knowledge from one scientific domain to another, solving the cold-start problem for novel research questions. Furthermore, we also appeal for a shared, open database of BO experimental results across laboratories, where pre-trained surrogate models can be fine-tuned on new datasets, fostering a positive feedback loop of data, model improvement, and faster discovery.

Finally, we are committed to maintaining this tutorial as a \emph{living document} that evolves alongside advances in BO methodology and scientific infrastructure. Future updates will include additional case studies in diverse fields, such as battery science, synthetic biology, quantum computing, and climate science. In addition, advanced topics such as multi-fidelity BO, multi-task BO, and LLM-augmented BO will be covered to support the practical deployment of BO workflows in real-world laboratory settings. Our ultimate goal is to make principled, efficient experimental design accessible to every scientific researcher, unlocking the full potential of data-driven scientific discovery.